\documentclass[default,iicol]{sn-jnl}

\jyear{2022}%

\usepackage{dirtytalk}
\usepackage{enumitem}
\usepackage[utf8]{inputenc} 
\usepackage[T1]{fontenc}    
\usepackage{hyperref}       
\usepackage{url}            
\usepackage{booktabs}       
\usepackage{amsfonts}       
\usepackage{nicefrac}       
\usepackage{microtype}      
\usepackage{xcolor}         
\usepackage{makecell, multirow}
\usepackage{algorithm}
\usepackage{amsthm}
\usepackage{algpseudocode}
\usepackage{wrapfig}
\usepackage{caption}
\usepackage{subcaption}

\usepackage{amsmath,amsfonts,bm}
\usepackage{ amssymb }

\def\eqref#1{equation~\ref{#1}}

\def\1{\bm{1}}

\def\vc{{\bm{c}}}
\def\vd{{\bm{d}}}
\def\ve{{\bm{e}}}

\def\vs{{\bm{s}}}

\def\vv{{\bm{v}}}

\def\vx{{\bm{x}}}

\def\vsigma{{\bm{\sigma}}}
\def\vlambda{{\bm{\lambda}}}

\def\mG{{\bm{G}}}
\def\mH{{\bm{H}}}
\def\mI{{\bm{I}}}
\def\mJ{{\bm{J}}}

\DeclareMathAlphabet{\mathsfit}{\encodingdefault}{\sfdefault}{m}{sl}
\SetMathAlphabet{\mathsfit}{bold}{\encodingdefault}{\sfdefault}{bx}{n}

\newcommand{ \paren }[1]{ \ensuremath {\left(  #1 \right)} }
\newcommand{\norm}[2]{\left\| #1 \right\|_{#2}}
\newcommand{\abs}[1]{\left| #1 \right|}

\usepackage{mathtools}
\newcommand{\defeq}{\vcentcolon=}

\usepackage{graphicx,amsfonts,amscd,amssymb,bm,url,color,latexsym,bbm,amsthm,amsmath}
\usepackage{physics}
\allowdisplaybreaks
\usepackage[capitalize,nameinlink]{cleveref}

\newtheorem{theorem}{Theorem}[section]
\newtheorem{lemma}[theorem]{Lemma}

\renewcommand{\mathbf}{\boldsymbol}

\newcommand{\mb}{\mathbf}
\newcommand{\mc}{\mathcal}

\newcommand{\bb}{\mathbb}

\newcommand{\setJu}[1]{\left\{ #1 \right\}}

\newcommand{\reals}{\bb R}

\newcommand{\eps}{\varepsilon}
\newcommand{\RJU}{\reals}

\newcommand{\indicator}{\indic}

\newcommand{\brac}{\bqty}
\newcommand{\Brac}{\Bqty}

\DeclareMathOperator{\conv}{conv}

\DeclareMathOperator{\sign}{sign}

\DeclareMathOperator{\st}{s.t.}

\DeclareMathOperator*{\argmin}{arg\,min}

\newcommand{\wh}{\widehat}

\newcommand{\TJU}{\intercal}

\newcommand{\indic}[1]{\mathbbm 1\left\{#1\right\}} 

\newcommand{\formulation}{Form.\,}

\newcommand{\pygranso}{\texttt{PyGRANSO}}
\newcommand{\change}{\textcolor{black}}

\begin{document}

\title[ ]{Optimization and Optimizers for Adversarial Robustness}

\author*[1]{\fnm{Hengyue} \sur{Liang}}\email{liang656@umn.edu}

\author[2]{\fnm{Buyun} \sur{Liang}}\email{liang664@umn.edu}

\author[2]{\fnm{Le} \sur{Peng}}\email{peng0347@umn.edu}

\author[3]{\fnm{Ying} \sur{Cui}}\email{yingcui@umn.edu}
\author[4]{\fnm{Tim} \sur{Mitchell}}\email{tim.mitchell@qc.cuny.edu}

\author*[2]{\fnm{Ju} \sur{Sun}}\email{jusun@umn.edu}

\affil[1]{\orgdiv{Department of Electrical \& Computer Engineering}, \orgname{University of Minnesota}
}
\affil[2]{\orgdiv{Department of Computer Science \& Engineering}, \orgname{University of Minnesota}
}
\affil[3]{\orgdiv{Department of Industrial \& Systems Engineering}, \orgname{University of Minnesota}
}
\affil[4]{\orgdiv{Department of Computer Science}, \orgname{Queens College, City University of New York}
}


\abstract{Empirical robustness evaluation (RE) of deep learning models against adversarial perturbations entails solving nontrivial constrained optimization problems. Existing numerical algorithms that are commonly used to solve them in practice predominantly rely on projected gradient, and mostly handle perturbations modeled by the $\ell_1$, $\ell_2$ and $\ell_\infty$ distances. In this paper, we introduce a novel algorithmic framework that blends a general-purpose constrained-optimization solver~\pygranso~\textbf{w}ith \textbf{C}onstraint-\textbf{F}olding (PWCF), which can add more reliability and generality to the state-of-the-art RE packages, e.g., \texttt{AutoAttack}. Regarding \emph{reliability}, PWCF provides solutions with stationarity measures and feasibility tests to assess the solution quality. 
For \emph{generality}, PWCF can handle perturbation models that are typically inaccessible to the existing projected gradient methods; the main requirement is the distance metric to be almost everywhere differentiable. Taking advantage of PWCF and other existing numerical algorithms, we further explore the distinct patterns in the solutions found for solving these optimization problems using various combinations of losses, perturbation models, and optimization algorithms. We then discuss the implications of these patterns on the current robustness evaluation and adversarial training.}


\maketitle

\section{Introduction}
\label{Sec:introduction}
\begin{table}[tb]
    \centering
    \begin{tabular}{l c}
        \hline
        \multicolumn{2}{c}{\textbf{Frequently Used Abbreviations}} \\
        \hline
        APGD & Auto-PGD \\
        AR & Adversarial robustness \\
        AT & Adversarial training \\
        CE & Cross entropy \\
        DNN  & Deep neural network \\
        FAB & Fast adaptive boundary (attack) \\
        LPA & Lagrange Perceptual Attack \\ 
        MaxIter & Maximum number of iteration allowed\\
        NLOPT & Nonlinear optimization \\
        PAT & Perceptual adversarial attack \\
        PD & Perceptual distance \\
        PGD & Projected gradient descent \\
        PWCF & \pygranso~with constraint-folding \\
        RE  & Robustness evaluation \\
        \hline
    \end{tabular}
\end{table}
In visual recognition, deep neural networks (DNNs) are not robust against perturbations that are easily discounted by human perception---either adversarially constructed or naturally occurring~\citep{szegedy2013intriguing,GoodfellowEtAl2015Explaining,hendrycks2018benchmarking,EngstromEtAl2019Exploring,xiao2018spatially,WongEtAl2019Wasserstein,LaidlawFeizi2019Functional,HosseiniPoovendran2018Semantic,BhattadEtAl2019Big}.
\change{A popular way to formalize robustness is by \emph{adversarial loss}, attempting to find the worst-case perturbation(s) of images within a prescribed radius,} by solving the following constrained optimization problem~\cite{HuangEtAl2015Learning, madry2017towards}:
\begin{align}
\label{eq:robust_loss}
   \begin{split}
       \max_{\mb x'}\; & \ell\paren{\mb y, f_{\mb \theta}(\mb x')}\\
   \st \; d\paren{\mb x, \mb x'}& \le \eps\; ,  \quad  \mb x' \in [0, 1]^n   \; \text{,}
   \end{split}
\end{align}
\change{which we will refer to as \emph{max-loss form} in the following context.
Here, $\mb x \in \mathcal{X}$ is a test image from the set $\mathcal{X}$}, $f_{\mb \theta}$ is a DNN parameterized by $\mb \theta$, 
$\mb x'$ is a perturbed version of $\mb x$ with an allowable perturbation radius $\eps$ measured by the distance metric $d$, and $\mb x' \in [0, 1]$ ensures that $\mb x'$ is a valid image ($n$: the number of pixels).
\change{Another popular formalism of robustness is by \emph{robustness radius}, defined as the minimal level of perturbation(s) so that $f_{\mb \theta}(\mb x)$ and $f_{\mb \theta}(\mb x')$ can lead to different predictions, which was first introduced in~\cite{szegedy2013intriguing}, even earlier than max-loss form (\ref{eq:robust_loss})}: 
\begin{align}  
\label{eq:min_distort}
\begin{split}
& \min_{\mb x'}\;  d\paren{\mb x, \mb x'} \\
\st \;  \max_{i \ne y} f_{\mb \theta}^i & (\mb x') \ge f_{\mb \theta}^y (\mb x')\; , \; \mb x' \in [0, 1]^n ~ \text{,} 
\end{split}
\end{align}
which we will refer to as \emph{min-radius form} in the following context.
Here, $y$ is the true label of $\mb x$ and $f_{\mb \theta}^i$ represents the $i^{th}$ logit output of the DNN model.
Early work assumes that the distance metric $d$ in max-loss form and min-radius form is the $\ell_p$ distance, where $p= 1, 2, \infty$ are popular choices~\cite{madry2017towards,GoodfellowEtAl2015Explaining}. Recent work has also modeled non-trivial transformations using metrics other than the $\ell_p$ distances~\cite{hendrycks2018benchmarking,EngstromEtAl2019Exploring,xiao2018spatially,WongEtAl2019Wasserstein,LaidlawFeizi2019Functional,HosseiniPoovendran2018Semantic,BhattadEtAl2019Big,laidlaw2021perceptual} to capture visually realistic perturbations and generate adversarial examples with more varieties.

As for applications, max-loss form is usually associated with an attacker in the adversarial setup, where the attacker tries to find an input that is visually similar to $\mb x$ but can fool $f_{\mb \theta}$ to predict incorrectly, as the solutions to (\ref{eq:robust_loss}) lead to worst-case perturbations to alter the prediction from the desired label.
Thus, it is popular to perform robust evaluation (RE) using max-loss form: generating perturbations over a validation dataset and reporting the classification accuracy (termed \emph{robust accuracy}) using the perturbed samples. 
Max-loss form also motivates \emph{adversarial training} (AT) 
to try to achieve adversarial robustness (AR)~\cite{GoodfellowEtAl2015Explaining,HuangEtAl2015Learning,madry2017towards}: 
\begin{align} 
\label{eq:minmax_obj}
    \min_{\mb \theta} \; \bb E_{\paren{\mb x, \mb y} \sim \mc D} \max_{\mb x' \in \Delta(\mb x)} \ell\paren{\mb y, f_{\mb \theta}(\mb x')}
\end{align} where $\mc D$ is the data distribution, and $\Delta(\mb x) = \{\mb x' \in [0, 1]^n: d(\mb x, \mb x') \le \eps\}$, in contrast to the classical supervised learning:
\begin{equation} 
\min_{\mb \theta} \bb E_{\paren{\mb x, \mb y} \sim \mc D}\; \ell\paren{\mb y, f_{\mb \theta}(\mb x)} \; \text{.}
\end{equation}
As for min-radius form, it is also a popular choice to calculate robust accuracy in RE\footnote{One can also perform AT using (\ref{eq:min_distort}) via bi-level optimization; see, e.g., \cite{ZhangEtAl2021Revisiting}.}~\cite{CroceHein2020Minimally, croce2020reliable, PintorEtAl2021Fast}, as solving (\ref{eq:min_distort}) will also produce samples $\mb x'$ to fool the model. But more importantly, solving min-radius form can be used to estimate the sample-wise robustness radius---a quantity that can be used to measure the robustness for every input. However, it is common that existing methods for solving (\ref{eq:min_distort}) emphasize the role of finding $\mb x'$ but overlook the importance of the robustness radius, e.g.,\cite{CroceHein2020Minimally, PintorEtAl2021Fast}.

Despite being popular, solving max-loss form and min-radius form is not easy when DNNs are involved. For max-loss form, the objective is non-concave for typical choices of loss $\ell$ and model $f_{\mb \theta}$; for non-$\ell_p$ metrics $d$, $\mb x'$ often belongs to a complicated non-convex set. In practice, there are two major lines of algorithms: 
\textbf{1) direct numerical maximization} that takes (sub-)differentiable $\ell$ and $f_{\mb \theta}$, and tries direct maximization: for example, using gradient-based methods~\cite{madry2017towards,croce2020reliable}. This often produces sub-optimal solutions and may lead to overoptimistic RE;
\textbf{2) upper-bound maximization} that constructs tractable upper bounds for the margin loss:
\begin{equation}
\label{eq: margin loss}
    \ell_{\mathrm{ML}} = \max_{i \ne y} f_{\mb \theta}^i (\mb x') - f_{\mb \theta}^y (\mb x') ~ \text{,}
\end{equation} 
and then optimizes against the upper bounds~\cite{SinghEtAl2018Fast,SinghEtAl2018Boosting,SinghEtAl2019abstract,SalmanEtAl2019Convex,DathathriEtAl2020Enabling,MuellerEtAl2022PRIMA}. Improving the tightness of the upper bound while maintaining tractability remains an active area of research. It is also worth mentioning that, since $\ell_{\mathrm{ML}} > 0$ implies attack success, upper-bound maximization can also be used for RE by providing an underestimate of robust accuracy~\cite{SinghEtAl2018Fast,SinghEtAl2018Boosting,SinghEtAl2019abstract,SalmanEtAl2019Convex,DathathriEtAl2020Enabling,MuellerEtAl2022PRIMA,WongKolter2017Provable,RaghunathanEtAl2018Certified,WongEtAl2018Scaling,DvijothamEtAl2018Training,LeeEtAl2021Towards}.
As for min-radius form, it can be solved exactly by mixed-integer programming for small-scale cases with certain restrictions on $f_{\mb \theta}$ and some choices of $d$~\cite{TjengEtAl2017Evaluating,KatzEtAl2017Reluplex,BunelEtAl2020Branch}. For general $f_{\mb \theta}$ and some $d$, lower bounds to the robustness radius can be calculated~\cite{WengEtAl2018Towards,ZhangEtAl2018Efficient,WengEtAl2018Evaluating,LyuEtAl2020Fastened}. Whereas in practice, min-radius form is heuristically solved by gradient-based methods or iterative linearization~\cite{szegedy2013intriguing,MoosaviDezfooliEtAl2015DeepFool,Hein2017,CarliniWagner2016Towards,rony2019decoupling,CroceHein2020Minimally,PintorEtAl2021Fast} to obtain upper bounds of the robustness radius; see \cref{sec:background}.

Currently, applying direct numerical methods to solve max-loss form and min-radius form is the most popular strategy for RE in practice, e.g.,~\cite{croce2021robustbench}. However, there are two major limitations:
\begin{itemize}[leftmargin=*]
    \item \textbf{Lack of reliability:} current direct numerical methods are able to obtain solutions but do not assess the solution quality. Therefore, the reliability of the solutions is questionable. \cref{Fig:APGD-FAB-Terminate-Iter} shows that the default stopping point (determined by the maximum allowed number of iterations (MaxIter)) used in \texttt{AutoAttack}~\cite{croce2020reliable} mainly leads to premature termination of the optimization process. In fact, it is challenging to terminate the optimization process by setting the MaxIter, since the actual number of iterations needed can vary from sample to sample.
    
    \item \textbf{Lack of generality:} existing numerical methods are applied mostly to problems where $d$ is $\ell_1$, $\ell_2$ or $\ell_\infty$ distance\footnote{The only exception we have found is~\cite{croce2019sparse}, where perturbations constrained by the $\ell_0$ distance are considered. However, this perturbation model has not been adapted to current empirical RE.} but cannot handle other distance metrics. In fact, the popular RE benchmark, \texttt{robustbench}~\cite{croce2021robustbench}, only has $\ell_2$ and $\ell_\infty$ leaderboards, and the most studied adversarial attack model is the $\ell_\infty$ attack~\cite{bai2021recent}.
\end{itemize}
It is important to address these two limitations, \change{as reliability relates to how much we can trust the evaluation, and the generality of handling more distance metrics can enable RE with more diverse perturbation models for broader applications.}

\begin{figure*}[!tb]
\centering
\begingroup 
\setlength{\tabcolsep}{1pt}
\renewcommand{\arraystretch}{0.8}
\begin{tabular}{c c c c c c}
\centering
{}
&\multicolumn{2}{c}{\textbf{CIFAR-10}}
&{ }
&\multicolumn{2}{c}{\textbf{ImageNet-100}}
\\
\cline{2-3}\cline{5-6}
\vspace{-1em}
\\
{}
&\includegraphics[width=0.24\textwidth]{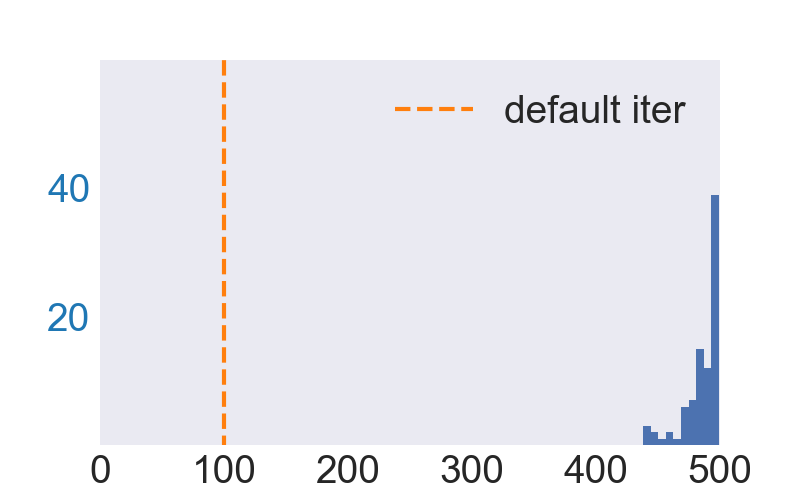}
&\includegraphics[width=0.24\textwidth]{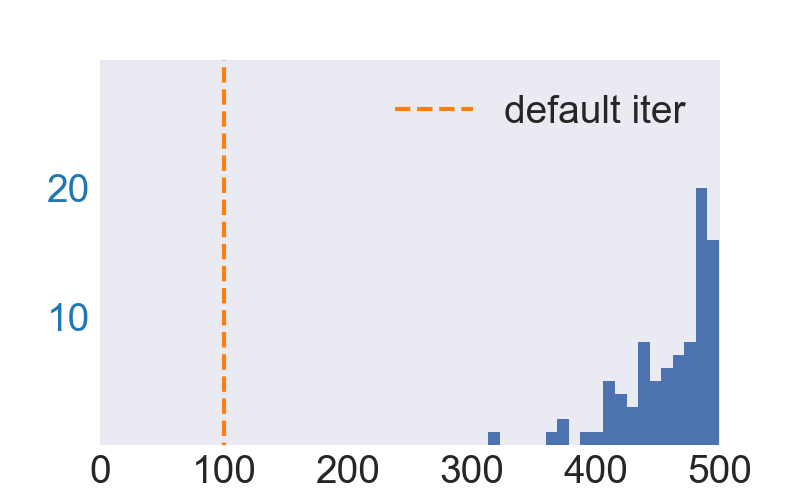}
&{ }
&\includegraphics[width=0.24\textwidth]{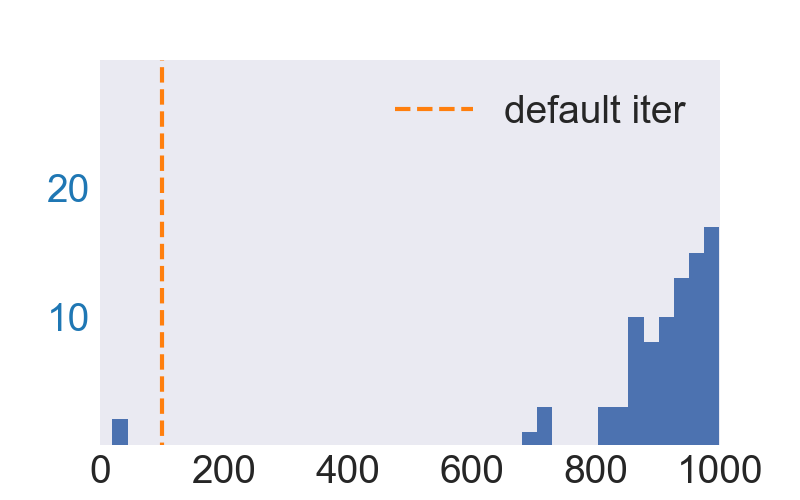}
&\includegraphics[width=0.24\textwidth]{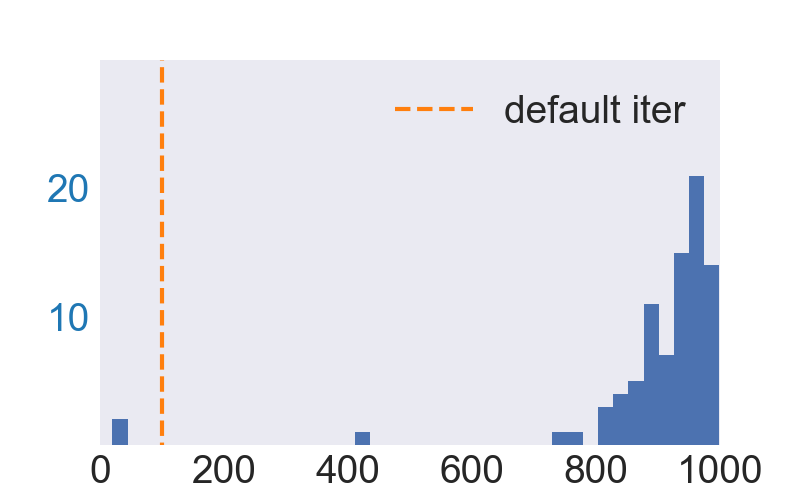}
\\
{}
&\small{APGD - $\ell_2$} 
&\small{APGD - $\ell_\infty$}
&{ }
&\small{APGD - $\ell_2$} 
&\small{APGD - $\ell_\infty$}
\\
{}
&\includegraphics[width=0.24\textwidth]{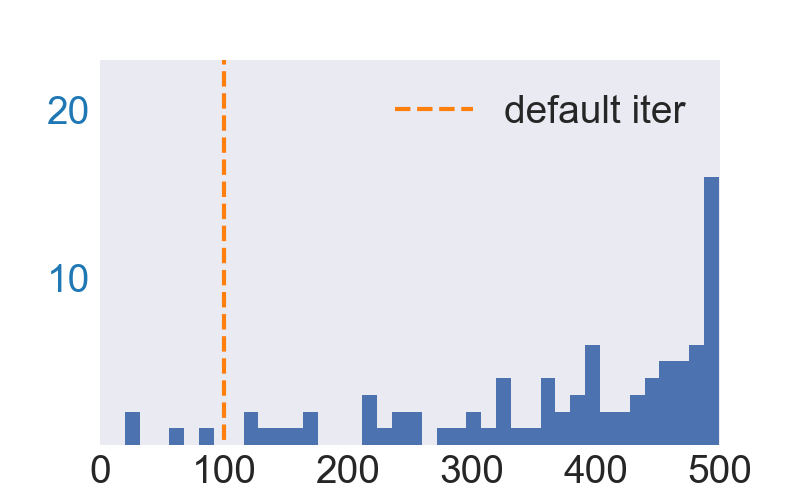}
&\includegraphics[width=0.24\textwidth]{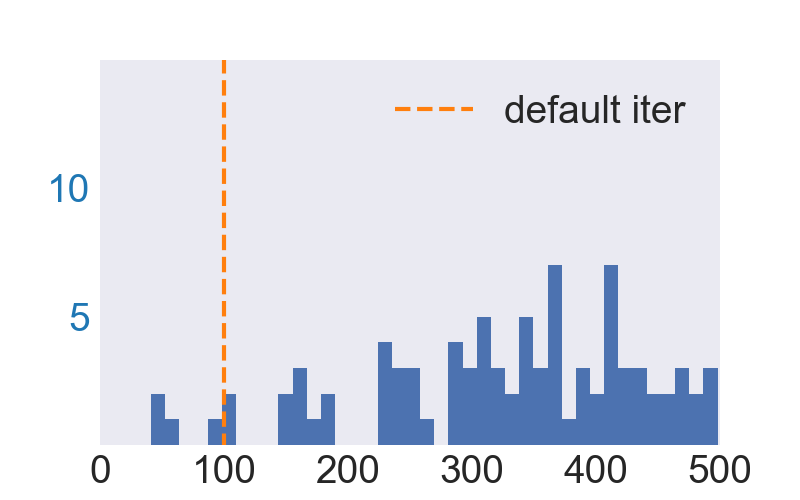}
&{ }
&\includegraphics[width=0.24\textwidth]{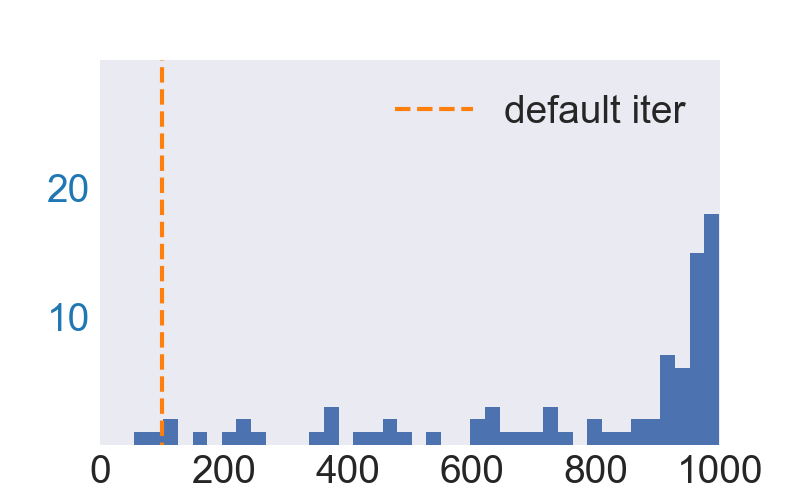}
&\includegraphics[width=0.24\textwidth]{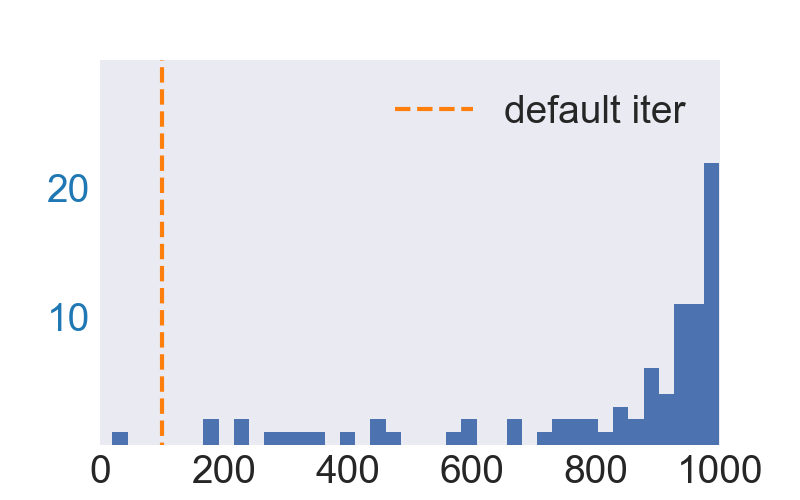}
\\
{}
&\small{FAB - $\ell_2$} 
&\small{FAB - $\ell_\infty$}
&{ }
&\small{FAB - $\ell_2$} 
&\small{FAB - $\ell_\infty$} 
\end{tabular}
\endgroup 
\caption{Histogram of the number of iterations where APGD and FAB (two numerical optimization methods in \texttt{AutoAttack} to solve max-loss form and min-radius form, respectively, which will be introduced in \cref{sec:background}) find the best objective values for $88$ images from CIFAR-10 and $85$ images from ImageNet-100. The dashed \textcolor{orange}{orange} line in each figure represents the default MaxIter used in \texttt{AutoAttack}. In the above experiments, we set the MaxIter to be $500$ for CIFAR-10 and $1000$ for ImageNet-100. We can conclude that 1) the best stopping points for APGD and FAB vary from sample to sample; 2) the default MaxIter used in \texttt{AutoAttack} lead to premature termination; 3) the MaxIter we set here can still be insufficient.} 
\label{Fig:APGD-FAB-Terminate-Iter}
\end{figure*}
\begin{figure*}[!tb]
\vspace{-1em}
\centering
\begingroup 
\setlength{\tabcolsep}{1pt}
\renewcommand{\arraystretch}{0.8}
\begin{tabular}{c c c c c c}
\centering
{}
&\multicolumn{2}{c}{\textbf{CIFAR-10}}
&{ }
&\multicolumn{2}{c}{\textbf{ImageNet-100}}
\\
\cline{2-3}\cline{5-6}
\vspace{-1em}
\\
{}
&\includegraphics[width=0.24\textwidth]{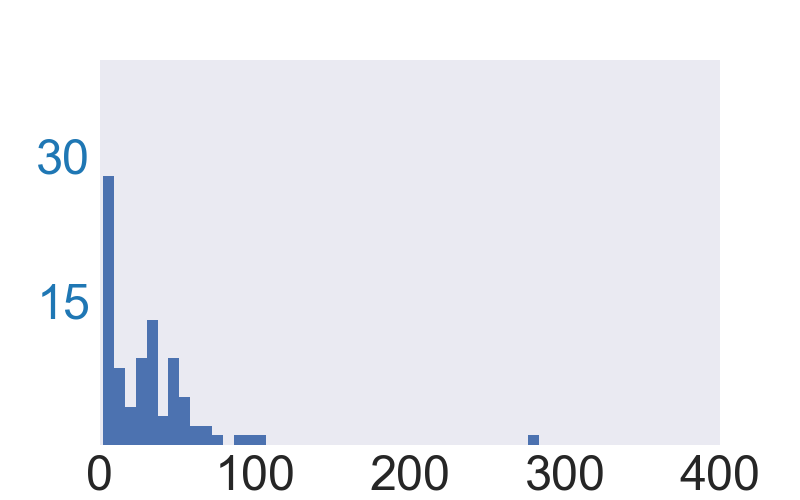}
&\includegraphics[width=0.24\textwidth]{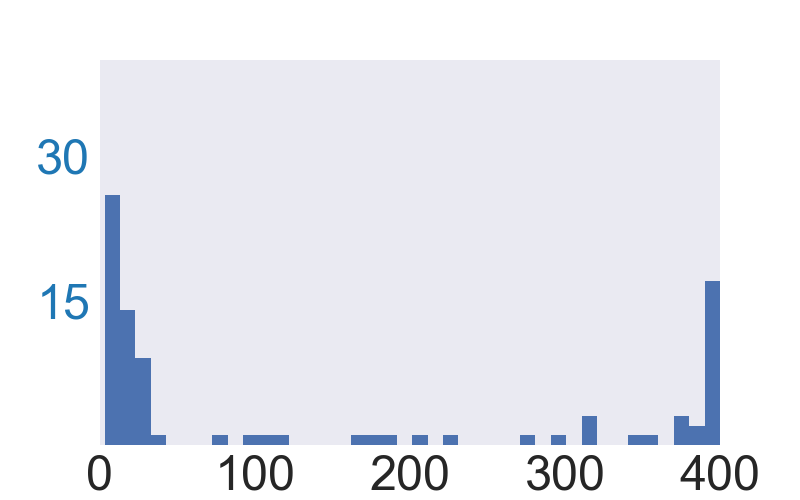}
&{ }
&\includegraphics[width=0.24\textwidth]{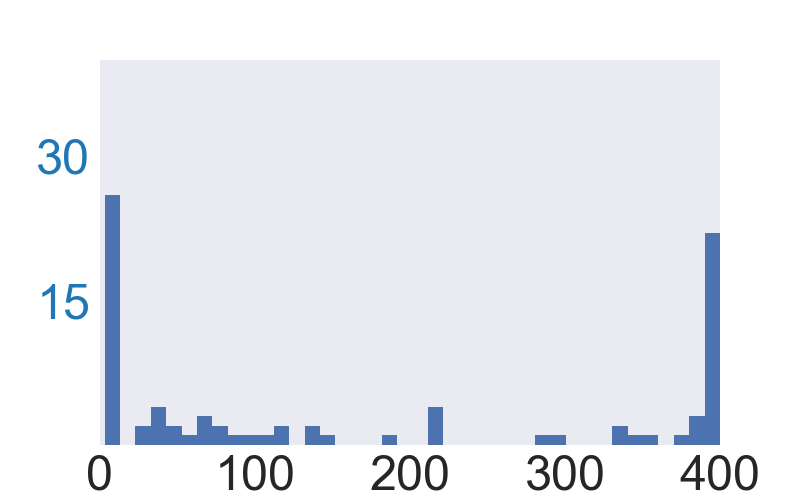}
&\includegraphics[width=0.24\textwidth]{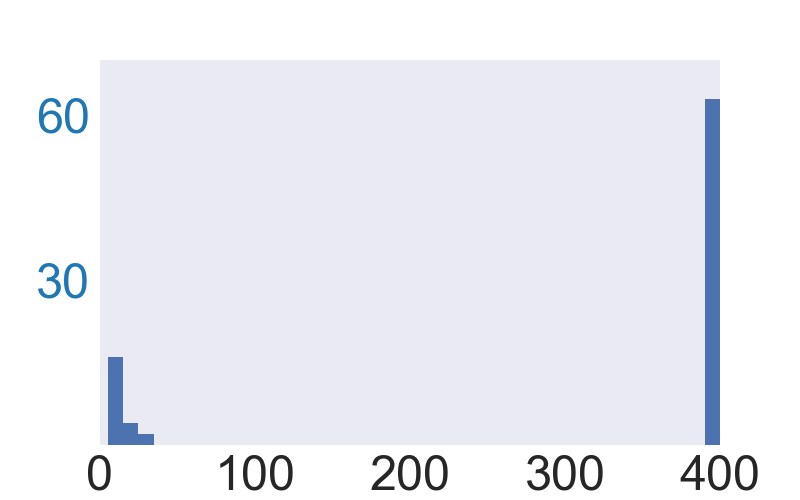}
\\
{}
&\small{\textbf{(a)} PWCF - $\ell_2$} 
&\small{\textbf{(b)} PWCF - $\ell_\infty$}
&{ }
&\small{\textbf{(c)} PWCF - $\ell_2$} 
&\small{\textbf{(d)} PWCF - $\ell_\infty$}
\\
&\includegraphics[width=0.24\textwidth]{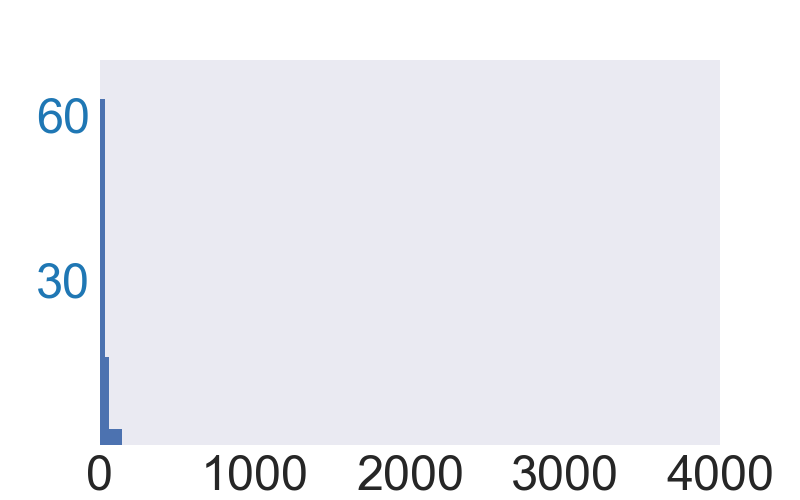}
&\includegraphics[width=0.24\textwidth]{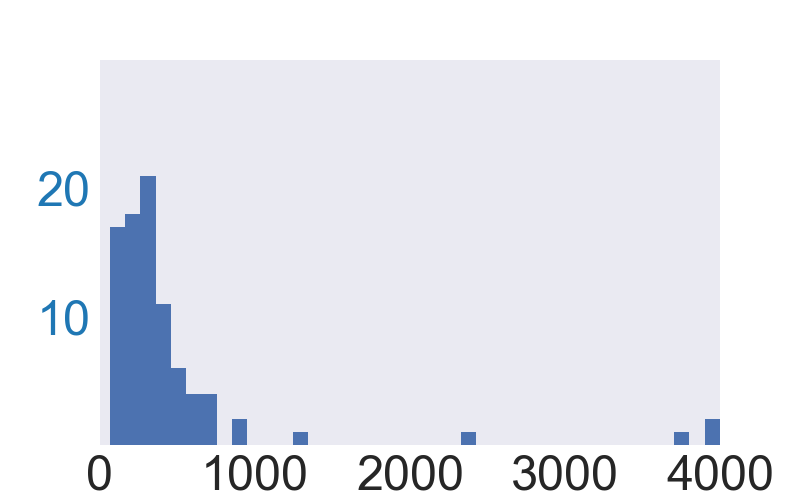}
&{ }
&\includegraphics[width=0.24\textwidth]{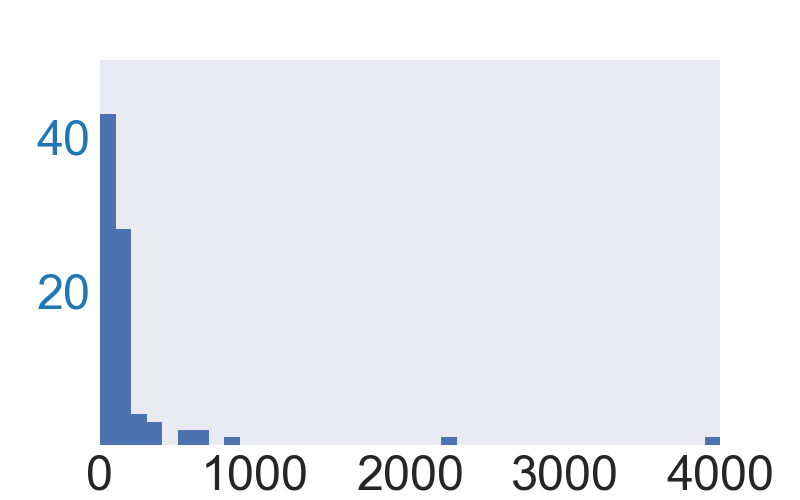}
&\includegraphics[width=0.24\textwidth]{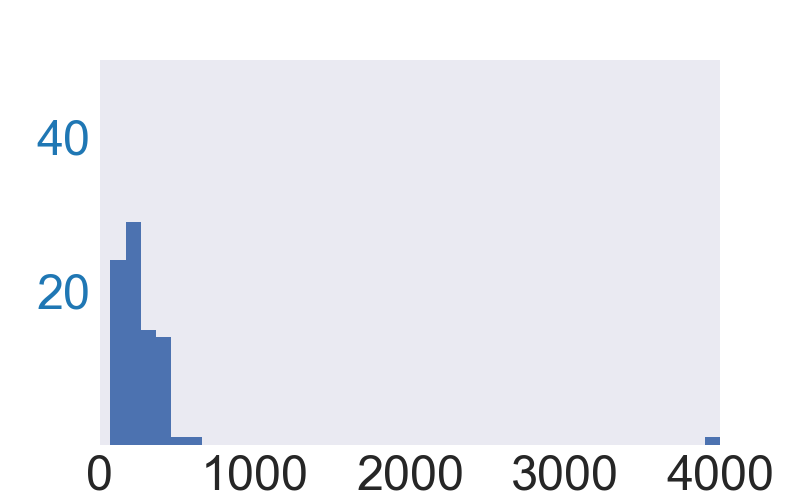}
\\
{}
&\small{\textbf{(e)} PWCF - $\ell_2$} 
&\small{\textbf{(f)} PWCF - $\ell_\infty$}
&{ }
&\small{\textbf{(g)} PWCF - $\ell_2$} 
&\small{\textbf{(h)} PWCF - $\ell_\infty$}
\end{tabular}
\endgroup 
\caption{Histograms of PWCF's termination iteration to solve max-loss form (\textbf{(a)}-\textbf{(d)}), and min-radius form (\textbf{(e)}-\textbf{(h)}) for $88$ images from CIFAR-10 and $100$ images from ImageNet-100. We set both the stationarity and the constraint violation tolerance (components of \pygranso's stopping criterion) as $10^{-2}$. We also set the MaxIter as $400$ for (\textbf{(a)}-\textbf{(d)}) and $4000$ for (\textbf{(e)}-\textbf{(h)}). Optimization processes terminated by reaching MaxIter indicate that the corresponding solutions do not meet the stationarity or feasibility tolerance. For these premature solutions, we can decide if further optimization is needed by assessing the final stationarity, constraint violation values, and the objective values.} 
\label{Fig: PWC-Max-Terminate-Iter}
\end{figure*}

\vspace{1em}
\noindent \textbf{Our contributions}  \quad
In this paper, we focus on the numerical optimization of max-loss form and min-radius form.\footnote{Both formulations have their targeted versions: i.e., replacing $[\max f_{\mb \theta}^i (\mb x')]$ by $f_{\mb \theta}^i (\mb x')$ in (\ref{eq:min_distort}); using $\ell_{\mathrm{ML}}$ in (\ref{eq:robust_loss}). The targeted versions are indeed simpler than the untargeted (\ref{eq:min_distort}) and (\ref{eq:robust_loss}). Thus, in this paper, we will focus only on untargeted ones.} First, we provide a general solver that can handle both formulations with almost everywhere differentiable distance metrics $d$, which also has a principled stopping criterion to assess the reliability of the solution.
\begin{enumerate}[leftmargin=*]
    \item{We adapt the constrained optimization solver~\texttt{PyGRANSO}~\cite{curtis2017bfgs,BuyunLiangSun2021NCVX}~\textbf{w}ith~\textbf{C}onstraint-\textbf{F}olding (PWCF), which is crucial for boosting the speed and solution quality of \pygranso\, in solving the max-loss form and min-radius form problems; see \cref{pygranso general tricks}.}
    
    \item{\pygranso~is equipped with a rigorous line-search rule and stopping criterion. \change{As a result, the reliability of PWCF's solution can be assessed by the constraint violations and stationarity estimate, and users can determine if further optimization is needed; see \cref{Fig: PWC-Max-Terminate-Iter} and \cref{subsec: PWCF techniques demo}.}}
    
    \item{We show that PWCF not only can perform comparably to state-of-the-art RE packages, such as \texttt{AutoAttack}, on solving max-loss form and min-radius form with $\ell_1$, $\ell_2$, and $\ell_\infty$ distances and provide diverse solutions, but also can handle distance metrics $d$ other than the popular but limited $\ell_1$, $\ell_2$, and $\ell_\infty$}---beyond the reach of existing numerical methods; see \cref{Sec: experiments and results}.
\end{enumerate}
We also investigate the sparsity patterns of the solutions found by solving max-loss form and min-radius form using different combinations of solvers, losses, and distance metrics, and discuss the implications:
\begin{enumerate}[leftmargin=*]
    \item Different combinations of distance metrics $d$, losses $\ell$, and optimization solvers can result in different sparsity patterns in the solutions found. In terms of computing robust accuracy, combining solutions with all possible patterns is important to obtain reliable and accurate results; see \cref{sec:pattern_theory} and \cref{subsec: diversity matters for robust accuracy}. 
    \item \change{The robust accuracy at a preset perturbation level $\eps$ used in current RE is often a bad metric to measure robustness. Instead, solving the min-radius form is more beneficial for RE, as the sample-wise robustness radius contains richer information; see \cref{subsec: robust accuracy is bad metric} and \cref{subsec: min radius is better RE metric}.}
    \item{Due to the pattern difference in solving max-loss form by different combinations of solvers, losses, and distance metrics, the common practice of AT with projected gradient descent (PGD) method on a single distance metric may not be able to achieve generalizable AR---DNNs that are adversarially trained may only be robust to the patterns they have seen during training; see \cref{subsec: difficult in achieving AR}.}
\end{enumerate}
\change{Although previous works, e.g., \cite{CarliniEtAl2019Evaluating, croce2020reliable, gilmer2018motivating}, have mentioned the necessity of involving diverse solvers to achieve more reliable RE, our paper is, to the best of our knowledge, the first to quantify the meaning of diversity in terms of the sparsity patterns. \cite{liang2022optimization} is a preliminary version of this work, which only contains the contents related to solving max-loss form numerically by PWCF. This paper expands~\cite{liang2022optimization} with results on solving min-radius form by PWCF, extensive analysis on the solution patterns, and discussions on the implications.}

Preliminary versions of this paper have been published in \cite{liang2022optimization,liang2023implications}. 

\section{Technical background}
\label{sec:background}
\subsection{Numerical optimization for max-loss form}
\label{
subsec: numeric method for robust loss}
Max-loss form (\ref{eq:robust_loss}) is popularly solved by PGD method. The basic update reads:
\begin{align}
    \mb x'_{new} = \mc P_{\Delta(\mb x)} \paren{\mb x'_{old} + \gamma \nabla \ell(\mb x'_{old})}
\end{align} 
where $\mc P_{\Delta(\mb x)}$ is the projection operator to the feasible set $\Delta(\mb x)$ and $\gamma$ is the step size. When $\Delta(\mb x) = \{\mb x' \in [0, 1]^n: \norm{\mb x' - \mb x}_p \le \eps\}$ with $p = 1, \infty$, $\mc P_{\Delta(\mb x)}$ takes simple forms. For $p=2$, sequential projection---first onto the box and then the norm ball, at least finds a feasible solution (see our clarification of these projections in Appendix \ref{Sec:app_projection}). Therefore, PGD is viable for these cases. For other general $\ell_p$ metrics and non-$\ell_p$ metrics, analytical projectors mostly cannot be derived and existing PGD methods do not apply. Previous work has shown that the solution quality of PGD is sensitive to hyperparameter tuning, e.g., step-size schedule and iteration budget~\cite{mosbach2018logit,CarliniEtAl2019Evaluating,croce2020reliable}. PGD variants, such as AGPD methods where `A' stands for Auto(matic)~\cite{croce2020reliable}, try to make the tuning process automatic by combining a heuristic adaptive step-size schedule and momentum acceleration under a fixed iteration budget and are built into the popular \texttt{AutoAttack} package\footnote{Package website: \url{https://github.com/fra31/auto-attack}.}. 
\texttt{AutoAttack} also includes non-PGD methods, e.g., Square Attack~\cite{andriushchenko2020square}, which is based on gradient-free random search. Square Attack does not have as effective attack performance as APGDs (see \cref{tab: granso_l1_acc} in later sections) but is included as \cite{croce2020reliable} states that the diversity in the evaluation algorithms is important to achieve reliable numerical RE.

\subsection{Numerical optimization of min-loss-form}
\label{subsec: numeric method for min_distort}
The difficulty in solving min-loss-form (\ref{eq:min_distort}) lies in dealing with the highly nonlinear constraint $\max_{i \ne y} f_{\mb \theta}^i (\mb x') \ge f_{\mb \theta}^y (\mb x')$. There are two lines of ideas to circumvent it: \textbf{1) penalty methods} turn the constraint into a penalty term and add it to the objective~\cite{szegedy2013intriguing,CarliniWagner2016Towards}. The remaining box-constrained problems can then be handled by optimization methods such as L-BFGS~\cite{wright1999numerical} or PGD. One caveat of penalty methods is that they do not guarantee to return feasible solutions to the original problem; \textbf{2) iterative linearization} linearizes the constraint at each step, leading to simple solutions to the projection (onto the intersection of the linearized decision boundary and the $[0, 1]^n$ box) for particular choices of $d$ (i.e., $\ell_1$, $\ell_2$ and $\ell_\infty$ distances); see, e.g. \cite{MoosaviDezfooliEtAl2015DeepFool,CroceHein2020Minimally}. Again, for general metrics $d$, this projection does not have a closed-form solution. In \cite{rony2019decoupling,PintorEtAl2021Fast}, the problem is reformulated as follows:
\begin{align}
\label{eq: min reform}
\begin{split}
    & \min_{\mb x', t} \, t \\
     \st \;  \max_{i \ne y} & f_{\mb \theta}^i (\mb x') \ge f_{\mb \theta}^y (\mb x')\\
     d(\mb x, & ~ \mb x')  \le t\,  ,\; \mb x' \in [0, 1]^n
    \end{split}
\end{align} 
so that the perturbation (determined by $\mb x'$) and the radius (determined by $t$) are decoupled. Then penalty methods and iterative linearization are combined to solve (\ref{eq: min reform}). The popular fast adaptive boundary (FAB) attack~\cite{CroceHein2020Minimally} included in \texttt{AutoAttack} belongs to iterative linearization family. %

\subsection{\texttt{PyGRANSO} for constrained optimization}
\label{subsec: pygranso for NO}
General \textbf{n}on\textbf{l}inear \textbf{opt}imization (NLOPT) problems take the following form~\cite{wright1999numerical,Bertsekas2016Nonlinear}: 
\begin{align}  
\label{eq:NO_form}
\begin{split}
        & \min_{\mb x}\; g(\mb x) \\
    \st \; c_i&(\mb x) \le 0, \; \forall\;  i \in\mc I\ \\
    h_j&(\mb x) = 0, \; \forall\; j \in \mc E
\end{split}
\end{align}
where $g(\cdot)$ is the continuous objective function; $c_i$'s and $h_j$'s are continuous inequality and equality constraints, respectively.
As instances of NLOPT, both max-loss form and min-radius form in principle can be solved by general-purpose NLOPT solvers such as \texttt{Knitro}~\cite{pillo2006large}, \texttt{IPOPT}~\cite{wachter2006implementation}, and \texttt{GENO}~\cite{laue2019geno,LaueEtAl2022Optimization}. However, there are two caveats: \textbf{1)} these solvers only handle continuously differentiable $g$, $c_i$'s and $h_j$'s, while non-differentiable ones are common in max-loss form and min-radius form, e.g., when $d$ is the $\ell_\infty$ distance, or $f_{\mb \theta}$ uses non-differentiable activations; \textbf{2)} most of them require gradients to be provided by the user, while they are impractical to derive when DNNs are involved. Although gradient derivation can be addressed by combining these solvers with existing auto-differentiation (AD) packages, requiring all components to be continuously differentiable is the major bottleneck that limits these solvers from solving max-loss form and min-radius form reliably.

\pygranso~\cite{BuyunLiangSun2021NCVX}\footnote{Package webpage: \url{https://ncvx.org/}} is a recent \texttt{PyTorch}-port of the powerful MATLAB package \texttt{GRANSO}~\cite{curtis2017bfgs} that can handle general NLOPT problems of form (\ref{eq:NO_form}) with non-differentiable $g$, $c_i$'s, and $h_j$'s. \pygranso~only requires them to be \emph{almost everywhere differentiable}~\cite{Clarke1990Optimization,BagirovEtAl2014Introduction,CuiPang2021Modern}, which covers a much wider range than almost all forms of (\ref{eq:robust_loss}) and (\ref{eq:min_distort}) proposed so far in the literature. \texttt{GRANSO} employs a quasi-Newton sequential quadratic programming (BFGS-SQP) method to solve (\ref{eq:NO_form}), and features a rigorous adaptive step-size rule through a line search and a principled stopping criterion inspired by gradient sampling~\cite{burke2020gradient}; see a sketch of the algorithm in Appendix \ref{Sec:granso_summary} and \cite{curtis2017bfgs} for details. \pygranso\,equips \texttt{GRANSO} with, among other new capabilities, AD and GPU acceleration powered by \texttt{PyTorch}---both are crucial for deep learning problems. \change{The stopping criterion is controlled by the tolerances for total constraint violation and stationarity assessment---the former determines how strict the constraints are enforced, while the latter estimates how close the solution is to a stationary point.} Thus, the solution quality can be transparently controlled.

\subsection{Min-max optimization for adversarial training}
\label{subsec: tecnnical background for min-max}
For a finite training set, AT (\ref{eq:minmax_obj}) becomes:
\begin{align}  
\label{eq:eq:minmax_obj_finite} 
& \min_{\mb \theta} \frac{1}{N} \sum_{i=1}^N  \max_{\mb x'_i \in \Delta(\mb x_i)} \ell\paren{\mb y_i, f_{\mb \theta}(\mb x'_i)} \\
\Longleftrightarrow & \min_{\mb \theta} \max_{\mb x'_i \in \Delta(\mb x_i)\; \forall i}  \frac{1}{N} \sum_{i=1}^N \ell\paren{\mb y_i, f_{\mb \theta}(\mb x'_i)} ~ \text{,}
\label{eq:eq:minmax_obj_finite_eq_form} 
\end{align} 
i.e., in the form of 
\begin{align}
    \min_{\mb \theta} \max_{\mb x'_i \in \Delta(\mb x_i)\; , \forall i} \phi(\mb \theta, \{\mb x_i\}) ~ \text{.}
\label{eq: min_max_practice_form}
\end{align}
(\ref{eq: min_max_practice_form}) is often solved by iterating between the (separable) inner maximization and a subgradient update for the outer minimization. The latter takes a subgradient of $h(\mb \theta) = \max_{\mb x'_i \in \Delta(\mb x_i)\; \forall i} \phi(\mb \theta, \{\mb x_i\})$ from the subdifferential $\partial_{\mb \theta}\phi(\mb \theta, \{\mb x_i^*\})$ ($\mb x_i^*$ are maximizers for the inner maximization), justified by the celebrated Danskin's theorem~\cite{Danskin1967Theory,BernhardRapaport1995theorem,RazaviyaynEtAl2020Non}. However, if the numerical solutions to the inner maximization are substantially suboptimal\footnote{On the other hand, \cite{madry2017towards} argues using numerical evidence that maximization landscapes are typically benign in practice and gradient-based methods often find solutions with function values close to the global optimal. }, the subgradient to update the outer minimization can be misleading; see our discussion in \cref{sec:danskin_minmax}. In practice, people generally do not strive to find the optimal solutions to the inner maximization but prematurely terminate the optimization process, e.g., \change{by a small MaxIter.}

\section{A generic solver for max-loss form and min-radius form: \pygranso~With Constraint-Folding (PWCF)}
\label{Sec:pygranso}
Although~\pygranso~can handle max-loss form and min-radius form, naive deployment of \pygranso~can suffer from slow convergence and low-quality solutions. To address this, we introduce \pygranso~\textbf{w}ith \textbf{C}onstraint-\textbf{F}olding (PWCF) to substantially speed up the optimization process and improve the solution quality. 

\subsection{General techniques}
\label{pygranso general tricks}
\begin{figure}[!tb]
\centering
\begingroup 
\setlength{\tabcolsep}{1pt}
\renewcommand{\arraystretch}{0.8}
\begin{tabular}{cc}
\centering
\includegraphics[width=0.24\textwidth]{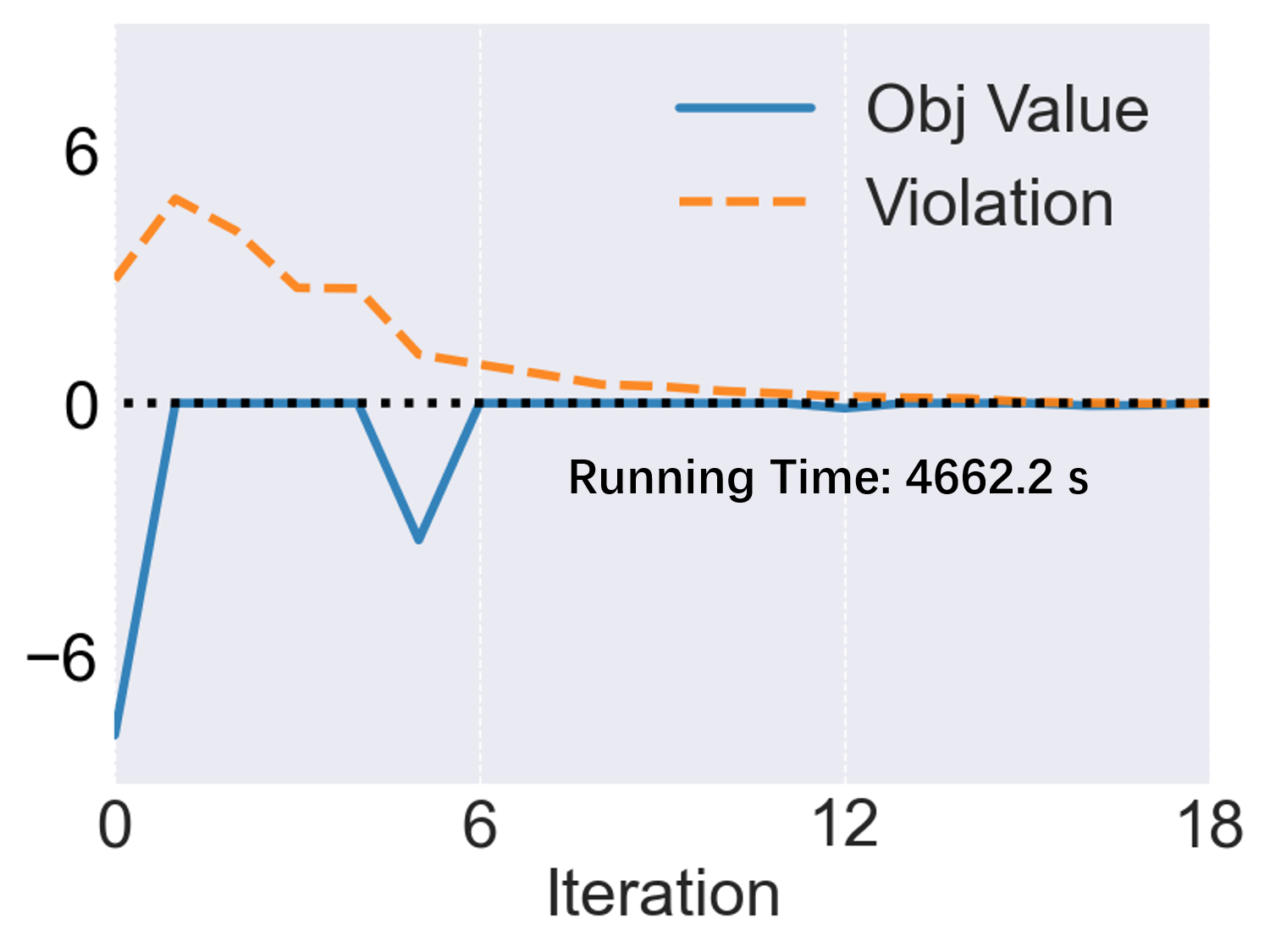}
&\includegraphics[width=0.24\textwidth]{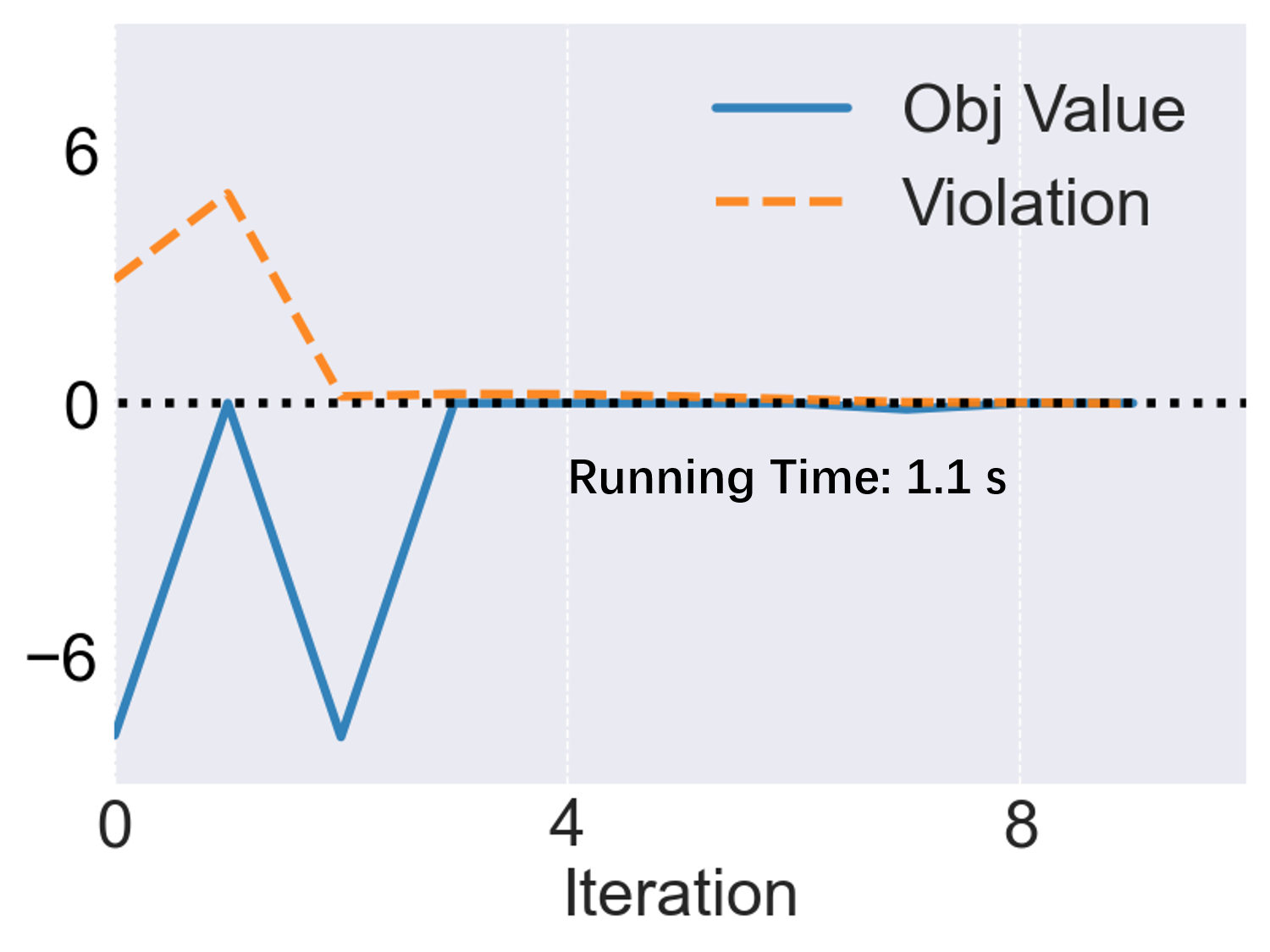}
\\
\small{\textbf{(a)} $n$ box constraints} 
&\small{\textbf{(b)} folded constraints} 
\end{tabular}
\endgroup 
\caption{Examples of \pygranso~optimization trajectories to solve max-loss form with $\ell_{2}$ distance $(\eps=0.5)$ and margin loss $\ell$ (a clipped version described in \cref{subsec: loss clip}) on a CIFAR-10 image. \textbf{(a)} the $\mb x' \in [0, 1]^n$ constraint is in the original form of $n$ linear box constraints; \textbf{(b)} $\mb x' \in [0, 1]^n$ constraint is folded with the $\ell_2$ function into a single non-smooth constraint function. The x-axes denote the iteration number. Here, an acceptable solution is found when the objective value reaches $0.01$ and the constraint violation reaches below the tolerance level ($10^{-2}$). We can conclude from (a) and (b) that it takes significantly less time and number of iterations with constraint-folding than without.} 
\label{Fig: Ablation on constraint folding}
\end{figure}
The following techniques are developed for solving max-loss form and min-raidus form, but can also be applied to other NLOPT problems in similar situations.
\subsubsection{Reducing the number of constraints: constraint-folding}
\label{subsec: folding}
The natural image constraint $\mb x' \in [0, 1]^n$ is a set of $n$ box constraints. The reformulations described in \cref{subsec: reformulate Linf constraint} and \cref{subsec: reformulate l1 and linf obj} will introduce another $\Theta(n)$ box constraints. Although all of these are simple linear constraints, the $\Theta(n)$-growth is daunting: for natural images, $n$ is the number of pixels that can easily go into hundreds of thousands. Typical NLOPT problems become much more difficult when the number of constraints becomes large, which can, e.g., lead to slow convergence for numerical algorithms. 

\change{To address this, we introduce constraint-folding to reduce the number of constratins by allowing multiple constraints to be turned into a single one. We note that folding or aggregating constraints is not a new idea, which has been popular in engineering design. For example, \cite{martins2005structural} uses $\ell_\infty$ folding and its log-sum-exponential approximation to deal with numerous design constraints; see~\cite{ZhangEtAl2018Constraint,DomesNeumaier2014Constraint,ErmolievEtAl1997Constraint,TrappProkopyev2015note}. However, applying folding to NLOPT problems in machine learning and computer vision seems rare, potentially because producing non-smooth constraint(s) due to folding seems counterproductive. However, \pygranso~is able to handle non-smooth functions reliably. Thus, constraint folding can enjoy benefiting from reducing the difficulty and expense of solving the two types of quadratic programming (QP) subproblems ((\ref{penalty_sqp_dual}) and (\ref{QP_termination}) in \cref{Sec:granso_summary}) in each iteration of \pygranso, which are harder to solve as the number of constraints increases. Although turning multiple constraints into fewer but non-smooth ones can possibly increase the per-iteration cost, we show in this paper that there indeed can be a beneficial trade-off between non-smoothness and large number of constraints, in terms of speeding up the entire optimization process.}

\change{As for the implementation of the constraint-folding, first note that any equality constraint $h_j(\mb x) = 0$ or inequality constraint $c_i(\mb x) \le 0$ can be reformulated as 
\begin{align} \label{eq:simple_constr_form}
    \begin{split}
            & h_j (\mb x) = 0 \Longleftrightarrow \abs{h_j(\mb x)} \le 0 ~ \text{,}\\ 
            & c_i(\mb x) \le 0 \Longleftrightarrow \max\{c_i(\mb x), 0\} \le 0 ~ \text{,}
    \end{split}
\end{align} 
Then we can further fold them together into 
\begin{equation} 
    \label{eq:folded_constraint} 
    \begin{split}
        \mc F(& \abs{h_1(\mb x)}, \cdots, \abs{h_i(\mb x)}, \max\{c_1(\mb x), 0\}, \\
        & \cdots, \max\{c_j(\mb x), 0\}) \le 0,
    \end{split}
\end{equation}
where $\mc F: \RJU^{i+j}_{+} \mapsto \RJU_+$ ($\RJU_+ = \{\alpha: \alpha \ge 0\}$) can be any function satisfying $\mc F(\mb z) = 0 \Longrightarrow \mb z = \mb 0$, e.g., any $\ell_p$ ($p \ge 1$) norm, and (\ref{eq:folded_constraint}) and (\ref{eq:simple_constr_form}) still shares the same feasible set.}

The constraint folding technique can be used for a subset of constraints (e.g., constraints grouped and folded by physical meanings) or all of them. 
\change{Throughout this paper, we use $\mc F = \norm{\cdot}_2$ for constraint-folding, and the constraints are folded by group (they are: $\mb x' \in [0, 1]^n$, distance metric $d$ and decision boundary constraint, respectively.)} \cref{Fig: Ablation on constraint folding} shows the benefit of constraint folding with an example of max-loss form, where the time efficiency is greatly improved while the solution quality remains.

\subsubsection{Two-stage optimization}
\label{subsec: pygranso early stop}
Numerical methods may converge to poor local minima for NLOPT problems. Running the optimization multiple times with different random initializations is an effective and practical way to overcome this problem. For example, each method in \texttt{AutoAttack} by default runs five times and ends with the preset MaxIter. Here, we apply a similar practice to PWCF, but in a two-stage fashion:
\begin{enumerate}[leftmargin=*]
    \item \textbf{Stage 1 (selecting the best initialization):} Optimize the problems by PWCF with $R$ different random initialization $\mb x^{(r, 0)}$ for $k$ iterations, where $r=1, \ldots, R$, and collect the final first-stage solution $\mb x^{(r, k)}$ for each run. Determine the best intermediate result $\mb x^{(*, k)}$ and the corresponding initialization $x^{(*, 0)}$ following \cref{alg:2-stage screening}.
    \item{\textbf{Stage 2 (optimization):} Restart the optimization process with $x^{(*, 0)}$ until the stopping criterion is met\footnote{This is equivalent to warm start (continue optimization) with the intermediate result $\mb x^{*, k}$ with the corresponding Hessian approximation stored. To make the warm start follow the exact same optimization trajectory, the configuration parameter "scaleH0" for \pygranso\, must be turned off, which is different from the default value.} (reaching the tolerance level of stationarity and total constraint violation, or reaching the MaxIter $K$).
    }
\end{enumerate}

\change{The purpose of the two-stage optimization is simply to avoid PWCF spending too much time searching on unpromising optimization trajectories---after a reasonable number of iterations, PWCF tends to only refine the solutions; see (b) and (d) in \cref{Fig: Abaltion-OPT-Traj} later for an example, where PWCF has reached solutions with reasonable quality in very early stages and only improves the solution quality marginally afterward. Picking reasonably good values of $R$, $k$ and $K$ can be simple, e.g., using a small subset of samples and tuning them empirically. We present our result in \cref{Sec:pygranso} using $k=20$ and $K=400$ for max-loss form, $k=50$ and $K=4000$ for min-radius form and $R=10$ for both formulations, as PWCF can produce solutions with sufficient quality in the experiments in \cref{Sec:pygranso}.}

\begin{algorithm}[!tb]
\caption{Selection of $x^{(*, k)}$ and $x^{(*, 0)}$ in the two-stage process}
\label{alg:2-stage screening}
\begin{algorithmic}[1]
\Require Initialization $x^{(r, 0)}$ and the corresponding intermediate optimization results $x^{(r, k)}$.

\If{Any $x^{(r, k)}$ is feasible for \formulation (\ref{eq:NO_form})}
\State Set $x^{(*, k)}$ to be the feasible $x^{(r, k)}$'s with the least objective value.
\Else
\State Set $x^{(*, k)}$ to be the $x^{(r, k)}$ with the least constraint violation.
\EndIf 
\State Set $x^{(*, 0)}$ corresponds to $x^{(*, k)}$ found.
\State \Return{$x^{(*, k)}$ and $x^{(*, 0)}$.}
\end{algorithmic}
\end{algorithm}

\subsection{Techniques specific to max-loss form and min-radius form}
\label{pygranso specific techniques}
In addition to the general techniques above, the following techniques can also help improve the performance of PWCF in solving (\ref{eq:robust_loss}) and (\ref{eq:min_distort}).

\begin{figure}[!tb]
\centering
\begingroup 
\setlength{\tabcolsep}{1pt}
\renewcommand{\arraystretch}{0.8}
\begin{tabular}{cc}
\centering
\includegraphics[width=0.24\textwidth]{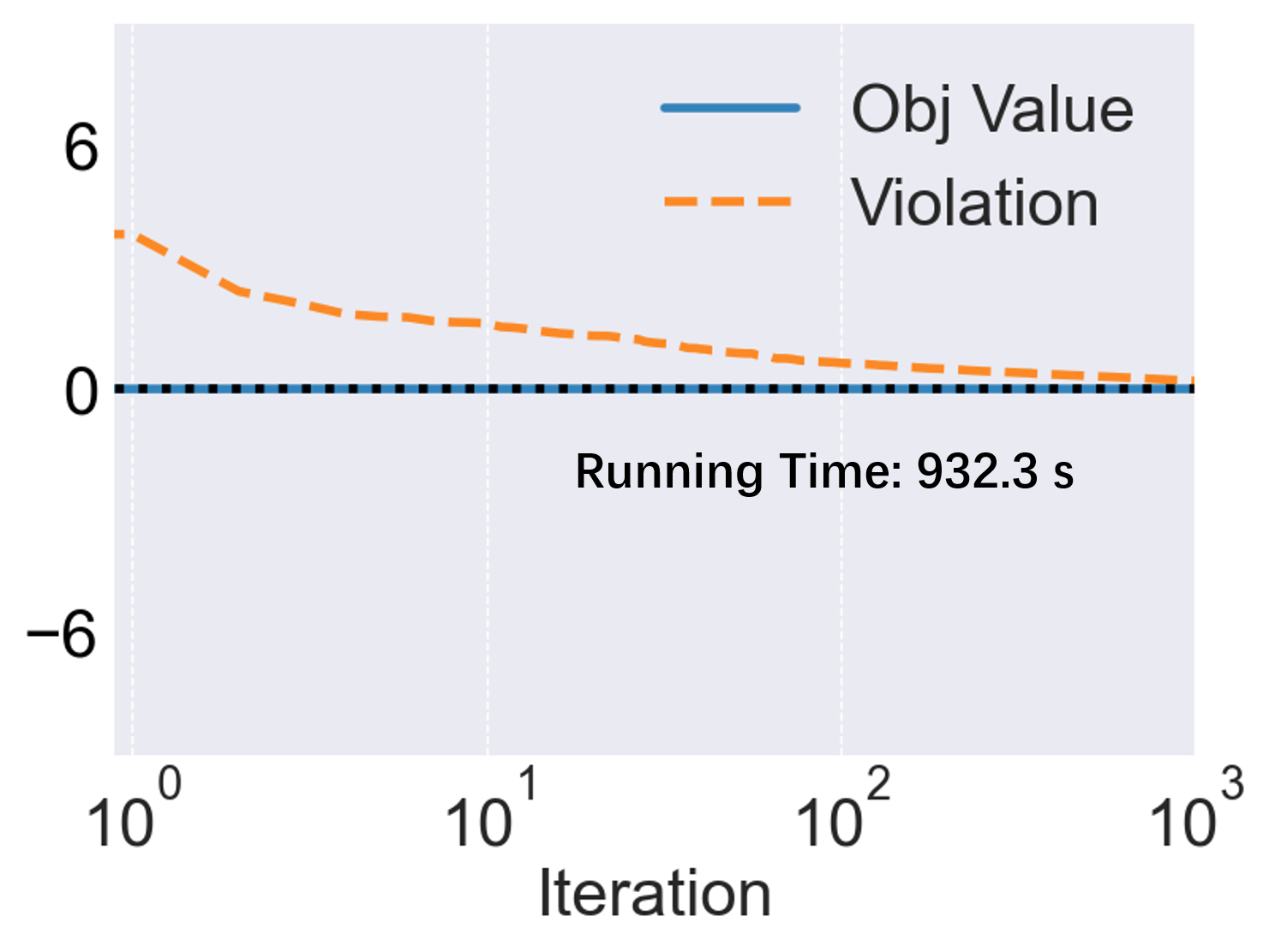}
&\includegraphics[width=0.24\textwidth]{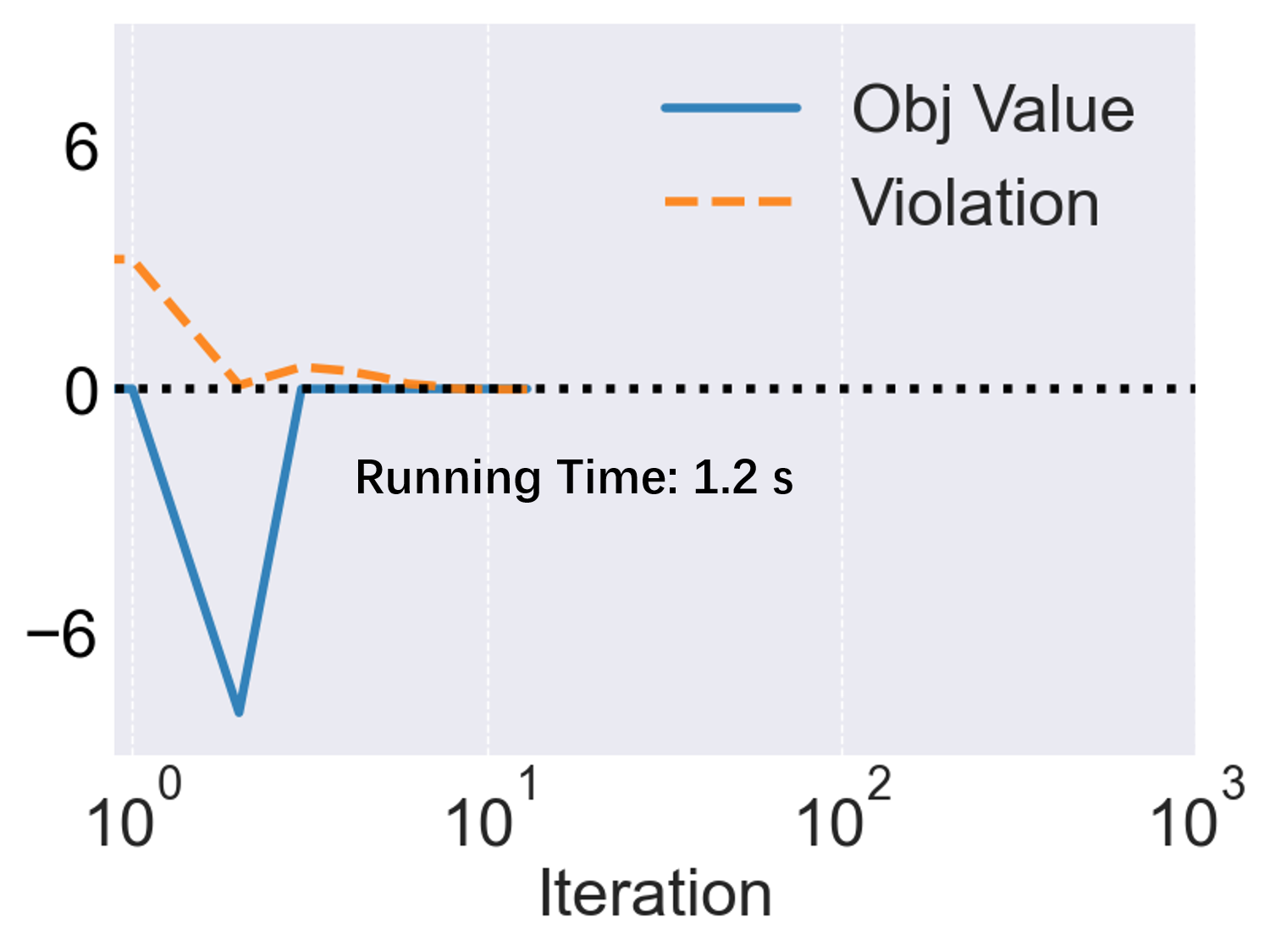}
\\
\small{\textbf{(a)} original $\ell_\infty$} 
&\small{\textbf{(b)} reformulated $\ell_\infty$} 
\end{tabular}
\endgroup 
\caption{Example of \pygranso~optimization trajectories to solve max-loss form (\ref{eq:robust_loss}) with the $\ell_{\infty}$ distance $(\eps=0.03)$ and margin loss $\ell$ (a clipped version described in \cref{subsec: loss clip}) on a CIFAR-10 image. The original form of constraint $\norm{\mb x' - \mb x}_\infty \leq \eps$ is used in \textbf{(a)}, while the reformulated and folded version of constraint is used in \textbf{(b)}. The x-axes denote the iteration number. Here, the optimization is terminated when the constraint violation is smaller than $10^{-2}$. After reformulating and folding the $\ell_\infty$ constraint, the optimization process in \textbf{(b)} runs much faster in terms of both time and iterations needed.} 
\label{Fig: Ablation on Linf Reform}
\end{figure}

\subsubsection{Avoiding sparse subgradients: reformulating $\ell_\infty$ constraints}
\label{subsec: reformulate Linf constraint}
\change{The BFGS-SQP algorithm inside \pygranso~uses the subgradients of the objective and the constraint functions to approximate the (inverse) Hessian of the penalty function, which is used to compute search directions. For the $\ell_\infty$ distance, the subdifferential (the set of subgradients) is:
\begin{align} 
\label{eq:subgrad_linf}
\nonumber
    \partial_{\mb z}  \norm{\mb z}_\infty = \conv \{\mb e_k \sign(z_k): z_k =  \norm{\mb z}_\infty \; \forall\, k\}
\end{align} 
where $\mb e_k$'s are the standard basis vectors, $\conv$ denotes the convex hull, and $\sign(z_k) = z_k/\abs{z_k}$ if $z_k \ne 0$, otherwise $[-1, 1]$.
Any subgradient within the subdifferential set contains no more than $n_k = \abs{\{k: z_k = \norm{\mb z}_\infty\}}$ nonzeros, and is sparse when $n_k$ is small. When all subgradients are sparse, only a handful of optimization variables may be updated on each iteration, which can result in slow convergence. To speed up the overall optimization process, we propose the reformulation:
\begin{equation}
    \label{eq: Linf to box}
    \norm{\mb x - \mb x'}_\infty \le \eps \Longleftrightarrow -\eps \mb 1 \le \mb x - \mb x' \le \eps \mb 1,
\end{equation}
where $\mb 1 \in \RJU^n$ is the all-ones vector. Then, the resulting $n$ box constraints can be folded into a single constraint as introduced in \cref{subsec: folding} to improve efficiency; see \cref{Fig: Ablation on Linf Reform} for an example of such benefits.}

\begin{figure}[!tb]
\centering
\includegraphics[width=0.4\textwidth]{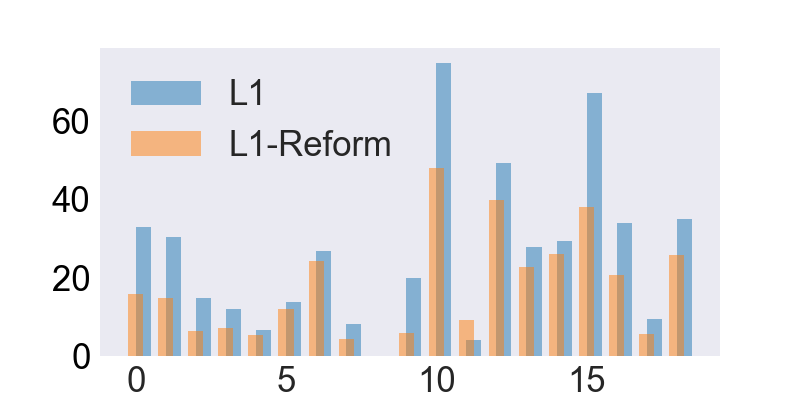}
\caption{Robustness radii found by solving min-radius from with $\ell_1$ distance in the original formulation (\ref{eq:min_distort}) (\textcolor{blue}{blue}) and in the reformulated version (\ref{eq: l1 min form reformualtion}) (\textcolor{orange}{orange}) on $18$ CIFAR-10 images. The x-axis denotes the sample index and the y-axis denotes the radius. As all results have reached the feasibility tolerance, the lower the radius found, the more effective the solver is at handling the optimization problem.} 
\label{fig:L1 Reform Ablation} 
\end{figure}

\subsubsection{Decoupling the update direction and the radius: reformulating $\ell_1$ and $\ell_\infty$ objectives}
\label{subsec: reformulate l1 and linf obj}
It is not surprising that for min-radius form with $\ell_\infty$ distance, it is more effective to solve the reformulated version (\ref{eq: min reform}) than the original one (\ref{eq:min_distort})---(\ref{eq: min reform}) moves the $\ell_\infty$ distance into the constraint, thus allowing us to apply constraint folding technique described in \cref{subsec: reformulate Linf constraint}. In practice, we also find that solving the reformulated version below is more effective when $d$ is the $\ell_1$ distance in min-radius form:
\begin{align} 
\label{eq: l1 min form reformualtion}
   \begin{split}
       & \min_{\mb x'\, , \mb t} \; \mb{1}^{\TJU}\mb t\\
    \st \; & \max_{i \ne y} f_{\mb \theta}^i (\mb x') \ge f_{\mb \theta}^y (\mb x') \\
    & \abs{\mb x_i - \mb x'_i} \leq t_i\, , \quad i = 1, 2, \cdots, n \\
    & \mb x' \in [0, 1]^n.
   \end{split}
\end{align}
where $t_i$ is the $i^{th}$ element of $\mb t$. The newly introduced box constraint $\abs{\mb x_i - \mb x'_i} \leq t_i$ is then folded as described in \cref{subsec: folding}. \cref{fig:L1 Reform Ablation} compares the robustness radius found by solving (\ref{eq:min_distort}) and (\ref{eq: l1 min form reformualtion}) on the same DNN model with $\ell_1$ distance on $18$ CIFAR-10 images, respectively. The radius found by solving (\ref{eq: l1 min form reformualtion}) are much smaller in all but one sample than by solving (\ref{eq:min_distort}), showing that solving (\ref{eq: l1 min form reformualtion}) with constraint-folding is more effective than solving the original form (\ref{eq:min_distort}).

\begin{figure}[!tb]
\centering
\begingroup 
\setlength{\tabcolsep}{1pt}
\renewcommand{\arraystretch}{0.8}
\begin{tabular}{cc}
\centering
\includegraphics[width=0.245\textwidth]{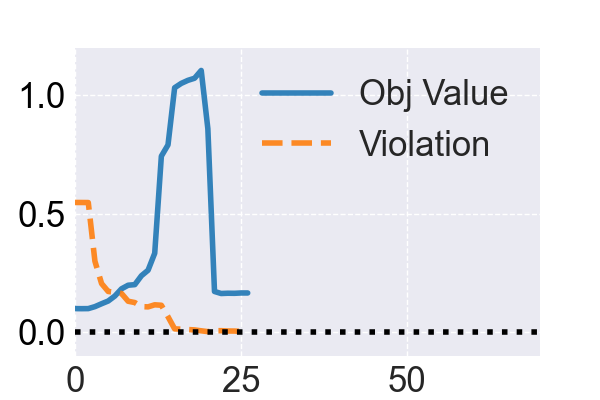}
&\includegraphics[width=0.245\textwidth]{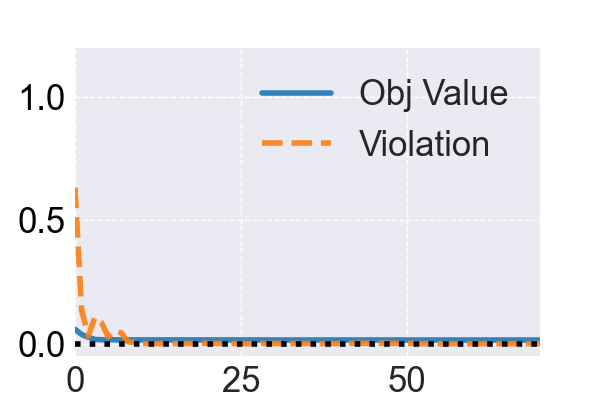}
\\
\small{\textbf{(a)} $\min~t$} 
&\small{\textbf{(b)} $\min~t \cdot \sqrt{n}$} 
\end{tabular}
\endgroup 
\caption{Examples of PWCF optimization trajectories for solving min-radius form (using (\ref{eq: min reform})) with the $\ell_\infty$ metric on a CIFAR-10 image \textbf{(a)} without rescaling and \textbf{(b)} with rescaling. The x-axes are the iteration number. The objective value in \textbf{(b)} is scaled back to the original value $t$ for fair comparisons with \textbf{(a)}. In Figure \textbf{(a)}, optimization is terminated around the $25^{\text{th}}$ iteration due to line-search failure, and the final solution has a much higher (worse) objective value than \textbf{(b)}. Here, we use $10^{-8}$ for the stationarity and constraint violation tolerances to rule out the possibility that the bad solution quality of \textbf{(a)} is due to premature termination.}
\label{Fig: Ablation on Min Linf Rescale}
\end{figure}

\subsubsection{Numerical re-scaling to balance objective and constraints}
\label{subsec: granso resscale}
The steering procedure for determining search directions in \pygranso~(see line 5 in \cref{{alg:steering}}, \cref{Sec:granso_summary}) can only successively decrease the influence of the objective function in order to push towards feasible solutions. Therefore, if the scale of the objective value is too small compared to the initial constraint violation value, numerical problems can arise---\pygranso~will try hard to push down the violation amount, while the objective hardly decreases. This can occur when solving min-radius form with the $\ell_\infty$ distance, using the reformulated version (\ref{eq: min reform}). The objective $t$ is expected to have the order of magnitude $10^{-2}$ while the folded constraints are the $\ell_2$ norm of a $n$-dimensional vector (e.g., $n = 3 \times 32 \times 32 = 3072$ for a CIFAR-10 image). To address this, we simply rebalance the objective by a constant scalar---we minimize $t\cdot\sqrt{n}$ instead of $t$, which can help PWCF perform as effectively as in other cases; see an example in \cref{Fig: Ablation on Min Linf Rescale}.

\begin{figure}[!tb]
\centering
\begingroup 
\setlength{\tabcolsep}{1pt}
\renewcommand{\arraystretch}{0.8}
\begin{tabular}{ccc}
\centering
\textbf{\small{loss}}
&{ }
&\textbf{\small{gradient magnitude}}
\\
\cline{1-1}\cline{3-3}
\vspace{-1em}
\\
\includegraphics[width=0.24\textwidth]{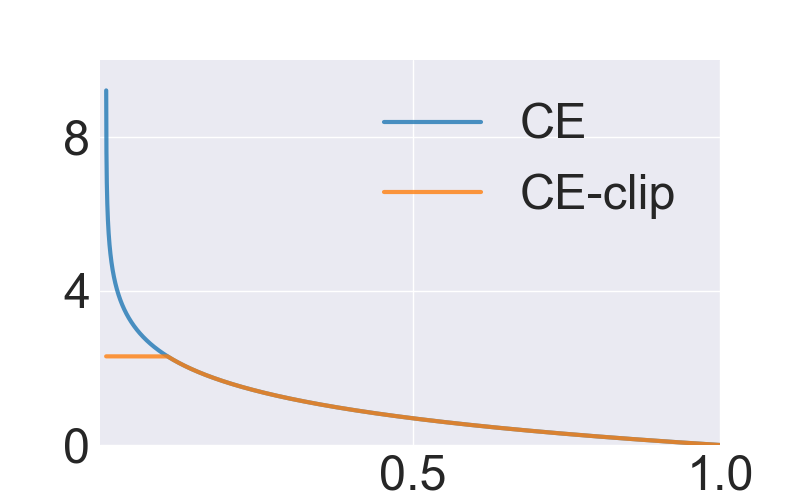}
&{ }
&\includegraphics[width=0.24\textwidth]{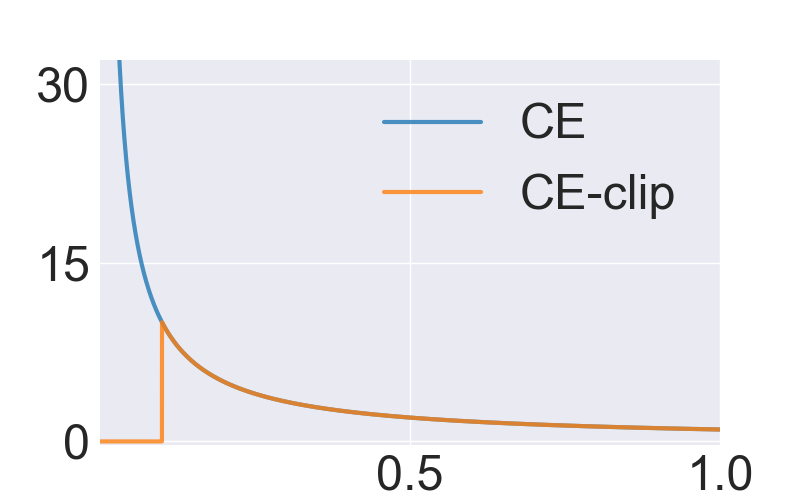}
\\
\textbf{\small{(a)}}
&{ }
&\textbf{\small{(b)}}
\\
\includegraphics[width=0.24\textwidth]{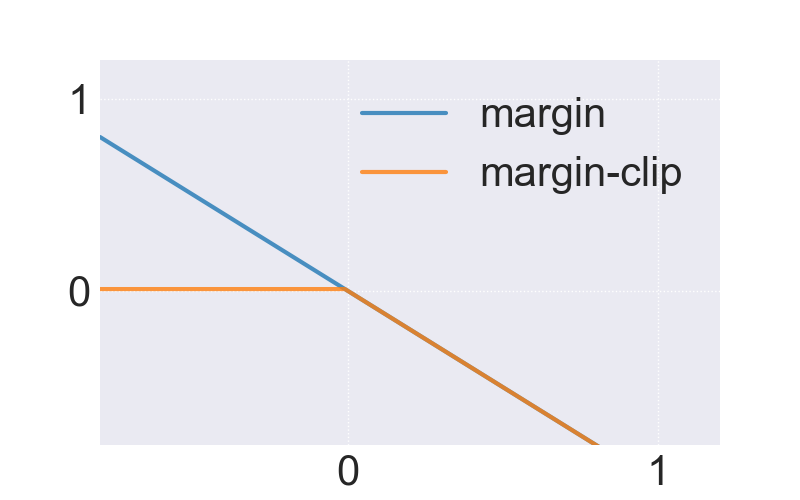}
&{ }
&\includegraphics[width=0.24\textwidth]{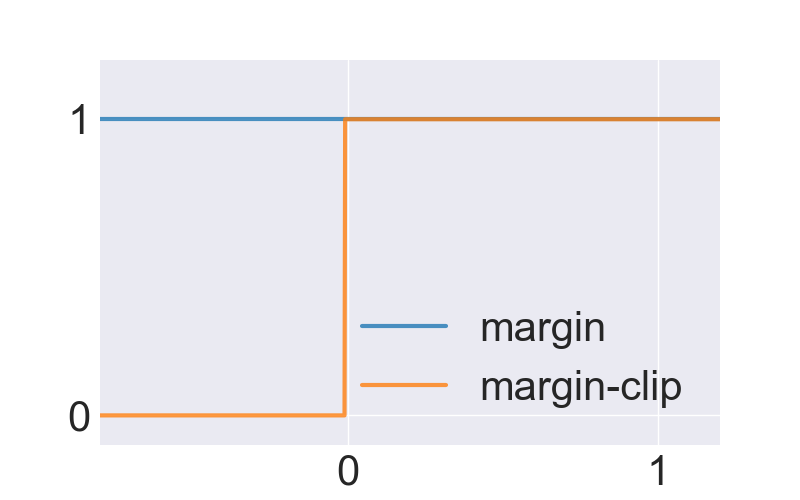}
\\
\textbf{\small{(c)}}
&{ }
&\textbf{\small{(d)}}
\end{tabular}
\endgroup 
\caption{Visualizations of loss clipping. The cross-entropy loss and its clipped version is shown in \textbf{(a)} and the corresponding norm of gradients is shown in \textbf{(b)}. The clipped version shown in \textbf{(a)} is the one used in the CIFAR-10 experiments. The x-axes in \textbf{(a)} and \textbf{(b)} are the network output value $f_{\mb \theta}^y (\mb x')$ after softmax regularization.The margin loss and its clipped version is shown in \textbf{(c)}, and the corresponding norm of gradients is shown in\textbf{(d)}. The x-axes in \textbf{(c)} and \textbf{(d)} are the value $\max_{i \ne y} f_{\mb \theta}^i (\mb x') - f_{\mb \theta}^y (\mb x')$ before the softmax regularization, which follows the definition of $\ell_{ML}$ in \cref{eq: margin loss}.} 
\label{fig:loss_clipping} 
\end{figure}

\subsubsection{Loss clipping in solving max-loss form with PWCF}
\label{subsec: loss clip}
When solving  max-loss form using the popular cross-entropy (CE) and margin losses as $\ell$ (both are unbounded in maximization problems, see \cref{fig:loss_clipping}), the objective value and its gradient can easily dominate over making progress towards pushing down the constraint violations. \change{While \pygranso~tries to make the best joint progress of these two components, PWCF can still persistently prioritize maximizing the objective over constraint satisfaction, which can lead to bad scaling problems in the early stages and result in slow progress towards feasibility. To address this, we propose using the margin and CE losses with clipping at critical maximum values. For margin loss, the maximum value is clipped at $0.01$, since $\ell_{\mathrm{ML}} > 0$ indicates a successful attack. Similarly, CE loss can be clipped as follows: the attack success must occur when the true logit output is less than $1/N_c$ (after softmax normalization is applied), where $N_c$ is the number of classes. The corresponding critical value is thus $\ln N_c$---approximately $2.3$ when $N_c=10$ (the number of total classes for the CIFAR-10 dataset) and $4.6$ when $N_c=100$ (for ImageNet-100\footnote{ImageNet-100~\cite{laidlaw2021perceptual} is a subset of ImageNet, where samples with label index in $\{0, 10, 20, \cdots, 990\}$ are selected. ImageNet-100 validation set thus contains in total $5000$ images with $100$ classes.} dataset). Loss clipping significantly speeds up the optimization process of PWCF to solve max-loss form; see \cref{Fig: ablation loss clipping} for an example.}

\begin{figure}[!tb]
\centering
\begingroup 
\setlength{\tabcolsep}{1pt}
\renewcommand{\arraystretch}{0.8}
\begin{tabular}{cc}
\centering
\includegraphics[width=0.24\textwidth]{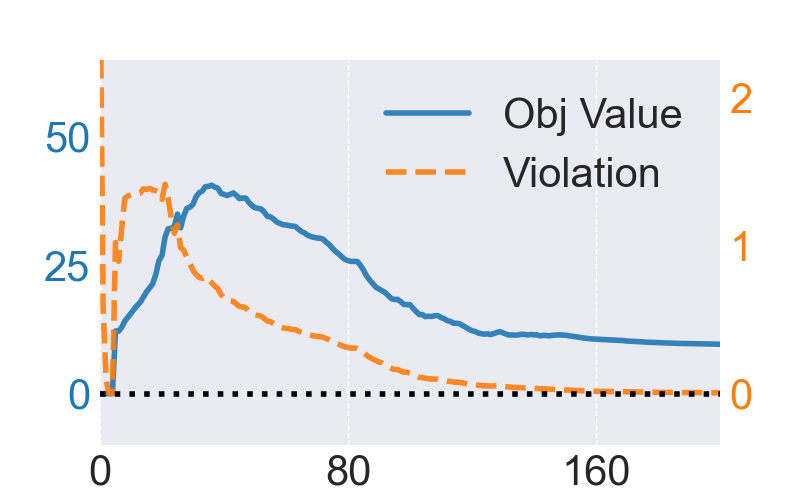}
&\includegraphics[width=0.24\textwidth]{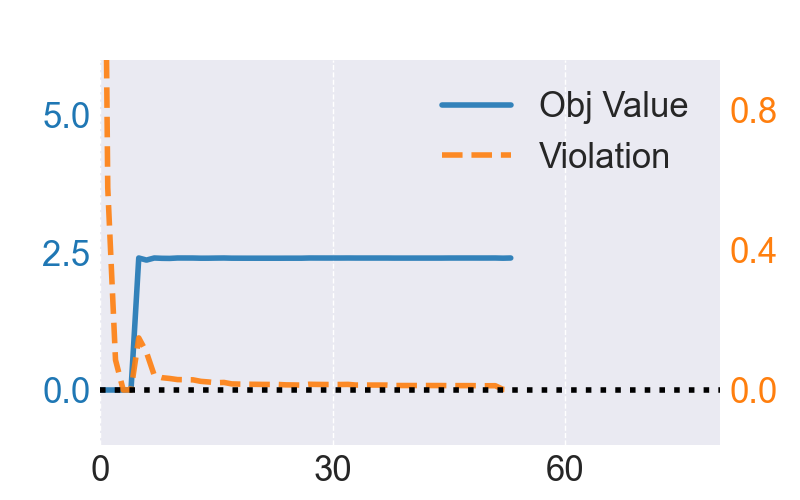}
\\
\small{\textbf{(a)} CE} 
&\small{\textbf{(b)} CE-clip}
\\
\includegraphics[width=0.24\textwidth]{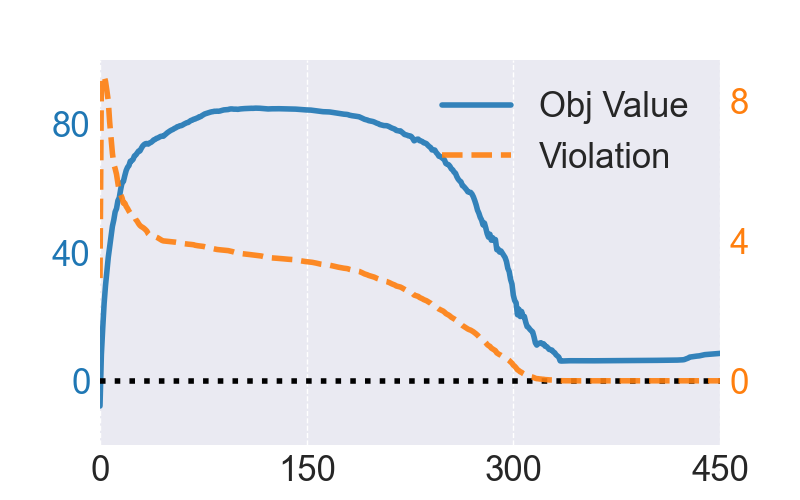}
&\includegraphics[width=0.24\textwidth]{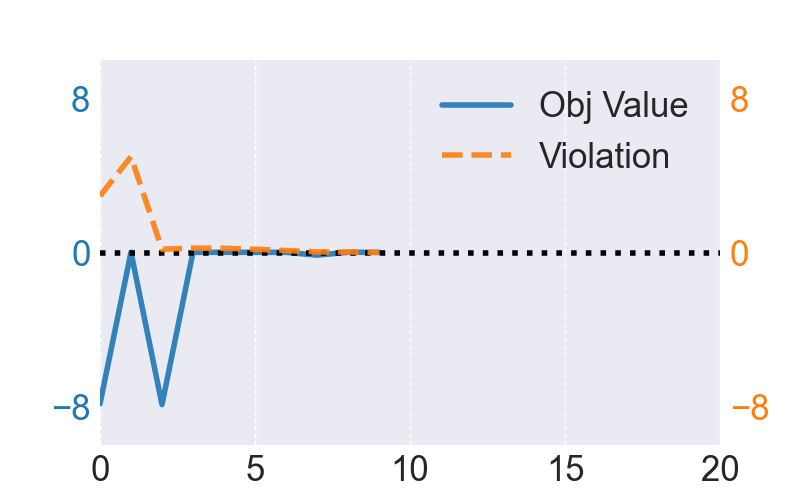}
\\
\small{\textbf{(c)} margin} 
&\small{\textbf{(d)} margin-clip}
\end{tabular}
\endgroup 
\caption{PWCF optimization trajectories using cross-entropy (CE) loss, margin loss, and their clipped versions to solve the max-loss form using the $\ell_2$ metric on a CIFAR-10 image. The x-axes represent the number of iterations. For both CE and margin loss without clipping (\textbf{(a)} and \textbf{(c)}), PWCF progresses slowly toward feasibility (dashed orange curve), while with clipping (\textbf{(b)} and \textbf{(d)}), PWCF finds an optimal and feasible solution within only a few iterations.} 
\label{Fig: ablation loss clipping}
\end{figure}
\begin{table*}[!tb]
\caption{Summary of PWCF formulations to solve max-loss form and min-radius form. \textcolor{blue}{Blue} highlights the reformulation of $d$ used in min-radius form, and the \textcolor{red}{red} highlights the use of constraint-folding. Note that we distinguish the vector operator $\mb{\max}$ (entry-wise maximal value of a set of vectors, returning a vector with the same dimension) from the scalar operator $\max$ (maximal value of a given vector); \textbf{concat} denotes vector concatenations. }
\label{alg:pwcf}
\begin{center}
\setlength{\tabcolsep}{1.0mm}{
\begin{tabular}{c c c c c c c c c}

{}
&{Metric $d$}
&\small{$\ell_1$}
&{~~~}
&\small{$\ell_2$}
&{~~~}
&\small{$\ell_\infty$}
&{~~~}
&\small{Others}
\\
\toprule
{}
&{$\max$}
&{$\ell\paren{\mb y, f_{\theta}\paren{\mb x'}}$}
&{ }
&{$\ell\paren{\mb y, f_{\theta}\paren{\mb x'}}$}
&{ }
&{$\ell\paren{\mb y, f_{\theta}\paren{\mb x'}}$}
&{ }
&{$\ell\paren{\mb y, f_{\theta}\paren{\mb x'}}$}
\\
{}
&\multicolumn{7}{c}{ }
\vspace{-0.7em}
\\
{}
&{$\st$}
&{$\norm{\mb x' - \mb x}_1 \le \eps$}
&{ }
&{$\norm{\mb x' - \mb x}_2 \le \eps$}
&{ }
&{\textcolor{red}{$\|$}$\mathbf{\max} \{\mb x' - \mb x - \eps \mb 1 \text{,}$}
&{ }
&{$d\paren{\mb x', \mb x} \le \eps$}
\\
{}
&{}
&{}
&{ }
&{}
&{ }
&{$\quad 0 \}$\textcolor{red}{$\|_{2}$}$\le 0$}
&{ }
&{}
\\
{}
&\multicolumn{8}{c}{ }
\vspace{-0.4em}
\\
{}
&{~}
&\multicolumn{7}{c}{\textcolor{red}{$\|$}$\mb{\max} \{ \textbf{concat} \paren{- \mb x', ~ \mb x'-\mb 1}, ~ \mb 0\}$\textcolor{red}{$\|_2$} $\le 0$}
\\
\midrule
\vspace{-1em}
\\
{}
&{$\min$}
&{\textcolor{blue}{$\mb{1}^{\TJU} \mb t$}}
&{ }
&{$\norm{\mb x' - \mb x}_2$}
&{ }
&{\textcolor{blue}{$t$}}
&{ }
&{$d \paren{\mb x', \mb x}$}
\\
{ }
&\multicolumn{7}{c}{ }
\vspace{-0.7em}
\\
{}
&{$\st$}
&{\textcolor{red}{$\|$}$\mb{\max} \{ ~~~~~~~~~~~~~~~~~~~$}
&{ }
&{ }
&{ }
&{\textcolor{red}{$\|$}$\mb{\max} \{ ~~~~~~~~~~~~~~~~~~~~~$}
&{ }
&{}
\\
{}
&{}
&{$\textbf{concat} (\mb x' - \mb x - \mb t,~~$}
&{ }
&{Not}
&{ }
&{$\textbf{concat} (\mb x' - \mb x - t\mb 1,~~$}
&{ }
&{Not}
\\
{ }
&{ }
&{$- \mb x' + \mb x - \mb t), ~ \mb 0 \}$\textcolor{red}{$\|_2$}$\le 0$}
&{ }
&{Applicable}
&{ }
&{$- \mb x' + \mb x - t \mb 1), ~ \mb 0 \}$\textcolor{red}{$\|_2$}$\le 0$}
&{ }
&{Applicable}
\\
{ }
&\multicolumn{7}{c}{ }
\vspace{-0.4em}
\\
{}
&{~}
&\multicolumn{7}{c}{$\max_{i \ne y} f_{\mb \theta}^i (\mb x') \ge f_{\mb \theta}^y (\mb x')$}
\\
&\multicolumn{7}{c}{ }
\vspace{-0.7em}
\\
{}
&{~}
&\multicolumn{7}{c}{\textcolor{red}{$\|$}$\mb{\max} \{ \textbf{concat} \paren{- \mb x', ~ \mb x'-\mb 1}, ~ \mb 0\}$\textcolor{red}{$\|_2$} $\le 0$}
\\
\bottomrule
\end{tabular}
}
\end{center}
\end{table*}
\begin{figure*}[!tb]
\vspace{-1em}
\centering
\begingroup 
\setlength{\tabcolsep}{1pt}
\renewcommand{\arraystretch}{0.8}
\begin{tabular}{c c c c c c}
\centering
{}
&\multicolumn{2}{c}{\textbf{max-loss form}}
&{ }
&\multicolumn{2}{c}{\textbf{min-radius form}}
\\
\cline{2-3}\cline{5-6}
\vspace{-1em}
\\
{}
&\includegraphics[width=0.24\textwidth]{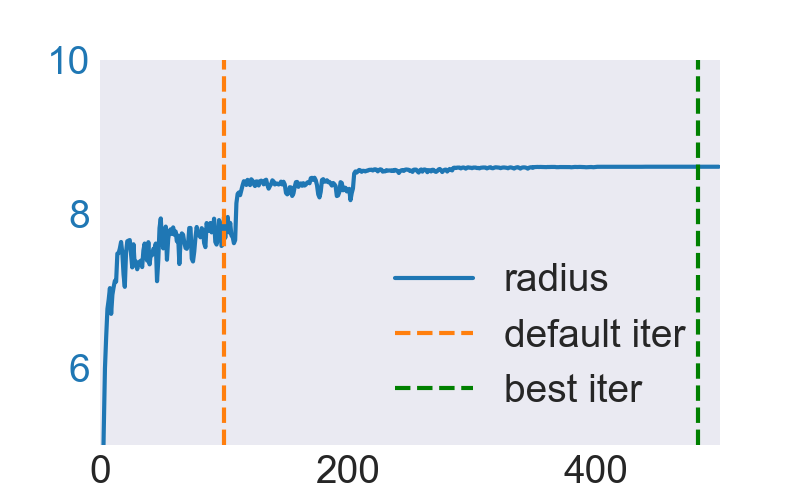}
&\includegraphics[width=0.24\textwidth]{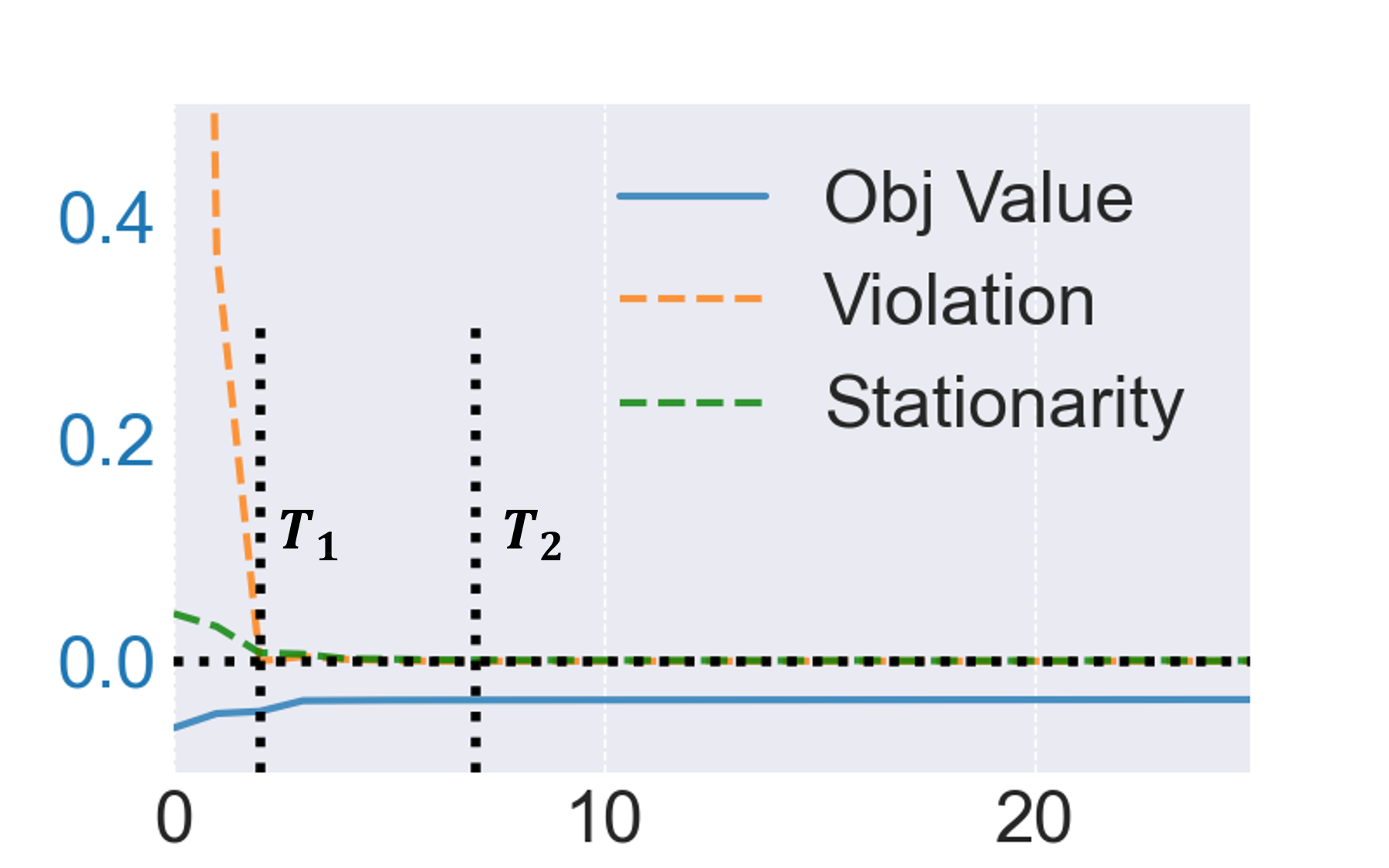}
&{ }
&\includegraphics[width=0.24\textwidth]{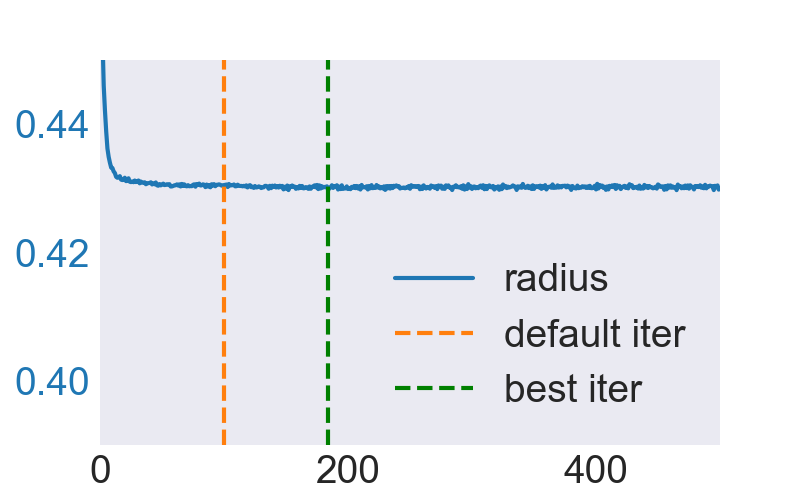}
&\includegraphics[width=0.24\textwidth]{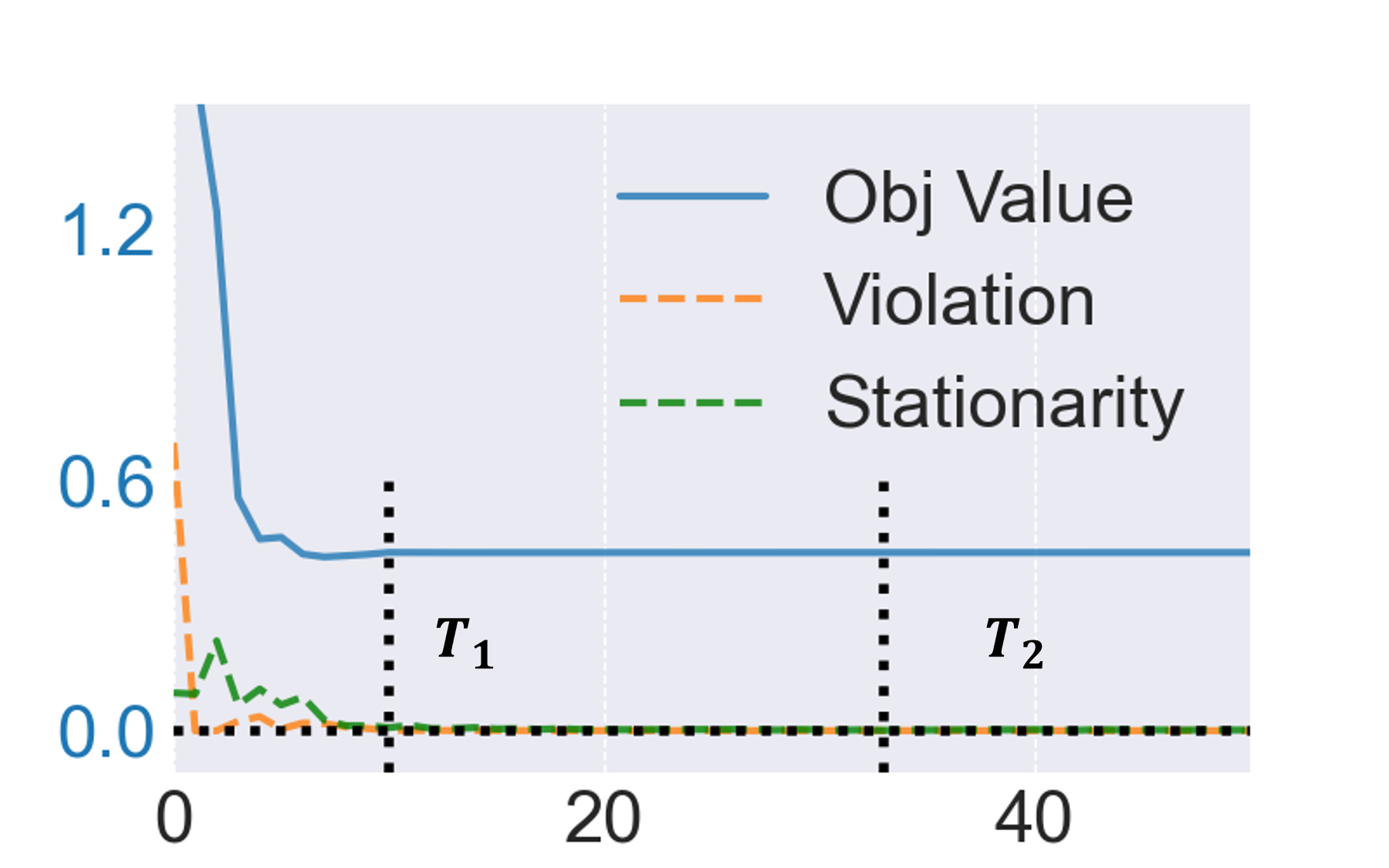}
\\
{}
&\small{\textbf{(a)} APGD} 
&\small{\textbf{(b)} PWCF - margin loss}
&{ }
&\small{\textbf{(c)} FAB} 
&\small{\textbf{(d)} PWCF}

\end{tabular}
\endgroup 
\caption{Examples of the optimization trajectories of APGD (with CE loss), FAB and PWCF for solving max-loss form and min-radius form of a CIFAR-10 image with the $\ell_2$ distance. The x-axes represent the iteration numbers. In \textbf{(a)} and \textbf{(c)}, the dashed orange lines are the default MaxIter used in \texttt{AutoAttack} and the dashed green lines are the iteration where the best feasible solutions are found. In \textbf{(b)} and \textbf{(d)}, we set the stationarity and constraint violation tolerances to be $10^{-8}$ for the termination condition of the PWCF to better visualize the optimization curve. \textbf{$T_1$} and \textbf{$T_2$} mark the iterations where both the stationarity and total violation reach $10^{-2}$ and $10^{-3}$, respectively. We can observe that the objective values only improve marginally after $T_1$, \change{indicating that $10^{-2}$ is a reasonable tolerance level for PWCF to trade for efficiency while maintaining the solution quality.}}
\label{Fig: Abaltion-OPT-Traj}
\end{figure*}

\subsection{Summary of PWCF to solve the max-loss form and min-radius form}
\label{subsec: PWCF techniques demo}
We now summarize PWCF for solving max-loss form and min-radius form in \cref{alg:pwcf}. We also provide information on PWCF's reliability and running time analysis on these two problems. \change{We will present the ability of PWCF to handle general distance metrics in \cref{Sec: experiments and results}.}

\subsubsection{Reliability}
\label{subsec: reliability}
As mentioned in \cref{Sec:introduction}, popular numerical methods only relying on preset MaxIter are subject to premature termination (see \cref{Fig:APGD-FAB-Terminate-Iter} to review). In contrast, PWCF terminates either when the stopping criterion is met (solution quality meet the preset tolerance level), or by MaxIter with stationarity estimate and total violation to assess whether further optimization is necessary (see \cref{Fig: PWC-Max-Terminate-Iter} to review). \cref{Fig: Abaltion-OPT-Traj} provide extra examples of the optimization trajectories of APGD, FAB and PWCF, which again shows why PWCF's termination mechanism is more reliable.

\begin{figure*}[!tb]
\vspace{1em}
\centering
\begingroup 
\setlength{\tabcolsep}{1pt}
\renewcommand{\arraystretch}{0.8}
\begin{tabular}{c c c c c c}
\centering
{}
&\multicolumn{2}{c}{\textbf{CIFAR-10}}
&{ }
&\multicolumn{2}{c}{\textbf{ImageNet-100}}
\\
\cline{2-3}\cline{5-6}
\vspace{-1em}
\\
{}
&\includegraphics[width=0.24\textwidth]{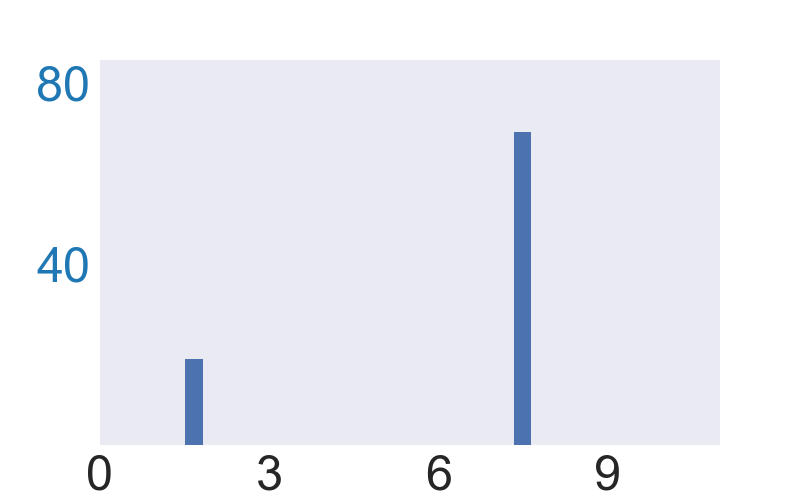}
&\includegraphics[width=0.24\textwidth]{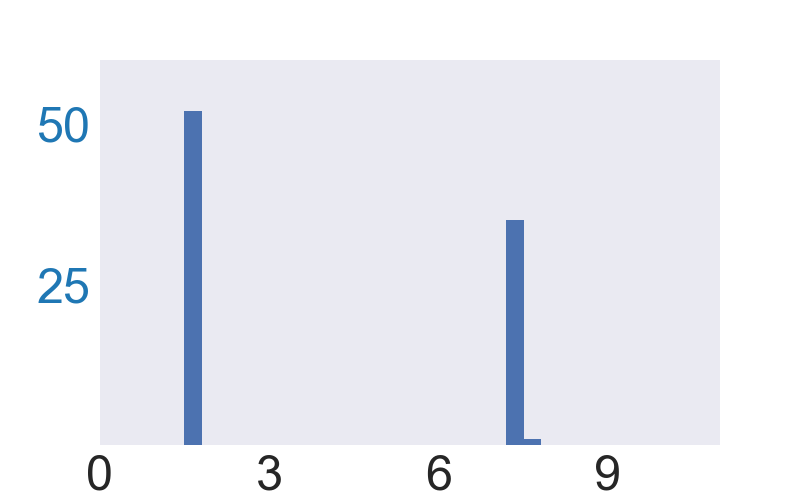}
&{ }
&\includegraphics[width=0.24\textwidth]{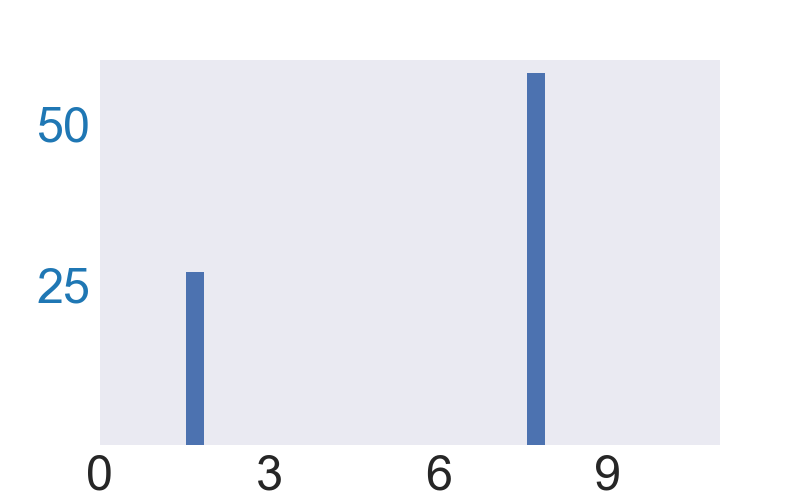}
&\includegraphics[width=0.24\textwidth]{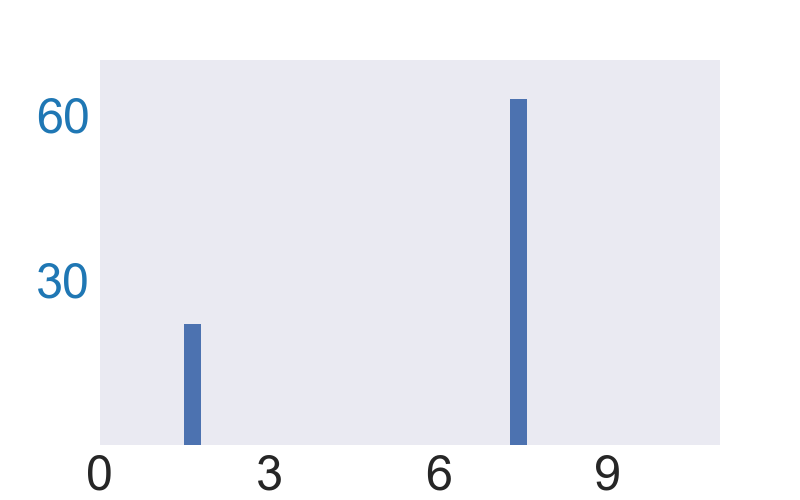}
\\
{}
&\small{\textbf{(a)} APGD - $\ell_2$} 
&\small{\textbf{(b)} APGD - $\ell_\infty$}
&{ }
&\small{\textbf{(c)} APGD - $\ell_2$} 
&\small{\textbf{(d)} APGD - $\ell_\infty$}
\\
&\includegraphics[width=0.24\textwidth]{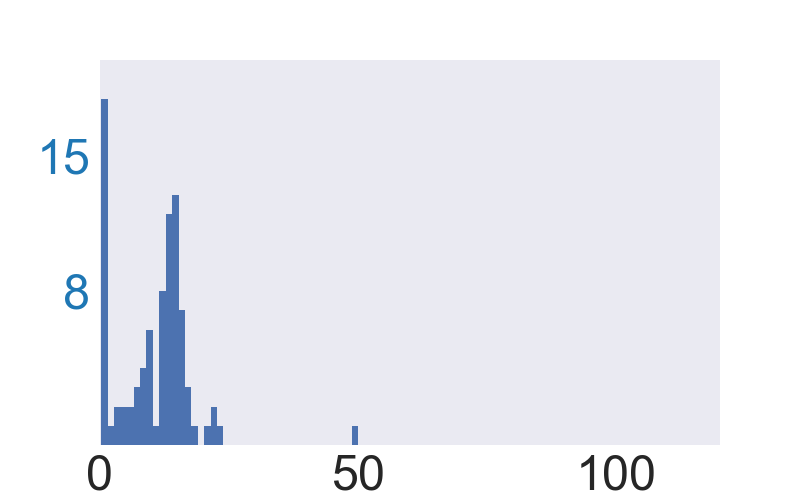}
&\includegraphics[width=0.24\textwidth]{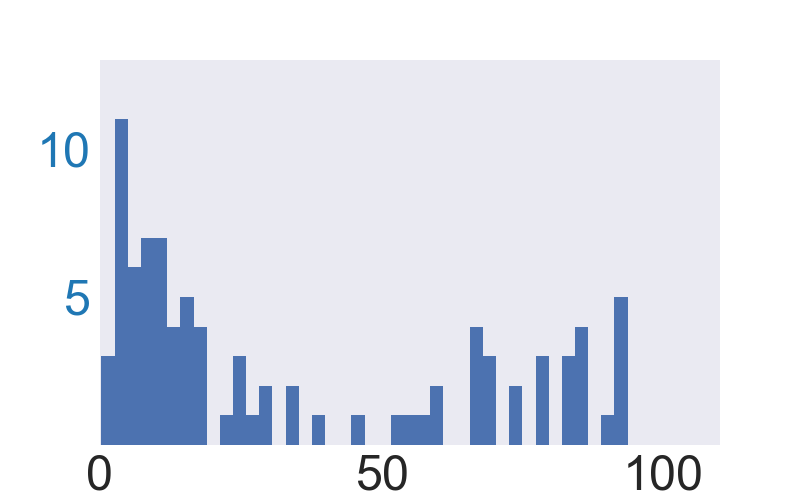}
&{ }
&\includegraphics[width=0.24\textwidth]{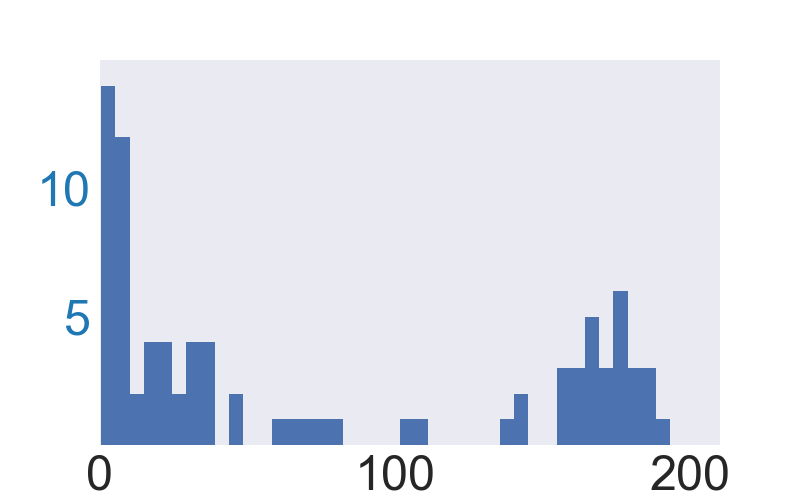}
&\includegraphics[width=0.24\textwidth]{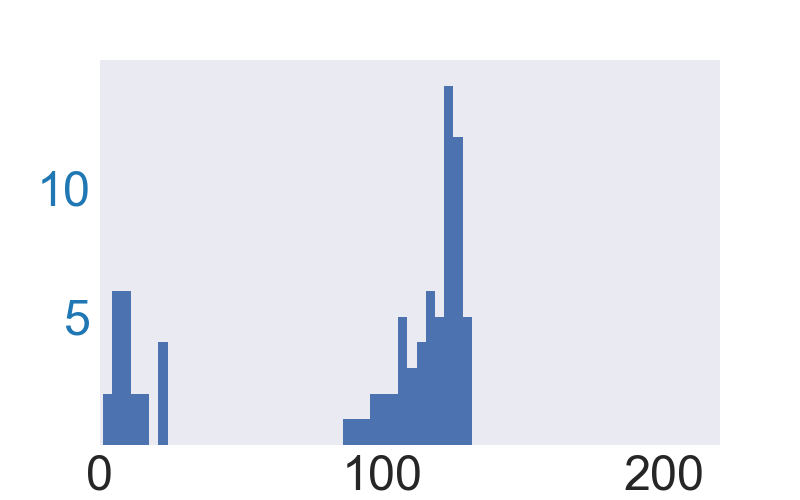}
\\
{}
&\small{\textbf{(e)} PWCF - $\ell_2$} 
&\small{\textbf{(f)} PWCF - $\ell_\infty$}
&{ }
&\small{\textbf{(g)} PWCF - $\ell_2$} 
&\small{\textbf{(h)} PWCF - $\ell_\infty$}
\end{tabular}
\endgroup 
\caption{Histograms of the termination time (x-axes, in seconds) of APGD (using \texttt{AutoAttack} default setup: five random restarts with 100 as MaxIter per run) and PWCF (ours) to solve max-loss form on $88$ images from CIFAR-10 and $85$ images from ImageNet-100. For this problem, PWCF can be slower than APGD by about one order of magnitude.}
\label{Fig: PWC-Terminate-Time-Max}
\end{figure*}
\begin{figure*}[!tb]
\vspace{1em}
\centering
\begingroup 
\setlength{\tabcolsep}{1pt}
\renewcommand{\arraystretch}{0.8}
\begin{tabular}{c c c c c c}
\centering
{}
&\multicolumn{2}{c}{\textbf{CIFAR-10}}
&{ }
&\multicolumn{2}{c}{\textbf{ImageNet-100}}
\\
\cline{2-3}\cline{5-6}
\vspace{-1em}
\\
{}
&\includegraphics[width=0.24\textwidth]{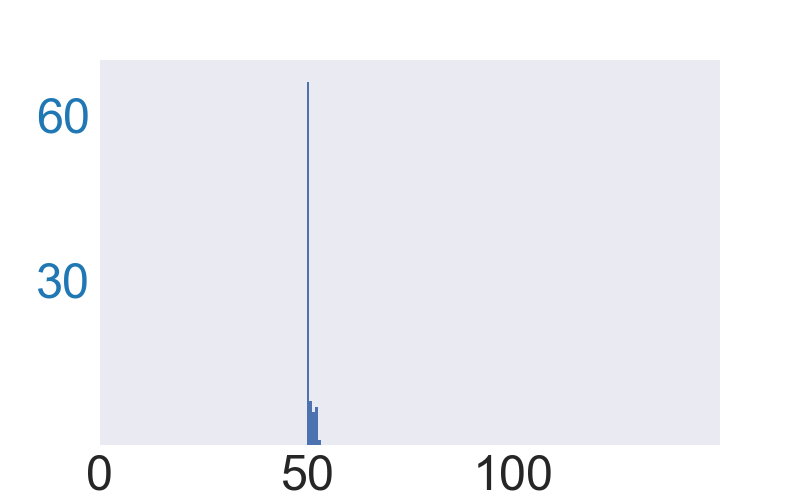}
&\includegraphics[width=0.24\textwidth]{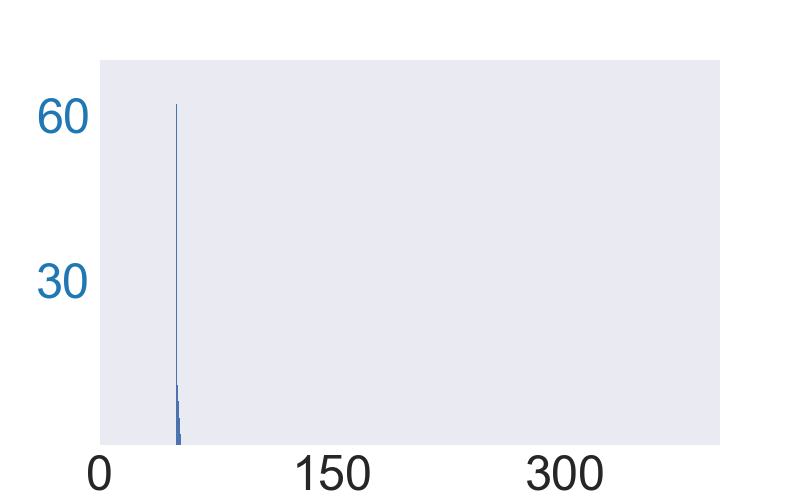}
&{ }
&\includegraphics[width=0.24\textwidth]{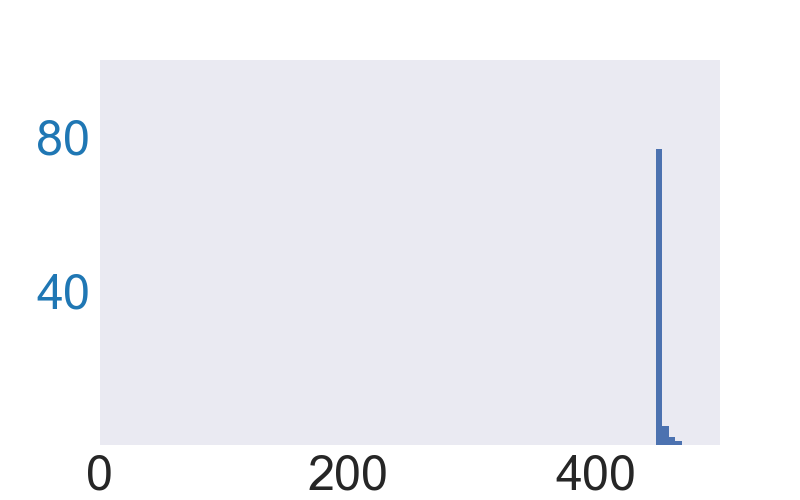}
&\includegraphics[width=0.24\textwidth]{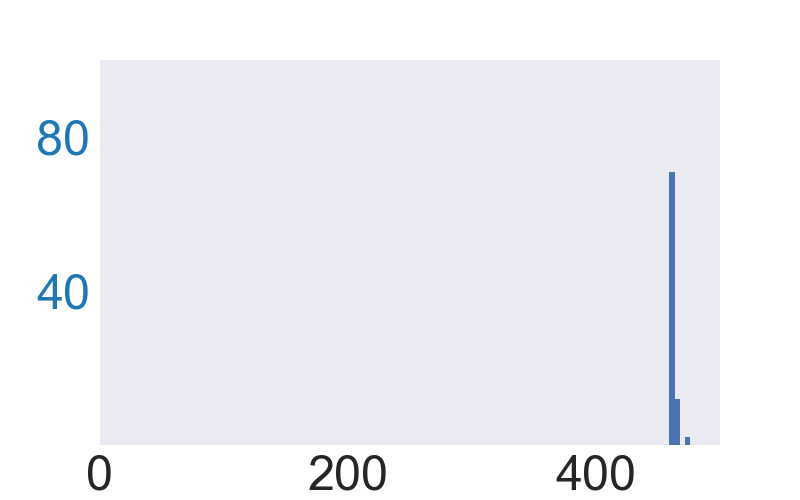}
\\
{}
&\small{\textbf{(a)} FAB - $\ell_2$} 
&\small{\textbf{(b)} FAB - $\ell_\infty$}
&{ }
&\small{\textbf{(c)} FAB - $\ell_2$} 
&\small{\textbf{(d)} FAB - $\ell_\infty$}
\\
&\includegraphics[width=0.24\textwidth]{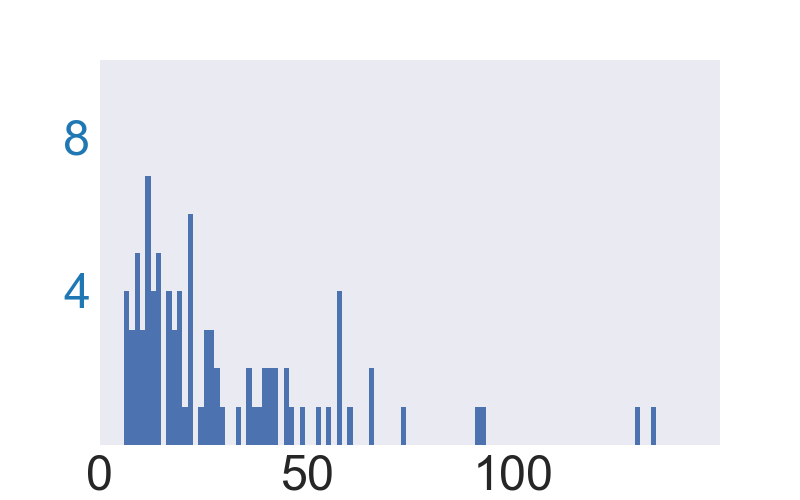}
&\includegraphics[width=0.24\textwidth]{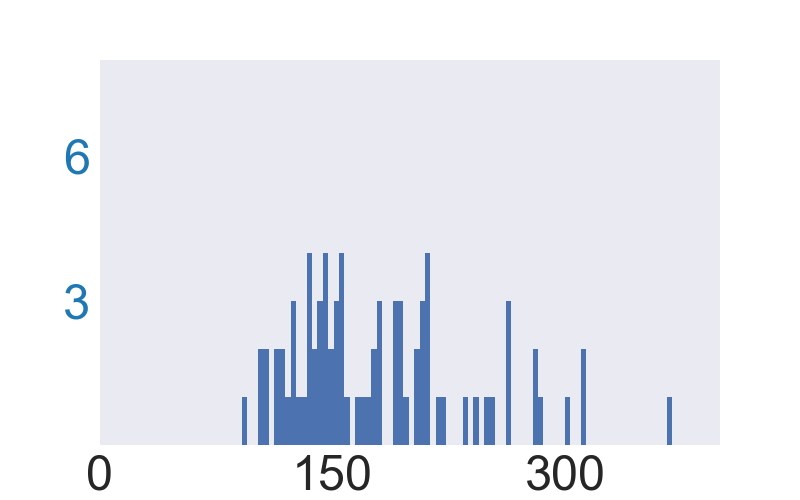}
&{ }
&\includegraphics[width=0.24\textwidth]{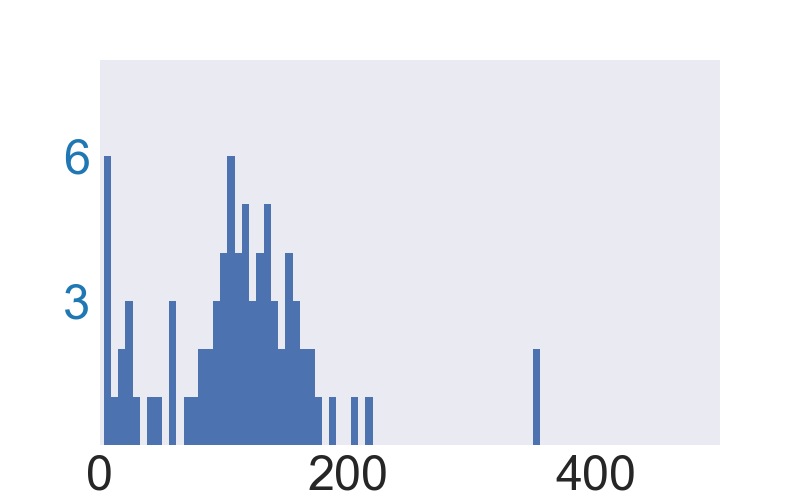}
&\includegraphics[width=0.24\textwidth]{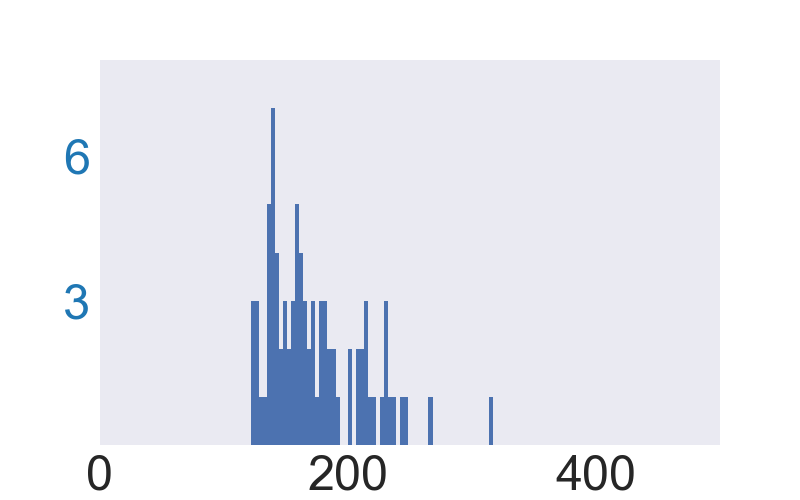}
\\
{}
&\small{\textbf{(e)} PWCF - $\ell_2$} 
&\small{\textbf{(f)} PWCF - $\ell_\infty$}
&{ }
&\small{\textbf{(g)} PWCF - $\ell_2$} 
&\small{\textbf{(h)} PWCF - $\ell_\infty$}
\end{tabular}
\endgroup 
\caption{Histograms of the termination time (x-axis, in seconds) for FAB (using \texttt{AutoAttack} default setup: five random restarts with 100 as MaxIter per run) and PWCF to solve min-radius form on $88$ images from CIFAR-10 and $85$ images from ImageNet-100. PWCF is on average faster than FAB (comparing \textbf{(a)} \& \textbf{(e)}, \textbf{(c)} \& \textbf{(g)} and \textbf{(d)} \& \textbf{(h)}) except for the $\ell_\infty$ case on CIFAR-10 images (\textbf{(b)} \& \textbf{(f)}).}
\label{Fig: PWC-Terminate-Time}
\end{figure*}
\begin{table*}[!tb]
\caption{\textbf{Comparison between PWCF and other numerical algorithms in solving max-loss form with the $\ell_{1}$, $\ell_{2}$ and $\ell_{\infty}$ distances.} \textbf{Metric ($\eps$)} denotes the choice of $d$ and the corresponding perturbation budget $\eps$ used. We report the model's \textcolor{blue}{clean} and robust accuracy (numbers are in $(\%)$) for comparison---lower robust accuracy reflects more effective optimization. We test APGD and PWCF using both \textbf{CE} and margin (\textbf{M}) loss. \textbf{CE+M} column shows the robust accuracy achieved by combining adversarial samples found using CE and margin losses; \textbf{A\&P} shows the robust accuracy achieved by combining all perturbation samples found by using APGD and PWCF with both CE and M losses. We highlight the best performance achieved by a single combination of solver and loss with \underline{underlines} for each $d$, and highlight the best performance achieved in \textbf{bold}.}

\label{tab: granso_l1_acc}
\begin{center}
\setlength{\tabcolsep}{1.0mm}{
\begin{tabular}{l c c c c c c c c c c c c c c}
{}
&{}
&{}
&{}
&\multicolumn{3}{c}{\small{\textbf{APGD}}}
&{}
&\multicolumn{3}{c}{\small{\textbf{PWCF(ours)}}}
&{}
&\small{\textbf{Square}}
&{}
&{}
\\
\cline{5-7}\cline{9-11}\cline{13-13}
\vspace{-12pt}
\\
{\small{\textbf{Dataset}}}
&{\small{\textbf{Metric ($\eps$)}}}
& \small{\textbf{Clean}}
&{}
& \small{\textbf{CE}}
& \small{\textbf{M}}
& \small{\textbf{CE+M}}
& {}
& \small{\textbf{CE}}
& \small{\textbf{M}}
& \small{\textbf{CE+M}}
& {}
& \small{\textbf{M}}
&{}
&\small{\textbf{A\&P}}
\\
\toprule
{\small{CIFAR-10}}
&\small{$\ell_{1} (12)$}
&\textcolor{blue}{73.29}
&{}
&{0.97}
&\underline{0.00}
&{0.00}
&{}
&{17.93}
&{0.01}
&{0.01}
&{}
&{2.28}
&{}
&\textbf{0.00}
\\

\cline{2-15}
\vspace{-12pt}
\\
{}
&\small{$\ell_{2} (0.5)$}
&\textcolor{blue}{94.61}
&{}
&{81.81}
&{81.06}
&{80.92}
&{}
&{81.99}
&\underline{81.02}
&{80.87}
&{}
&{87.9}
&{}
&\textbf{80.77}
\\
\cline{2-15}
\vspace{-12pt}
\\
{}
&\small{$\ell_{\infty} (0.03)$}
&\textcolor{blue}{90.81}
&{}
&{69.44}
&\underline{67.71}
&{67.33}
&{}
&{88.71}
&{68.20}
&{68.17}
&{}
&{71.6}
&{}
&\textbf{67.26}
\\
\midrule
\midrule
{\small{ImageNet}-\small{100}}
&\small{$\ell_{2} (4.7)$}
&\textcolor{blue}{75.04}
&{}
&\underline{42.44}
&{44.06}
&{40.86}
&{}
&{42.50}
&{43.52}
&{40.60}
&{}
&{63.1}
&{}
&\textbf{40.46}
\\
\cline{2-15}
\vspace{-12pt}
\\
{}
&\small{$\ell_\infty (0.016)$}
&\textcolor{blue}{75.04}
&{}
&\underline{46.78}
&{47.54}
&{45.20}
&{}
&{73.92}
&{47.72}
&{47.72}
&{}
&{59.9}
&{}
&\textbf{45.12}
\\
\bottomrule
\end{tabular}
}
\end{center}
\end{table*}

\subsubsection{Running cost}
\label{subsec: running cost}
We now consider run-time comparisons of APGD and FAB v.s. PWCF in solving max-loss form and min-radius form on CIFAR-10 and ImageNet images, as in \cref{Fig: PWC-Terminate-Time-Max} and \cref{Fig: PWC-Terminate-Time}. Our computing environment uses an AMD Milan 7763 64-core processor and a NVIDIA A100 GPU (40G version). We use $10^{-2}$ as PWCF's stationarity and violation tolerances to benchmark PWCF's running cost. The results show that PWCF can be faster than FAB (except for the $\ell_\infty$ case on CIFAR-10 images) when solving min-radius form (\cref{Fig: PWC-Terminate-Time}), but can be about $10$ times slower than APGD when solving max-loss form (\cref{Fig: PWC-Terminate-Time-Max}).  Again, we remark that solution reliability should come with a higher priority than speed for RE; we leave improving the speed of PWCF as future work. The run-time result of PWCF implies that PWCF may be favorable for RE where reliability and accuracy are crucial but may be non-ideal to be used in training pipelines.

\section{Performance of PWCF in solving max-loss form and min-radius form}
\label{Sec: experiments and results}
In this section, we show the effectiveness of PWCF in solving max-loss form and min-radius form with general distance metrics $d$. First in \cref{min max l1 l2 linf}, we compare the performance of PWCF with other existing numerical algorithms when $d$ is the $\ell_1$, $\ell_2$, and $\ell_\infty$ distance. Next in \cref{min max l15 l8}, we take the $\ell_{1.5}$ and $\ell_8$ distances as examples to show that PWCF can effectively handle max-loss form and min-radius form with general $\ell_p$ metrics. Finally in \cref{subsec: min max perceptual}, we show that PWCF can also handle both formulations with the perceptual distance (PD, a non-$\ell_p$ metric). In this section, due to the variety of capacity differences of the DNN models used, we conservatively use $k=40$ and $K=400$ for experiments on max-loss form, $k=50$ and $K=4000$ for min-radius form on CIFAR-10 dataset, $k=200$ and $K=5000$ for min-radius form on ImageNet-100 dataset and $r=10$ for all cases to avoid possible premature terminations.

\subsection{Solving max-loss form and min-radius form with the $\ell_1$, $\ell_2$ and $\ell_\infty$ distance}
\label{min max l1 l2 linf}
We now compare the performance of PWCF with other existing numerical algorithms  in solving max-loss form and min-radius form when $d$ is the $\ell_1$, $\ell_2$, and $\ell_\infty$ distance to show that PWCF can solve both formulations effectively.

\begin{figure*}[!tb]
\centering
\begingroup 
\setlength{\tabcolsep}{1pt}
\renewcommand{\arraystretch}{0.8}
\begin{tabular}{ccc}
\centering
\includegraphics[width=0.33\textwidth]{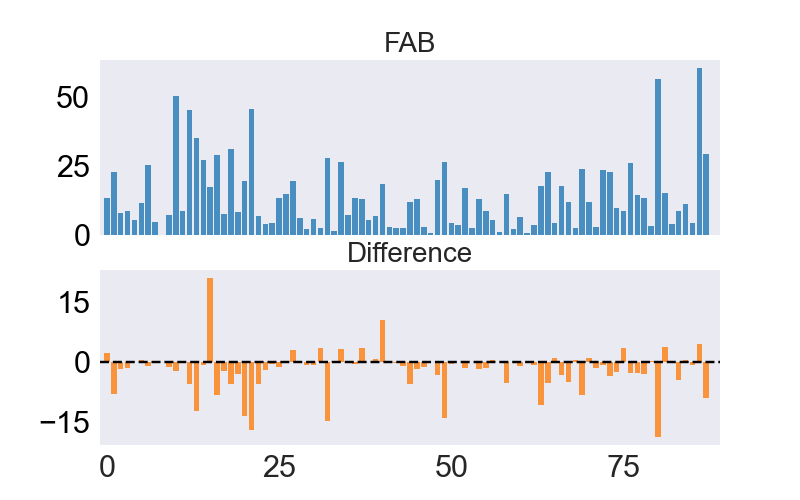}
&\includegraphics[width=0.33\textwidth]{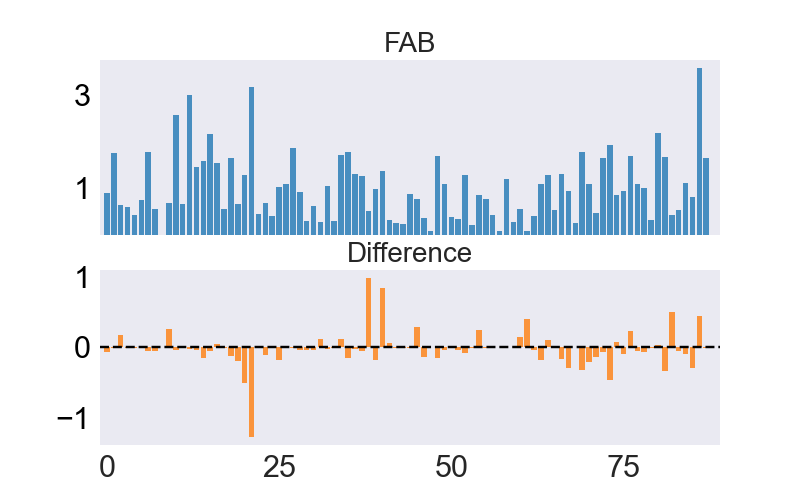}
&\includegraphics[width=0.33\textwidth]{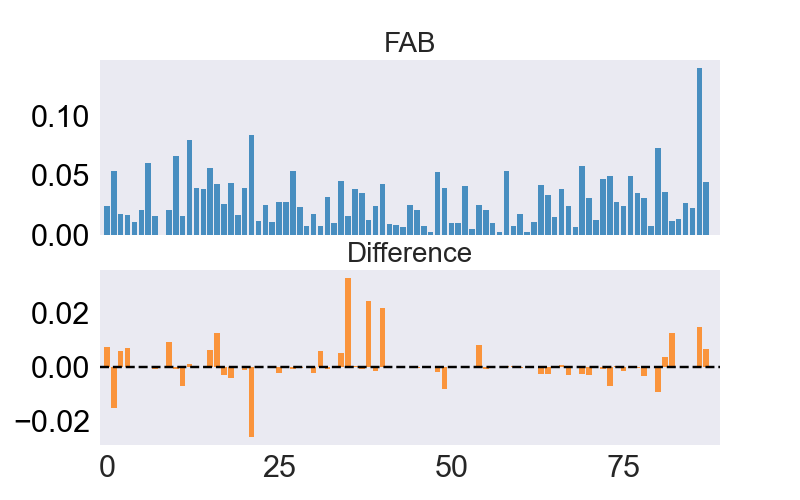}
\\
\small{CIFAR-10 - $\ell_1$} 
&\small{CIFAR-10 - $\ell_2$} 
&\small{CIFAR-10 - $\ell_\infty$}
\\
\includegraphics[width=0.33\textwidth]{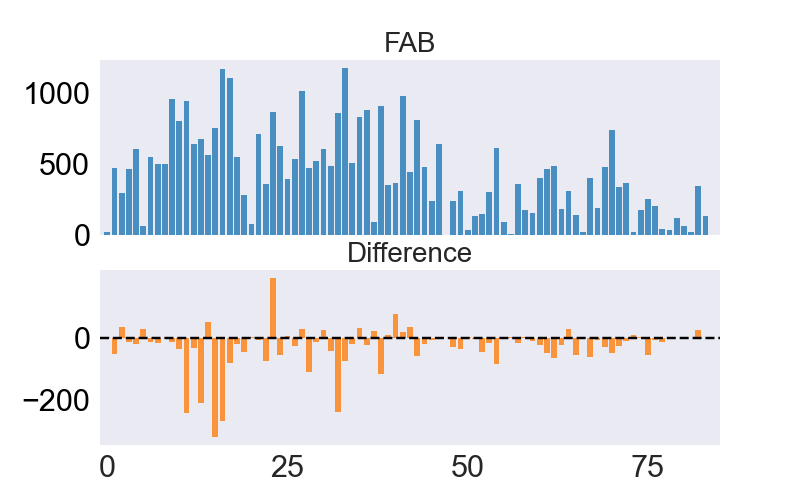}
&\includegraphics[width=0.33\textwidth]{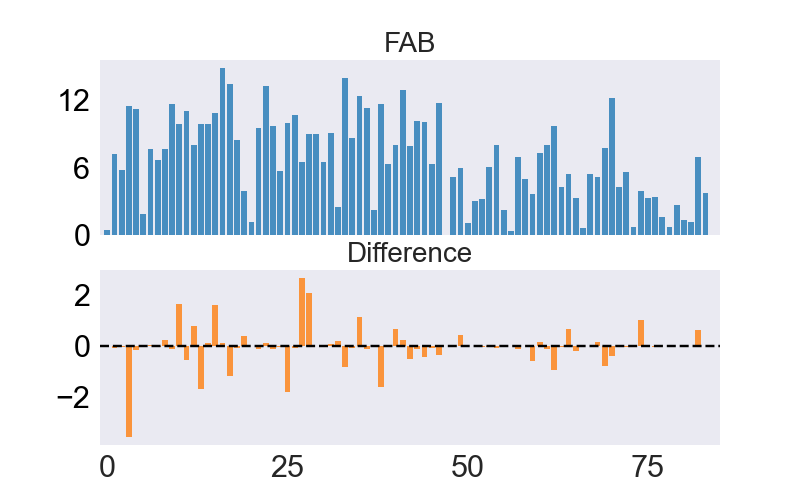}
&\includegraphics[width=0.33\textwidth]{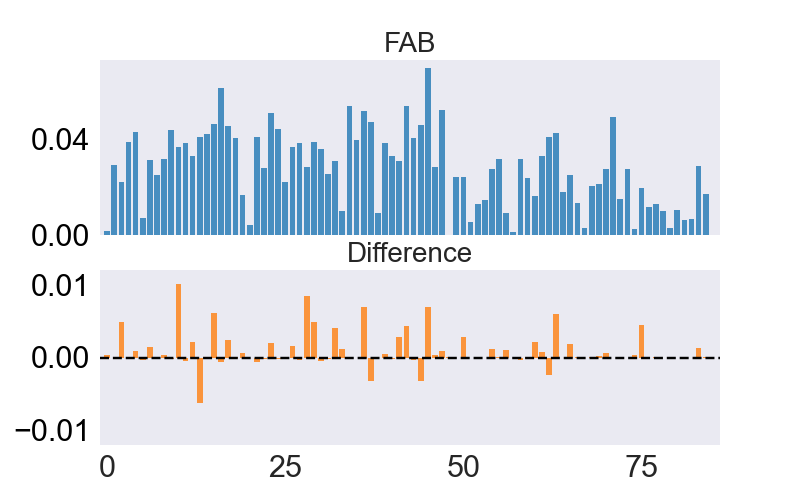}
\\
\small{ImageNet-100 - $\ell_1$} 
&\small{ImageNet-100 - $\ell_2$} 
&\small{ImageNet-100 - $\ell_\infty$}
\end{tabular}
\endgroup 
\caption{Comparison of the per-sample robustness radius between PWCF(ours) and FAB for $88$ images from CIFAR-10 and $85$ images from ImageNet-100. In each subfigure, the $1^{\text{st}}$ row shows the sample-wise robustness radius found by FAB, and the $2^{\text{nd}}$ row shows the radius difference between PWCF and FAB (PWCF minus FAB): values $< 0$ indicating a better result found by PWCF than FAB.} 
\label{Fig:FAB-min-radius-details}
\end{figure*}
\begin{table*}[!tb]
\caption{\textbf{Statistical comparison of the minimal radius found by PWCF(ours) and FAB in solving the \emph{min-form} with $\ell_1$, $\ell_2$ and $\ell_\infty$ as the metric $d$ (summary of \cref{Fig:FAB-min-radius-details}).} We experimented with $88$ fixed images from CIFAR-10 and ImageNet-100 dataset with each $d$. We report the \textbf{Mean}, \textbf{Median} and \textbf{standard deviation (STD)} of the minimal perturbation radius---\emph{lower} radius means \emph{more effective} minimization. The columns under \textbf{Difference} are calculated based on the \emph{sample-wise} radius difference (PWCF radius minus FAB), where \textbf{Mean} and \textbf{Median} $\leq 0$ indicates PWCF performs better than FAB on average.}
\label{tab: granso_min}
\begin{center}
\setlength{\tabcolsep}{1.0mm}{
\begin{tabular}{l c c c c c c c c c c c c}
{}
&{}
&\multicolumn{3}{c}{\small{\textbf{FAB}}}
&{}
&\multicolumn{3}{c}{\small{\textbf{PWCF (ours)}}}
&{}
&\multicolumn{3}{c}{\small{\textbf{Difference}}}
\\
\cline{3-5}\cline{7-9}\cline{11-13}
\vspace{-12pt}
\\
{\small{\textbf{Dataset}}}
&{\small{\textbf{Metric} $d$}}
& \small{\textbf{Mean}}
& \small{\textbf{Median}}
& \small{\textbf{STD}}
& {}
& \small{\textbf{Mean}}
& \small{\textbf{Median}}
& \small{\textbf{STD}}
& {}
& \small{\textbf{Mean}}
& \small{\textbf{Median}}
& \small{\textbf{STD}}
\\
\toprule
\small{CIFAR-10}
&\small{$\ell_{1}$}
&{13.92}
&{10.50}
&{12.63}
&{}
&{12.02}
&{7.29}
&{11.46}
&{}
&{\textbf{-1.89}}
&{\textbf{-0.81}}
&{5.24}
\\
\cline{2-13}
\vspace{-12pt}
\\
{}
&\small{$\ell_{2}$}
&{1.02}
&{0.90}
&{0.71}
&{}
&{1.00}
&{0.88}
&{0.69}
&{}
&\textbf{-0.019}
&\textbf{-0.015}
&{0.250}
\\
\cline{2-13}
\vspace{-12pt}
\\
{}
&\small{$\ell_{\infty}$}
&{0.0298}
&{0.0245}
&{0.0220}
&{}
&{0.0298}
&{0.0252}
&{0.0224}
&{}
&\textbf{0.0008}
&\textbf{-0.00008}
&{0.007}
\\
\midrule
\midrule
{\small{ImageNet-100}}
&\small{$\ell_{1}$}
&{435.4}
&{400.6}
&{303.9}
&{ }
&{408.1}
&{390.6}
&{284.7}
&{}
&\textbf{-27.31}
&\textbf{-13.46}
&{70.55}
\\
\cline{2-13}
\vspace{-12pt}
\\
{}
&\small{$\ell_2$}
&{6.75}
&{6.81}
&{3.82}
&{}
&{6.71}
&{6.88}
&{3.76}
&{}
&\textbf{-0.042}
&\textbf{-0.035}
&{0.758}
\\
\cline{2-13}
\vspace{-12pt}
\\
{}
&\small{$\ell_\infty$}
&{0.028}
&{0.028}
&{0.016}
&{}
&{0.029}
&{0.029}
&{0.016}
&{}
&\textbf{0.0009}
&\textbf{0.00002}
&{0.002}
\\
\bottomrule
\end{tabular}
}
\end{center}
\end{table*}

\subsubsection{PWCF offers competitive attack performance in solving max-form with diverse solutions}
\label{subsec: max formulation with l1, l2, linf}
We use several publicly available models that are adversarially trained by $\ell_1$\footnote{For $\ell_1$ experiment, we use the model `L1.pt' from \url{https://github.com/locuslab/robust_union/tree/master/CIFAR10}, which is adversarially trained by $\ell_1$-attack.}, $\ell_2$, and $\ell_\infty$\footnote{For $\ell_2$ ad $\ell_\infty$ experiments, we use the models `L2-Extra.pt' and `Linf-Extra.pt' from \url{https://github.com/deepmind/deepmind-research/tree/master/adversarial_robustness}, with the WRN-70-16 network architecuture.} attacks on CIFAR-10, and by perceptual attack (See \cref{subsec: formulation with general lp norm} for details) on ImageNet-100~\cite{laidlaw2021perceptual}.\footnote{We use the `pat\_alexnet\_0.5.pt' from \url{https://github.com/cassidylaidlaw/perceptual-advex}, where the authors tested and showed its $\ell_2$- and $\ell_\infty$- robustness in the original work.} We then compare the robust accuracy achieved by these selected models by solving max-loss form with PWCF and APGD\footnote{We implement the margin loss on top of the original APGD.} from \texttt{AutoAttack}. The bound $\eps$ for each case is set to follow the common practice of RE\footnote{The $\eps$ of $\ell_2$ and $\ell_\infty$ for CIFAR-10 are chosen from \url{https://robustbench.github.io/}; $\ell_1$ for CIFAR-10 is chosen from \url{https://github.com/locuslab/robust_union}; $\ell_2$ and $\ell_\infty$ for ImageNet-100 are from \cite{laidlaw2021perceptual}.}. The result is shown in \cref{tab: granso_l1_acc}.

We can conclude from \cref{tab: granso_l1_acc} that \textbf{1)} PWCF performs strongly and comparably to APGD on solving max-loss form with the $\ell_1$, $\ell_2$ and $\ell_\infty$ distances, especially when margin loss is used. The weak performance of PWCF on $\ell_1$ and $\ell_\infty$ cases using CE loss is likely due to poor numerical scaling of the loss itself: the gradient magnitude of CE loss grows much larger as the loss value increases (see \cref{fig:loss_clipping} for a visualization example of the loss). Therefore, as the CE loss increases during the optimization of max-loss form while the constraint violation scale remains unchanged or even decreases, PWCF may suffer from the imbalanced contributions from the objective and constraint. In contrast, the gradient scale is always $1$ using the margin loss and PWCF will not have similar struggles. APGD has an explicit step-size rule different from the vanilla PGD algorithms to improve attack performance under CE loss~\cite{croce2020reliable}, while PWCF does not have special handling for this case. In fact, APGD methods also prefer the use of margin loss over CE as pointed out in~\cite{croce2020reliable}, although the consideration is different from PWCF. \textbf{2)} Combining all successful attack samples found by APGD and PWCF using CE and margin loss (column A\&P in \cref{tab: granso_l1_acc}) achieves the lowest robust accuracy for all distance metrics---PWCF and APGD provide diverse and complementary solutions in terms of attack effectiveness. A direct message here is that lacking diversity (e.g., solving max-loss form with a restricted set of algorithms) will result in overestimated robust accuracy. Note that~\cite{CarliniEtAl2019Evaluating} also remarks that the diversity of solvers matters more than the superiority of individual solver, which motivates \texttt{AutoAttack} to include Square Attack---a zero-th order black-box attack method that does not perform strongly itself as shown in \cref{tab: granso_l1_acc}. We will provide further discussions on the importance of diversity later in \cref{sec:pattern_theory}, based on the differences in the solution patterns found in solving max-loss form. 

\subsubsection{PWCF provides competitive solutions to min-radius form}
\label{subsec: min formulation with l1, l2 and linf}
We take models adversarially trained on CIFAR-10\footnote{We use model `pat\_self\_0.5.pt' from \url{https://github.com/cassidylaidlaw/perceptual-advex}.} and ImageNet-100\footnote{The same ImageNet-100 model used in \cref{tab: granso_l1_acc}.} and compare the robustness radii found by solving min-radius form with PWCF and FAB in \cref{Fig:FAB-min-radius-details}, and \cref{tab: granso_min} summarizes the mean, median, and standard deviation of the results in \cref{Fig:FAB-min-radius-details}. From the column Mean and Median in \cref{tab: granso_min}, we can conclude that PWCF performs on average \textbf{1)} better than FAB in solving min-radius form with the $\ell_1$ and $\ell_2$ distances, and \textbf{2)} comparably to FAB for the $\ell_\infty$ case.

\subsection{Solving max-loss form and min-radius form with general distance metrics}
\label{subsec: formulation with general lp norm}
As highlighted in \cref{sec:background}, a major limitation of the existing numerical methods is that they mostly handle limited choice of $d$. On the contrary, PWCF stands out as a convenient choice for other general distances. We now present solving max-loss form and min-radius form with $\ell_{1.5}$, $\ell_{8}$ norm and the perceptual distance (PD, a non-$\ell_p$ distance that involves a DNN) by PWCF. To the best of our knowledge, no prior work has studied handling general distance metrics in the two constraint optimization problems;~\cite{laidlaw2021perceptual} has proposed 3 algorithms to solve max-loss form with PD, which will be compared with PWCF in the following sections.

\subsubsection{Solving min-radius form and max-loss form with $\ell_{1.5}$ and $\ell_{8}$ distances}
\label{min max l15 l8}
Due to the lack of existing methods for comparison, we conduct the following experiments to show the effectiveness of PWCF:
\begin{itemize}[leftmargin=*]
    \item We first apply PWCF to solve min-radius form with the $\ell_{1.5}$ and $\ell_{8}$ distances, and compare the robustness radii with the $\ell_2$ results found in \cref{Fig:FAB-min-radius-details}. One necessary condition for effective optimization is that the robustness radius found using different $\ell_p$ metrics should have $\ell_{1.5} \ge \ell_2 \ge \ell_{8}$. \cref{fig: L1.5 L8 min radius plot} shows the per-sample radii found by PWCF under $\ell_{1.5}$ and $\ell_{8}$ metrics and confirms the satisfaction of the above condition.
    \item We employ two sample-adaptive strategies to set the perturbation budget $\eps$ for PWCF to solve max-loss form with the $\ell_{1.5}$ and $\ell_{8}$ distances: using the same DNN model to evaluate, we take $0.8$ and $1.2$ times the robustness radii found in \cref{fig: L1.5 L8 min radius plot} as $\eps$. If the robustness radii found in \cref{fig: L1.5 L8 min radius plot} are tight, PWCF should achieve close to $100 \%$ robust accuracy under the $0.8$ strategy and $0 \%$ robust accuracy under the $1.2$ strategy, respectively when solving max-loss form. In fact, PWCF achieves $98.33 \%$ and $13 \%$ robust accuracy, respectively for the $\ell_{1.5}$ case; PWCF achieves $93.33 \%$ and $1\%$, respectively for $\ell_{8}$ case---PWCF solves both max-loss form and min-radius form with reasonable quality.
\end{itemize}

\begin{figure}[!tb]
    \centering
    \includegraphics[width=0.45\textwidth]{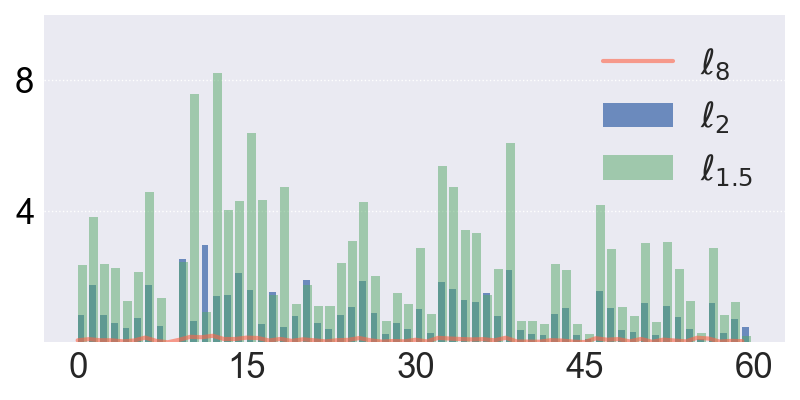}
    \caption{Robustness radii (y-axis) found by PWCF to solve min-radius form with the $\ell_{1.5}$, $\ell_2$ and $\ell_8$ distances on $60$ CIFAR-10 images. The x-axis represent the sample indices. Effective optimization should respect that the robustness radii found with different metrics satisfy $\ell_{1.5} \ge \ell_2 \ge \ell_{8}$ for every sample.} 
    \label{fig: L1.5 L8 min radius plot} 
\end{figure}

\subsubsection{Solving min-radius form and max-loss form with PD}
\label{subsec: min max perceptual}
Similar to the problems with the $\ell_{1.5}$ and $\ell_{8}$ distances, we are unaware of any existing work that has considered solving min-radius form with PD:\footnote{There are several existing variants of the perceptual distance. Here, we consider the LPIPS distance (first introduced in~\cite{Zhang_2018_CVPR}).}
\begin{align}
\label{Eq. LPIPS Constraint} 
\begin{split}
& d(\mb x, \mb x') \doteq \norm{\phi(\mb x) - \phi(\mb x')}_{2}\\
\text{where} \quad & \phi(\mb x) \doteq [~\wh{g}_{1}(\mb x), \dots, \wh{g}_{L}(\mb x)~]
\end{split}
\end{align}
where $\wh{g}_{1}(\mb x), \dots, \wh{g}_{L}(\mb x)$ are the vectorized intermediate feature maps from pre-trained DNNs (e.g., AlexNet). For max-loss form, three methods are proposed in~\citep{laidlaw2021perceptual}: Perceptual Projected Gradient Descent (PPGD), Lagrangian Perceptual Attack (LPA) and its variant Fast Lagrangian Perceptual Attack (Fast-LPA), all developed in~\cite{laidlaw2021perceptual} based on iterative linearization and projection (PPGD), or penalty method (LPA, Fast-LPA), respectively. In~\citep{laidlaw2021perceptual}, a preset perturbation level $\eps=0.5$ is used in max-loss form (termed perceptual adversarial attack, PAT). 

\begin{figure}[!tb]
    \centering
    \includegraphics[width=0.45\textwidth]{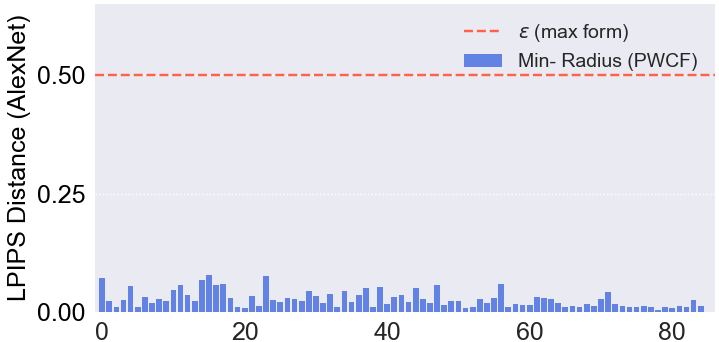}
    \caption{Robustness radius (y-axis) found by solving min-radius form using PWCF with PD on $85$ ImageNet-100 images. The x-axis shows the image indices. The \textcolor{red}{red} dashed line is the proposed bound $\eps$ used to solve max-loss form in~\cite{laidlaw2021perceptual}, which is much larger than each radius found by solving min-radius form with PWCF.} 
    \label{fig: perceptual min radius plot} 
\end{figure}

\begin{table*}[!tb]
\caption{Performance comparison of solving max-loss form with PD for the entire ImageNet-100 validation set, using (clipped) cross-entropy and margin losses, respectively. \textbf{Viol.} reports the ratio of final solutions that violate the constraint; \textbf{Att. Succ.} is the ratio of all feasible and successful attack samples divided by total number of samples---higher indicates more effective optimization performance.}
\label{tab: granso_pat_compare}
\begin{center}
\setlength{\tabcolsep}{1.0mm}{
\begin{tabular}{l c c c c c}
{}
&\multicolumn{2}{c}{\small{\textbf{cross-entropy loss}}}
&{}
&\multicolumn{2}{c}{\small{\textbf{margin loss}}}
\\
\cline{2-3}\cline{5-6}
\vspace{-10pt}
\\
\small{\textbf{Method}}
&\small{\textbf{Viol. ($\%$) $\downarrow$}}
&\small{\textbf{Att. Succ. ($\%$) $\uparrow$}}
&{}
&\small{\textbf{Viol. ($\%$) $\downarrow$}}
&\small{\textbf{Att. Succ. ($\%$) $\uparrow$}}
\\
\toprule
\small{Fast-LPA}
&{$73.8$}
&\textcolor{black}{$3.54$}
&{}
&{$41.6$}
&{$56.8$}
\\
\small{LPA}
&{\textbf{0.00}}
&{$80.5$}
&{}
&{\textbf{0.00}}
&\textcolor{black}{$97.0$}
\\
\small{PPGD}
&{$5.44$}
&{$25.5$}
&{}
&{\textbf{0.00}}
&{$38.5$}
\\
\midrule
\small{PWCF (ours)}
&{$0.62$}
&\textcolor{red}{$93.6$}
&{}
&{\textbf{0.00}}
&\textcolor{red}{$100$}
\\
\bottomrule
\end{tabular}
}
\end{center}
\end{table*}

We first plot the robustness radii\footnote{Using the same model adversarially pretrained on ImageNet dataset as in \cref{tab: granso_min}} found by PWCF in solving min-radius form in \cref{fig: perceptual min radius plot} on $85$ ImageNet-100 images. Comparing each robustness radius and the preset $\eps$ used in the max-loss form proposed in~\citep{laidlaw2021perceptual}, we observe that PWCF finds much smaller robustness radii for every sample. We can conclude that: 1) PWCF solves min-radius form with PD reasonably well; 2) the choice of $\eps$ is too large in~\cite{laidlaw2021perceptual} to be a reasonable perturbation budget in max-loss form.

Next, we solve max-loss form on the ImageNet-100 validation set with $\eps=0.5$\footnote{Using the same model as in \cref{fig: perceptual min radius plot}}, reporting both the attack success rate and the constraint violation rate of the solutions found. According to \cref{fig: perceptual min radius plot}, the sample-wise robustness radii are much smaller than the preset $\eps$, indicating that effective solvers should achieve $100\%$ attack success rate with $0\%$ violations. As shown in \cref{tab: granso_pat_compare}, PWCF with margin loss is the only one that meets this standard.

\begin{figure}
    \centering
    \includegraphics[width=0.2\textwidth]{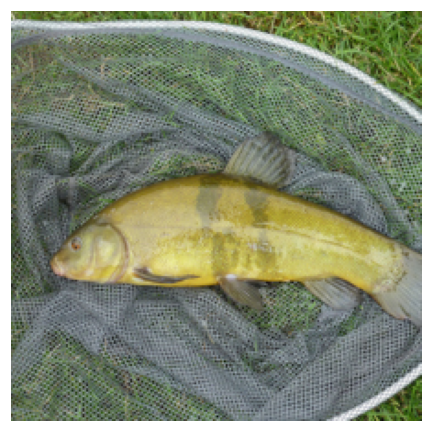}
    \caption{A `fish' image example from the Imagenet-100 validation that is used to generate the pattern visualizations in \cref{Fig:max pattern vis} and \cref{Fig:min pattern vis}.}
    \label{fig:dish image example}
\end{figure}

\section{Different combinations of $\ell$, $d$, and the solvers prefer different patterns}
\label{sec:pattern_theory}

\begin{figure*}[!tb]
\centering
\begingroup 
\setlength{\tabcolsep}{1pt}
\renewcommand{\arraystretch}{0.8}
\begin{tabular}{c c c c c c c c}
\centering
{ }
&{ }
&{ }
&\multicolumn{2}{c}{\textbf{APGD}}
&{ }
&\multicolumn{2}{c}{\textbf{PWCF}}
\\
\cline{4-5}\cline{7-8}
\vspace{-1em}
\\
{ }
&{ }
&{ }
&{cross-entropy}
&{margin}
&{ }
&{cross-entropy}
&{margin}
\\
\textbf{$\ell_1$}
&{ }
&\includegraphics[width=0.08\textwidth]{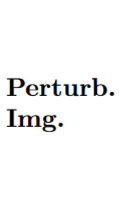}
&\includegraphics[width=0.2\textwidth]{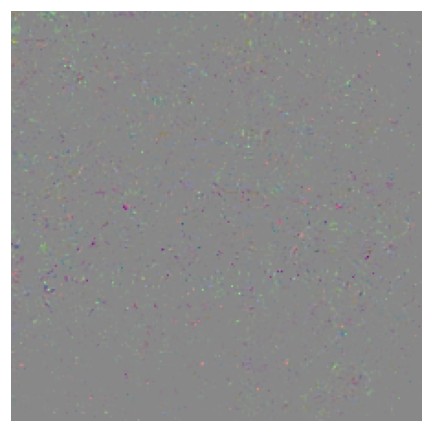}
&\includegraphics[width=0.2\textwidth]{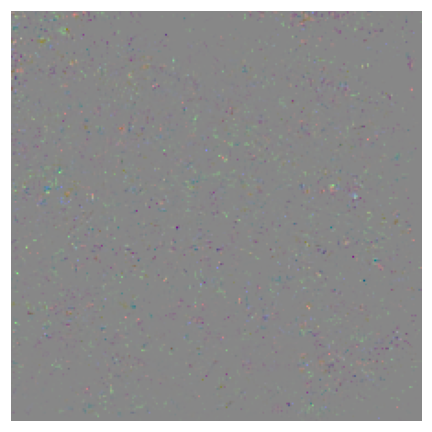}
&{ }
&\includegraphics[width=0.2\textwidth]{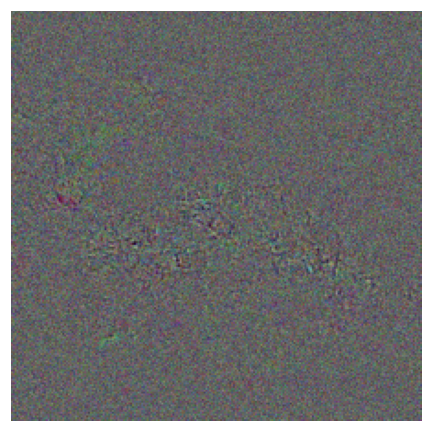}
&\includegraphics[width=0.2\textwidth]{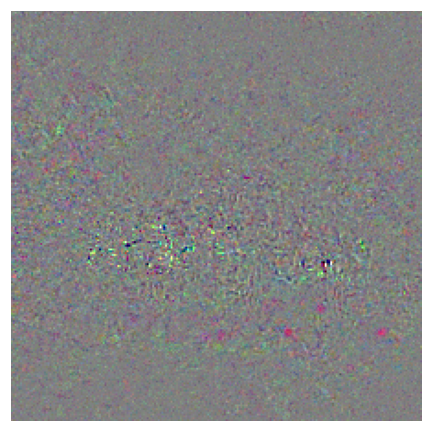}
\\
{ }
&{ }
&\includegraphics[width=0.08\textwidth]{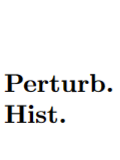}
&\includegraphics[width=0.2\textwidth]{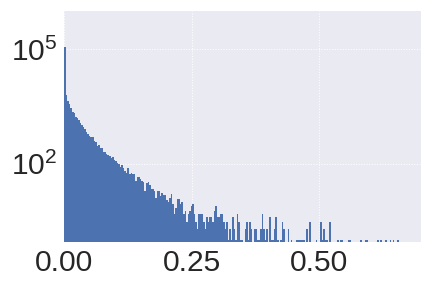}
&\includegraphics[width=0.2\textwidth]{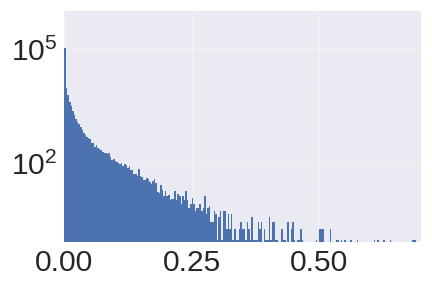}
&{ }
&\includegraphics[width=0.2\textwidth]{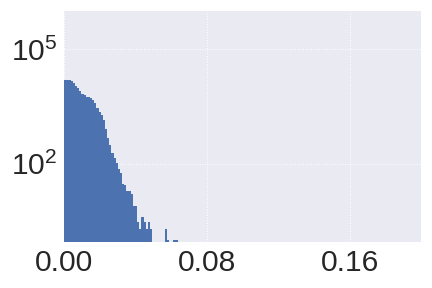}
&\includegraphics[width=0.2\textwidth]{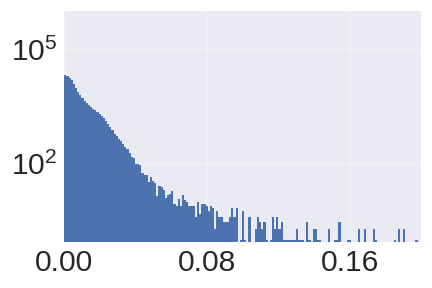}
\\
\cline{1-5}\cline{7-8}
\vspace{-1em}
\\
\textbf{$\ell_2$}
&{ }
&\includegraphics[width=0.08\textwidth]{Figures/markers/err_img.png}
&\includegraphics[width=0.2\textwidth]{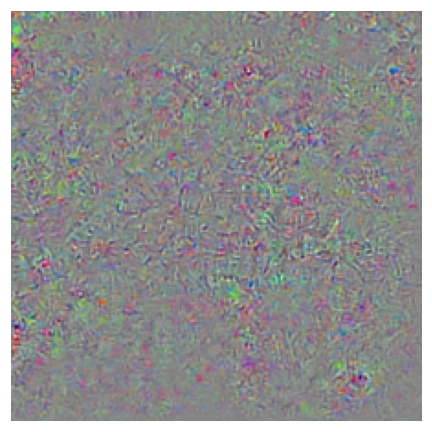}
&\includegraphics[width=0.2\textwidth]{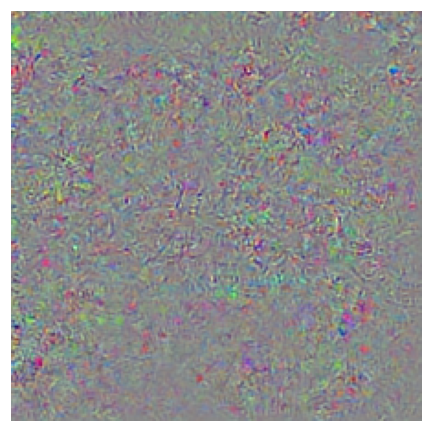}
&{ }
&\includegraphics[width=0.2\textwidth]{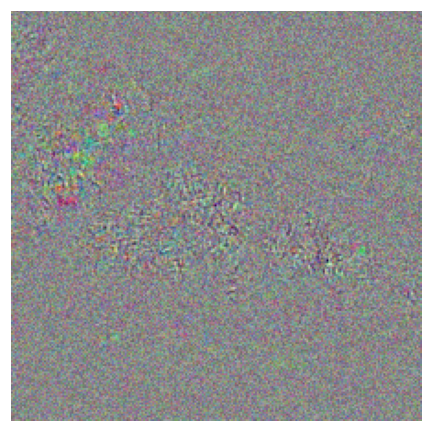}
&\includegraphics[width=0.2\textwidth]{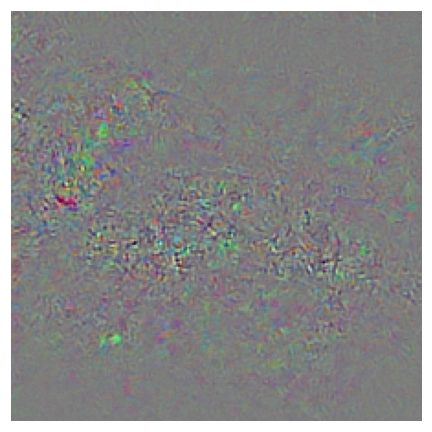}
\\
{ }
&{ }
&\includegraphics[width=0.08\textwidth]{Figures/markers/err_hist.png}
&\includegraphics[width=0.2\textwidth]{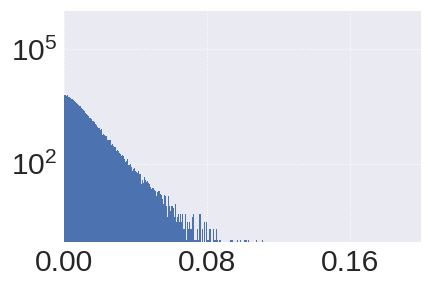}
&\includegraphics[width=0.2\textwidth]{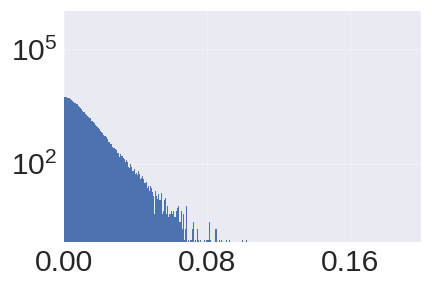}
&{ }
&\includegraphics[width=0.2\textwidth]{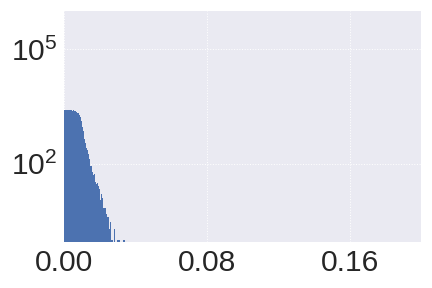}
&\includegraphics[width=0.2\textwidth]{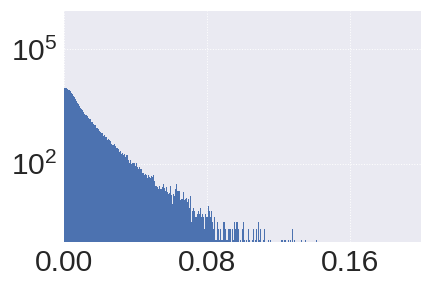}
\\
\cline{1-5}\cline{7-8}
\vspace{-1em}
\\
\textbf{$\ell_\infty$}
&{ }
&\includegraphics[width=0.08\textwidth]{Figures/markers/err_img.png}
&\includegraphics[width=0.2\textwidth]{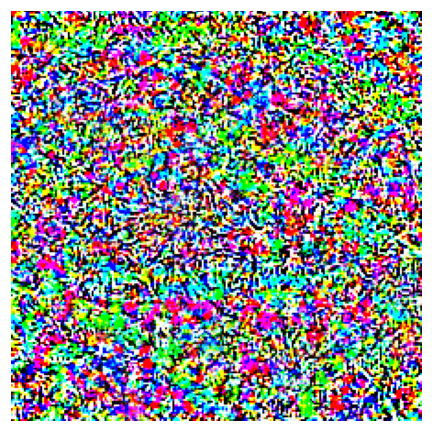}
&\includegraphics[width=0.2\textwidth]{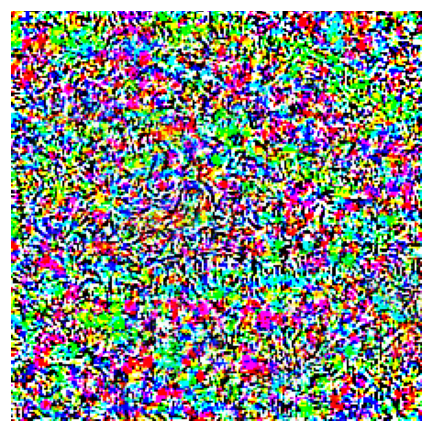}
&{ }
&\includegraphics[width=0.2\textwidth]{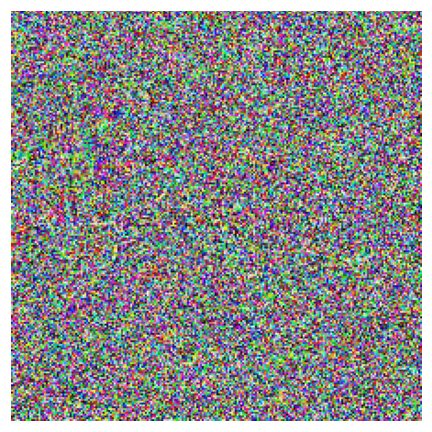}
&\includegraphics[width=0.2\textwidth]{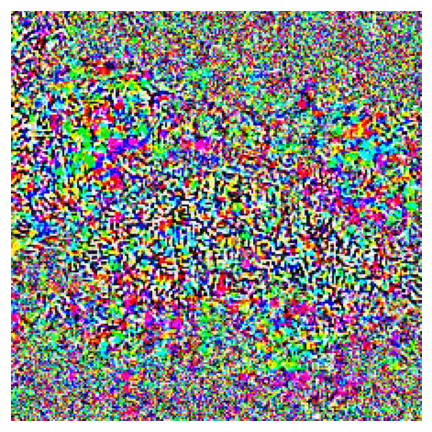}
\\
{ }
&{ }
&\includegraphics[width=0.08\textwidth]{Figures/markers/err_hist.png}
&\includegraphics[width=0.2\textwidth]{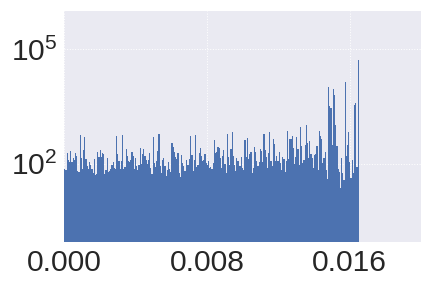}
&\includegraphics[width=0.2\textwidth]{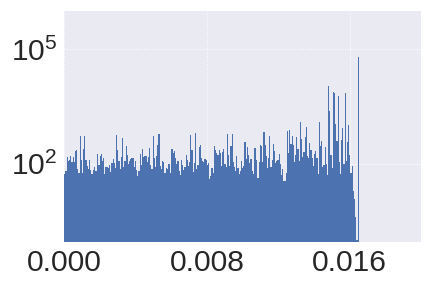}
&{ }
&\includegraphics[width=0.2\textwidth]{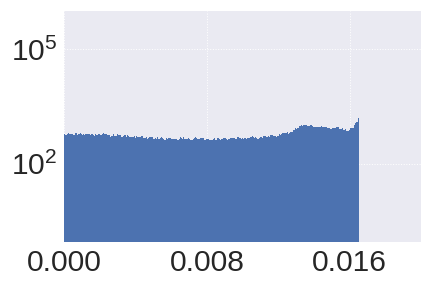}
&\includegraphics[width=0.2\textwidth]{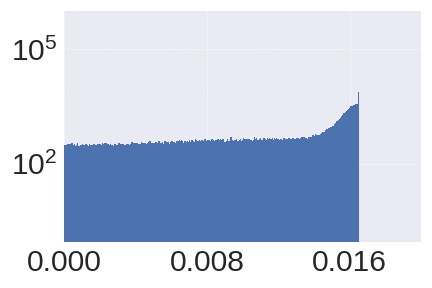}
\\
\cline{1-8}
\vspace{-1em}
\\

\end{tabular}
\endgroup 
\caption{Visualization of perturbation images found by solving max-loss form with different losses (cross-entropy and margin), different $d$'s ($\ell_1$, $\ell_2$ and $\ell_\infty$) and different solvers (APGD and PWCF). Within each group by $d$, the top rows are the perturbation images $\mb x' - \mb x$, which are normalized to the range $[0, 1]$ for better visualization; the bottom rows are the histograms of the element-wise perturbation magnitude $\abs{\mb x' - \mb x}$, where the x-axes are the absolute pixel values.}
\label{Fig:max pattern vis}
\end{figure*}
\begin{figure*}[!tb]
\centering
\begingroup 
\setlength{\tabcolsep}{1pt}
\renewcommand{\arraystretch}{0.8}
\begin{tabular}{c c c c c c c c}
\centering
{}
&{$\ell_1$}
&{$\ell_2$}
&{$\ell_\infty$}
&{ }
&{$\ell_1$}
&{$\ell_2$}
&{$\ell_\infty$}
\\
{}
&\includegraphics[width=0.16\textwidth]{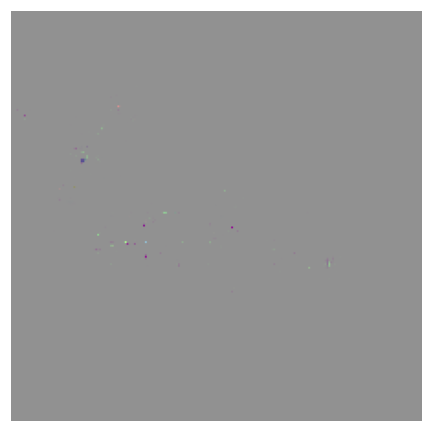}
&\includegraphics[width=0.16\textwidth]{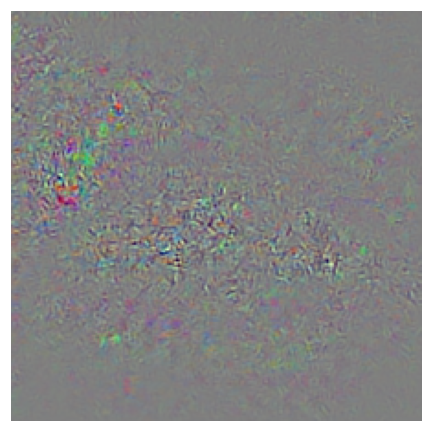}
&\includegraphics[width=0.16\textwidth]{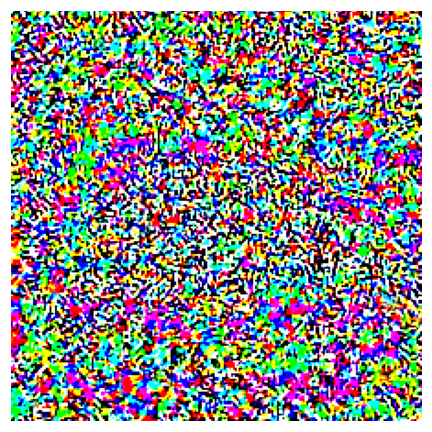}
&{ }
&\includegraphics[width=0.16\textwidth]{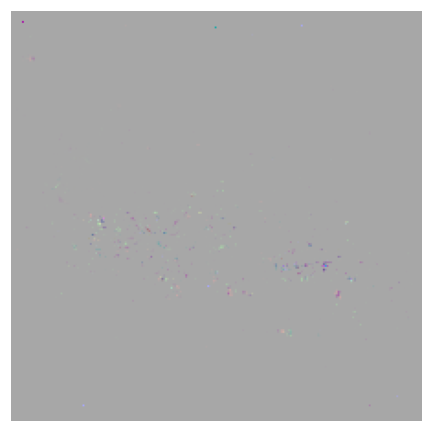}
&\includegraphics[width=0.16\textwidth]{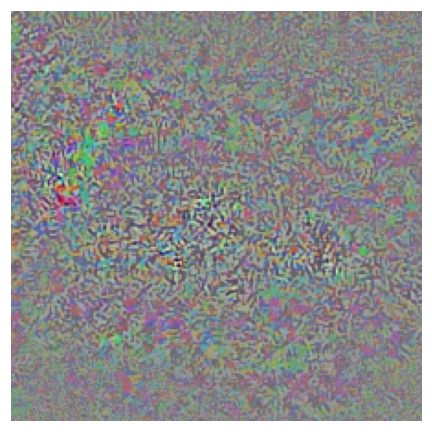}
&\includegraphics[width=0.16\textwidth]{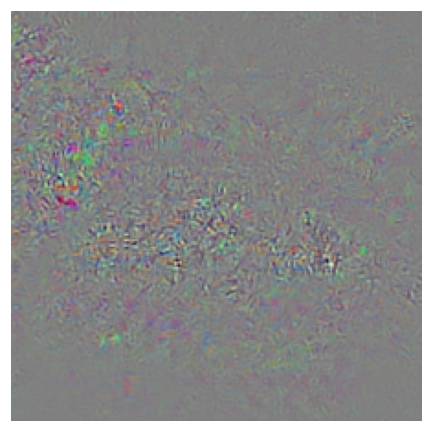}
\\
{}
&\includegraphics[width=0.16\textwidth]{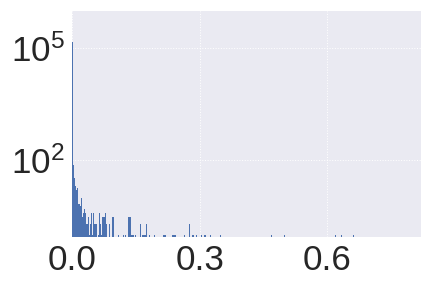}
&\includegraphics[width=0.16\textwidth]{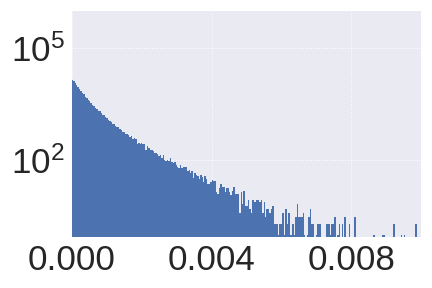}
&\includegraphics[width=0.16\textwidth]{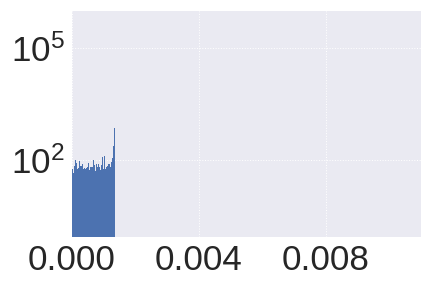}
&{ }
&\includegraphics[width=0.16\textwidth]{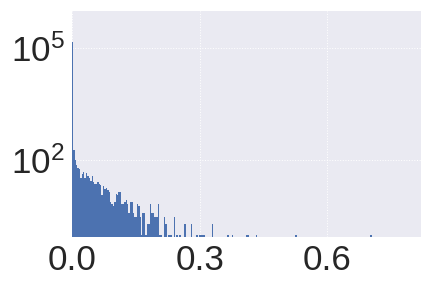}
&\includegraphics[width=0.16\textwidth]{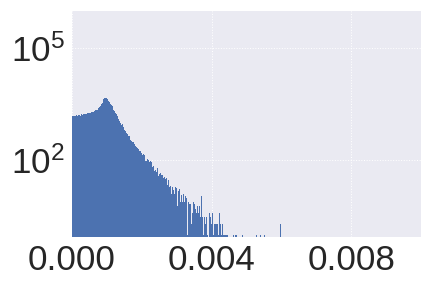}
&\includegraphics[width=0.16\textwidth]{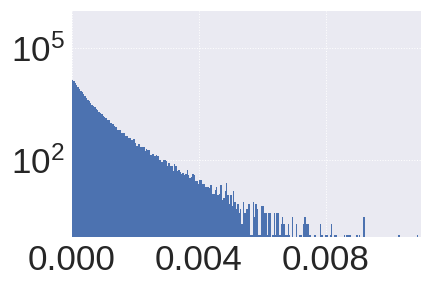}
\\
\cline{2-4}\cline{6-8}
\vspace{-1em}
\\
{}
&\multicolumn{3}{c}{\textbf{FAB}}
&{ }
&\multicolumn{3}{c}{\textbf{PWCF}}
\\
\end{tabular}
\endgroup 
\caption{Visualizations of perturbation images ($\mb x' - \mb x$, top row) and the histogram of element-wise perturbation magnitude ($\abs{\mb x' - \mb x}$, bottom row) by solving min-radius form. Note that the comparison between FAB and PWCF may not be as straightforward as \cref{Fig:max pattern vis} because the radii found by solving min-radius form are likely different in scale. However, the shape of the histograms can still reveal the pattern differences.}
\label{Fig:min pattern vis}
\end{figure*}
\begin{figure*}[!tb]
\centering
\begingroup 
\setlength{\tabcolsep}{1pt}
\renewcommand{\arraystretch}{0.8}
\begin{tabular}{c c c c c c}
\centering
{}
&{cross-entropy}
&{margin}
&{ }
&{cross-entropy}
&{margin}
\\
\includegraphics[width=0.09\textwidth]{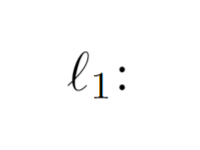}
&\includegraphics[width=0.22\textwidth]{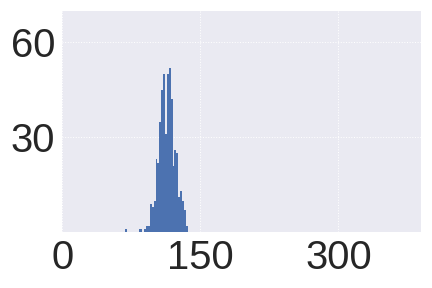}
&\includegraphics[width=0.22\textwidth]{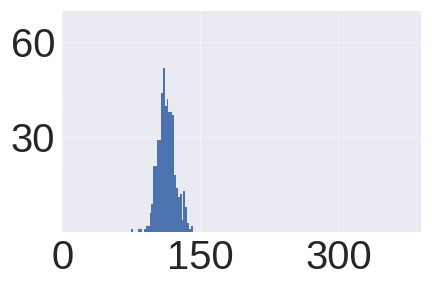}
&{ }
&\includegraphics[width=0.22\textwidth]{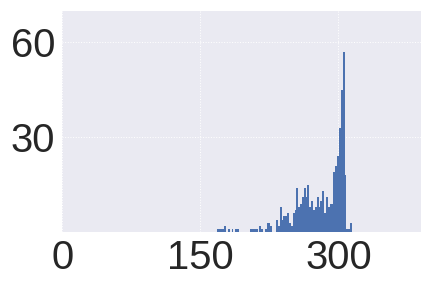}
&\includegraphics[width=0.22\textwidth]{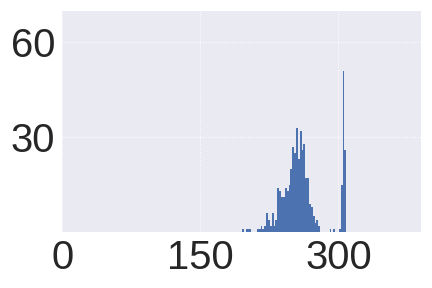}
\\
\vspace{-1em}
\\
\includegraphics[width=0.09\textwidth]{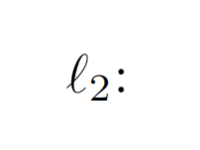}
&\includegraphics[width=0.22\textwidth]{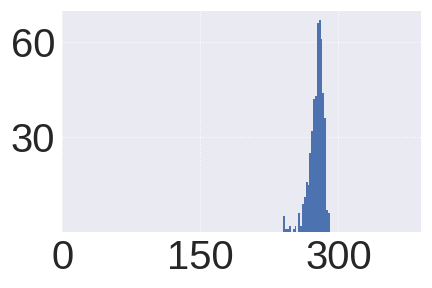}
&\includegraphics[width=0.22\textwidth]{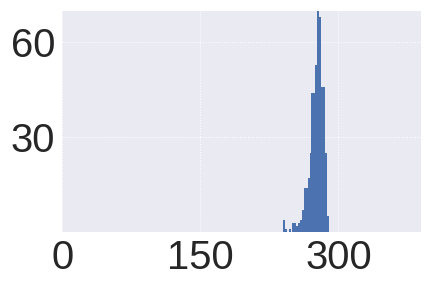}
&{ }
&\includegraphics[width=0.22\textwidth]{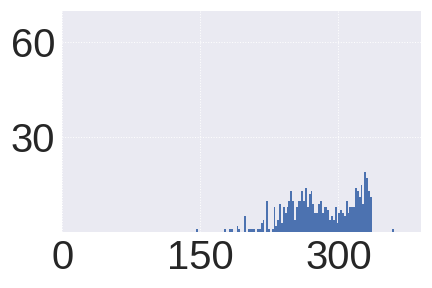}
&\includegraphics[width=0.22\textwidth]{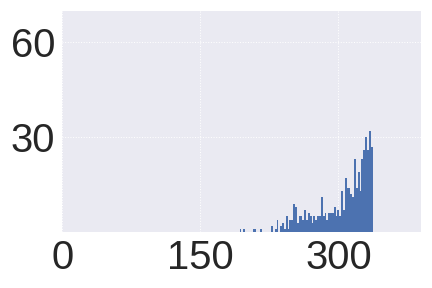}
\\
\vspace{-1em}
\\
\includegraphics[width=0.09\textwidth]{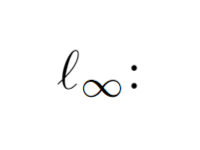}
&\includegraphics[width=0.22\textwidth]{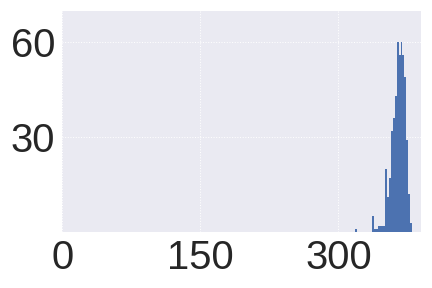}
&\includegraphics[width=0.22\textwidth]{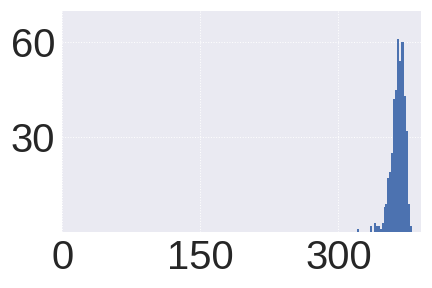}
&{ }
&\includegraphics[width=0.22\textwidth]{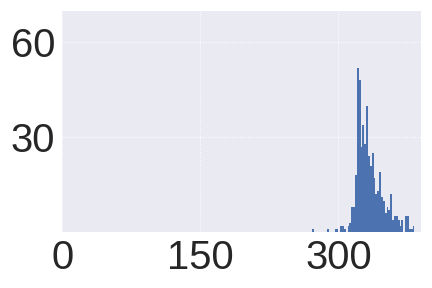}
&\includegraphics[width=0.22\textwidth]{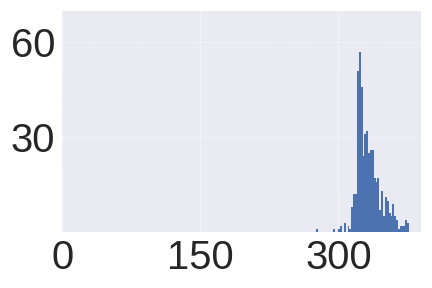}
\\
\cline{1-3}\cline{5-6}
\vspace{-1em}
\\
{}
&\multicolumn{2}{c}{\textbf{APGD}}
&{ }
&\multicolumn{2}{c}{\textbf{PWCF}}
\\

\end{tabular}
\endgroup 
\caption{Histograms of the sparsity measure by solving max-loss form with different $\ell$'s (cross-entropy and margin losses), different $d$'s ($\ell_1$, $\ell_2$ and $\ell_\infty$) and different solvers (APGD and PWCF). For fair comparisons, we use a model non-adversarially trained here so that each $x'$ is a successful adversarial example. With the same $d$ and $\ell$, the difference in the distributions of the sparsity measure between APGD and PWCF clearly exists; the sparsity variation is noticeably greater in PWCF than in APGD when the same $d$ and $\ell$ are used.}
\label{Fig:max pattern hist}
\end{figure*}
\begin{figure*}[!tb]
\centering
\begingroup 
\setlength{\tabcolsep}{1pt}
\renewcommand{\arraystretch}{0.8}
\begin{tabular}{c c c c c c c c}
\centering
{}
&{$\ell_1$}
&{$\ell_2$}
&{$\ell_\infty$}
&{ }
&{$\ell_1$}
&{$\ell_2$}
&{$\ell_\infty$}
\\
{}
&\includegraphics[width=0.16\textwidth]{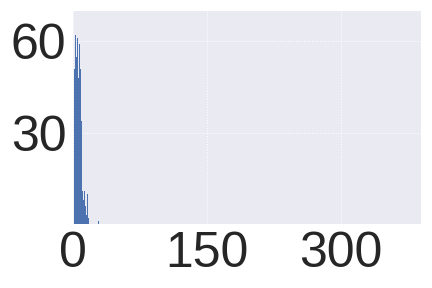}
&\includegraphics[width=0.16\textwidth]{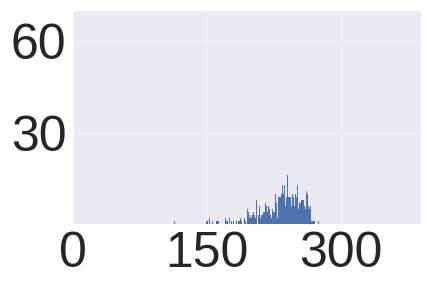}
&\includegraphics[width=0.16\textwidth]{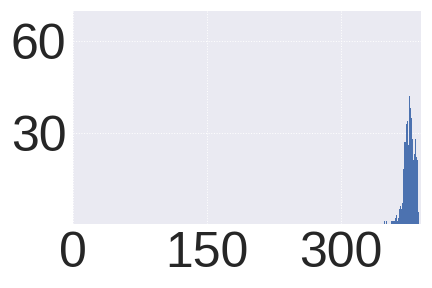}
&{ }
&\includegraphics[width=0.16\textwidth]{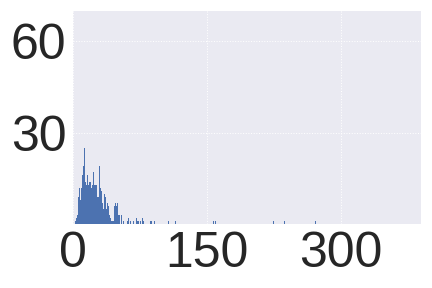}
&\includegraphics[width=0.16\textwidth]{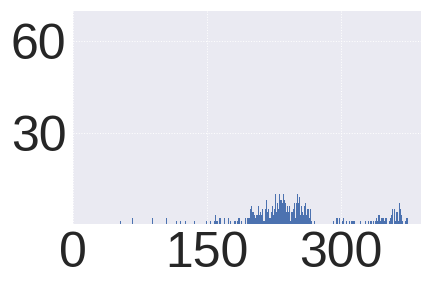}
&\includegraphics[width=0.16\textwidth]{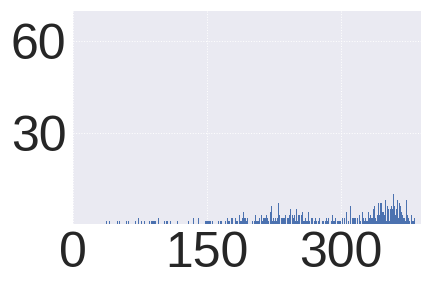}
\\
\cline{2-4}\cline{6-8}
\vspace{-1em}
\\
{}
&\multicolumn{3}{c}{\textbf{FAB}}
&{ }
&\multicolumn{3}{c}{\textbf{PWCF}}
\\
\end{tabular}
\endgroup 
\caption{Histograms of the sparsity measure by solving min-radius form with different $d$'s ($\ell_1$, $\ell_2$ and $\ell_\infty$). Under the same solver (FAB or PWCF), the shift of solution patterns from sparse to dense due to $d$ is obvious. Given the same $d$, the solutions of PWCF have more variety in sparsity than those of FAB, showing the influence of the solver on the solution patterns.}
\label{Fig:min pattern hist}
\end{figure*}

We now demonstrate that using different combinations of \textbf{1)} distance metrics $d$, \textbf{2)} solvers, and \textbf{3)} losses $\ell$ can lead to different sparsity patterns in the following two ways:
\begin{enumerate}[leftmargin=*]
    \item \textbf{Visualization of perturbation images:} we take a `fish' image (\cref{fig:dish image example}) from ImageNet-100 validation set, employ various combinations of losses $\ell$, distance metrices $d$ and solvers to the max-loss form and min-radius form. \cref{Fig:max pattern vis} and \cref{Fig:min pattern vis} visualize the perturbation image $\mb x' - \mb x$, and the histogram of the element-wise error magnitude $\abs{\mb x' - \mb x}$ to display the difference in pattern.
    \item \textbf{Statistics of sparsity levels:} we use the soft sparsity measure $\norm{\mb x' - \mb x}_1 / \norm{\mb x' - \mb x}_2$ to quantify the patterns---the higher the value, the denser the pattern. \cref{Fig:max pattern hist} and \cref{Fig:min pattern hist} display the histograms of the sparsity levels of the error images derived by solving max-loss form and min-radius form, respectively. Here, we used a fixed set of $500$ ImageNet-100 images from the validation set.
\end{enumerate} 

Contrary to the $\ell_1$ distance which induces sparsity, $\ell_\infty$ promotes dense perturbations with comparable entry-wise magnitudes~\cite{StuderEtAl2012Signal} and $\ell_2$ promotes dense perturbations whose entries follow power-law distributions. These varying sparsity patterns due to $d$'s are evident when we compare the solutions with the same solver and loss but with different distances, where \textbf{1)} the shapes of the histograms in \cref{Fig:min pattern vis} and the ranges of the values are very different; \textbf{2)} the sparsity measures show a shift from left to right along the horizontal axis in \cref{Fig:max pattern hist} and \cref{Fig:min pattern hist}. In addition to $d$'s, we also highlight other patterns induced by the loss $\ell$ and the solver:
\begin{itemize}[leftmargin=*]
    \item \textbf{Using margin and cross-entropy losses in solving max-loss form induce different sparsity patterns} \quad 
    Columns `cross-entropy' and `margin' of PWCF in \cref{Fig:max pattern vis} depict the pattern difference with clear divergences in the histograms of error magnitude; for example, the error values of PWCF-$\ell_2$-margin are more concentrated towards $0$ compared to PWCF-$\ell_2$-cross-entropy. The sparsity measures in \cref{Fig:max pattern hist} can further confirm the existence of the difference due to the loss used to solve max-loss form, more for PWCF than APGD.
    \item \textbf{PWCF's solutions have more variety in sparsity than APGD and FAB} \quad
    For the same $d$ and loss used to solve max-loss form, \cref{Fig:max pattern hist} shows that PWCF's solutions have a wider spread in the sparsity measure than APGD. The same observation can be found in \cref{Fig:min pattern hist} as well between PWCF and FAB in solving min-radius form.
\end{itemize} 

\begin{figure}[!tb]
  \centering
    \includegraphics[width=0.3\textwidth]{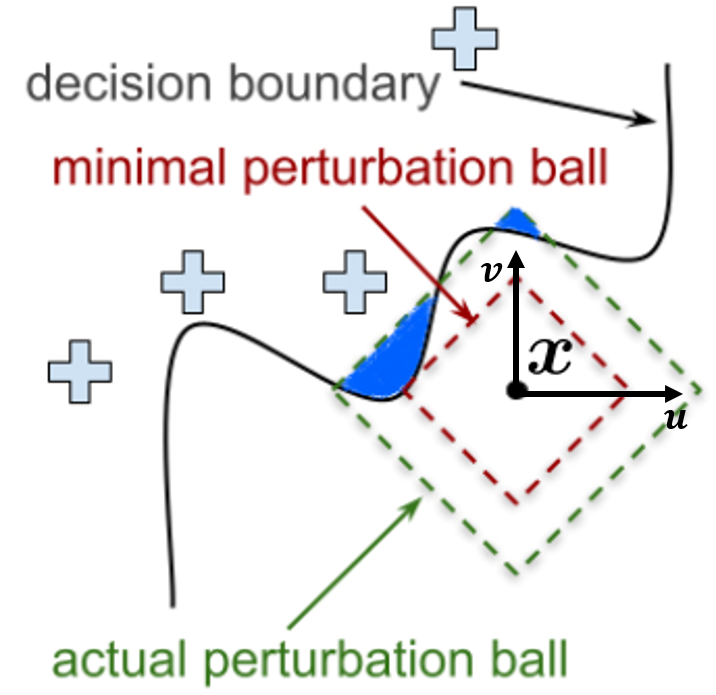}
  \caption{Geometry of max-loss form with multiple global maximizers. $\mb u$ and $\mb v$ are the basis vectors of the 2-dimensional coordinate. Here we consider the $\ell_1$-norm ball around $\mb x$, and ignore the box constraint $\mb x' \in [0, 1]^n$. Depending on the loss $\ell$ used, part or the whole of the blue regions becomes the set of global or near-global maximizers.} 
  \label{fig:multi_sol} 
\end{figure}
\begin{figure}
  \centering
  \includegraphics[width=0.4\textwidth]{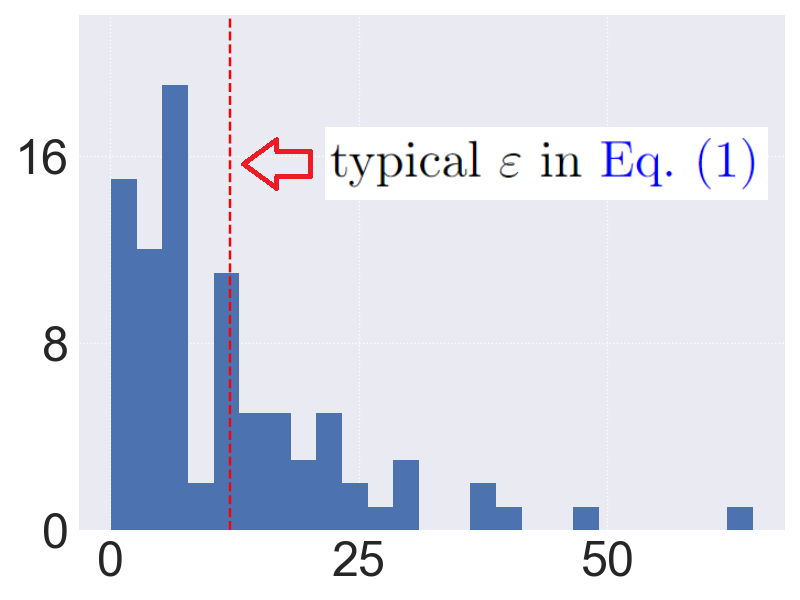}
  \caption{Histogram of the $\ell_1$ robustness radii estimated by solving min-radius form for $88$ CIFAR-10 images. $\eps=12$ (red dashed line) is the typical preset perturbation budget used in max-loss form (\formulation (\ref{eq:robust_loss})).} 
  \label{fig: eps selection} 
\end{figure}

Here, we provide a conceptual explanation of why different sparsity patterns can occur. We take the $\ell_1$ distance (i.e., $\norm{\mb x - \mb x'}_1 \le \eps$) and ignore the box constraint $\mb x' \in [0, 1]$ in max-loss form as an example. For simplicity, we take the loss $\ell$ as $0/1$ classification error $\ell (\mb y, f_{\mb \theta}(\mb x')) = \indicator{\max_{i} f^i_{\mb \theta}(\mb x') \ne y}$. Note that $\ell$ is maximized whenever $f^i_{\mb \theta}(\mb x') > f^y_{\mb \theta}(\mb x')$ for a certain $i \ne y$, so that $\mb x'$ crosses the local decision boundary between the $i$-th and $y$-th classes; see \cref{fig:multi_sol}. In practice, people set a substantially larger perturbation budget in max-loss form than the robustness radius of many samples, which can be estimated by solving min-radius form---see \cref{fig: eps selection}. Thus, there can be infinitely many global maximizers (the shaded blue regions in \cref{fig:multi_sol}). As for the patterns, the solutions in the shaded blue region on the left are denser in pattern than the solutions on the top. For other general losses, such as cross-entropy or margin loss, the set of global maximizers may change, but the patterns can possibly be more complicated due to the typically complicated nonlinear decision boundaries associated with DNNs. As for min-radius form, multiple global optimizers and pattern differences can exist as well, but the optimizers share the same radius.

\section{Implications from the variety patterns}
\label{sec: implication} 
Now that we have demonstrated the complex interplay of loss $\ell$, distance metric $d$, and the numerical solver for the final solution patterns in \cref{sec:pattern_theory}, we will discuss what this can imply for the reseach of adversarial robustness.

\subsection{Current empirical RE may not be sufficient}
\label{subsec: limiation of current RE}
As introduced in \cref{Sec:introduction}, the most popular empirical RE practice currently is solving max-loss form with a preset level of $\eps$, using a fixed set of algorithms. Then robust accuracy is reported using the perturbed samples found~\cite{croce2021robustbench, papernot2016technical, rauber2017foolbox}. Here, we challenge its validity for measuring robustness.

\subsubsection{Diversity matters for robust accuracy to be trustworthy}
\label{subsec: diversity matters for robust accuracy}
As shown in \cref{sec:pattern_theory}, the perturbations found by different numerical methods can have different sparsity patterns; \cref{{tab: granso_l1_acc}} also shows that combining multiple methods can lead to a lower robust accuracy than any single method. This implies that for robust accuracy to be numerically reliable, including as many solvers to cover as many patterns as possible is necessary. Although works as~\cite{CarliniEtAl2019Evaluating, croce2020reliable, gilmer2018motivating} have mentioned the necessity of diversity in solvers, our paper is the first to quantify such diversity in terms of sparsity patterns from their solutions. However, the existence of infinitely many patterns may be possible, and it is thus possible that faithful robust accuracy may not be able to achieve in practice.

\subsubsection{Robust accuracy is not a good robustness metric}
\label{subsec: robust accuracy is bad metric}
The motivation of using max-loss form for RE is usually associated with the attacker-defender setting, where solutions ($\mb x'$) are viewed as a test bench for all possible future attacks. Ideally, the DNNs must be robust against all of the adversarial samples found. However, it is questionable whether robust accuracy faithfully reflects this notion of robustness: 
\textbf{1)} why the commonly used budget $\eps$ in max-loss form is a reasonable choice needs to be justified. For example, $\eps=0.03$ is commonly used for the $\ell_\infty$ distance, e.g., in~\cite{croce2021robustbench}. We could not find rigorous answers in the previous literature and suspect that the choices are purely empirical. For example, \cite{croce2021robustbench} states the motivation as
\say{...the true label should stay the same for each in-distribution input within the perturbation set...}
but this claim can also support using other values; \textbf{2)} more importantly, Fig. 1. in~\cite{sridhar2022towards} shows that a model having a higher robust accuracy than other models at one $\eps$ level does not imply that such model is also more robust at other levels. The clean-robust accuracy trade-off~\cite{raghunathan2019adversarial, yang2020closer} may also be interpreted similarly\footnote{The `clean-robust accuracy trade-off' refers to the phenomenon where a model that is non-adversarially trained has the best clean accuracy (at level $\eps=0$) and the worst robust accuracy (at the commonly used $\eps$), and vise versa for the models that are adversarially trained.}---they are just most robust to different $\eps$ levels. Thus, robust accuracy is not a complete and trustworthy measure, and conclusions about robustness drawn from robust accuracy based on a single $\eps$ level are misleading.

\subsubsection{Robustness radius is a better robustness measures}
\label{subsec: min radius is better RE metric}
If our goal is indeed to understand the robustness limit of a given DNN model, solving min-loss-form seems more advantageous, especially that:
\begin{enumerate}[leftmargin=*]
    \item \textbf{Robustness radius is more reliable:} unlike that the pattern differences in solving max-loss form can lead to unreliable robust accuracy due to the possibility of multiple solutions, the robustness radius found by solving min-radius form is not sensitive to the existence of multiple solutions.
    \item \textbf{Robustness radius is sample adaptive:} in contrast to the rigid perturbation budget $\eps$ used in max-loss form, the robustness radius is the (sample-wise) distance to the closest decision boundaries.
\end{enumerate}
A clear application of the robustness radius is that we can identify hard (less robust) samples for a given model if the corresponding robustness radii are small.

\begin{figure*}[!tb]
\centering
\begingroup 
\setlength{\tabcolsep}{1pt}
\renewcommand{\arraystretch}{0.8}
\begin{tabular}{c c c c c c c c}
\centering
{}
&\multicolumn{3}{c}{\textbf{APGD cross-entropy}}
&{ }
&{\textbf{LPA}}
&{\textbf{PWCF}}
&{\textbf{PWCF}}
\\
{}
&\includegraphics[width=0.16\textwidth]{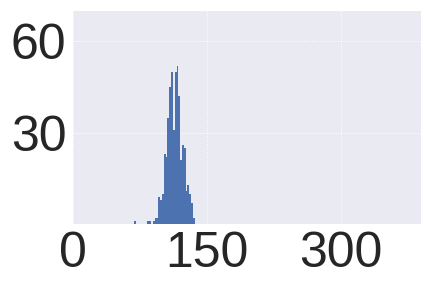}
&\includegraphics[width=0.16\textwidth]{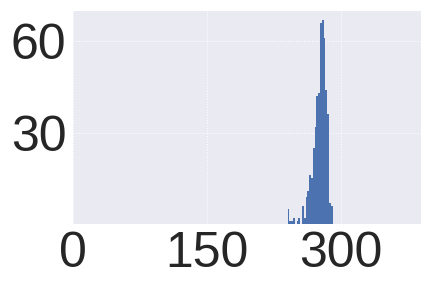}
&\includegraphics[width=0.16\textwidth]{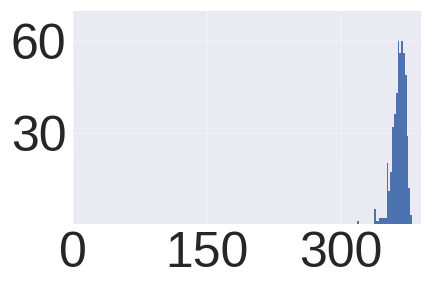}
&{ }
&\includegraphics[width=0.16\textwidth]{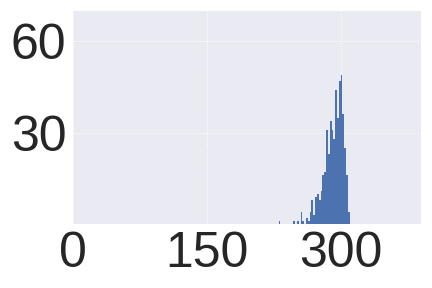}
&\includegraphics[width=0.16\textwidth]{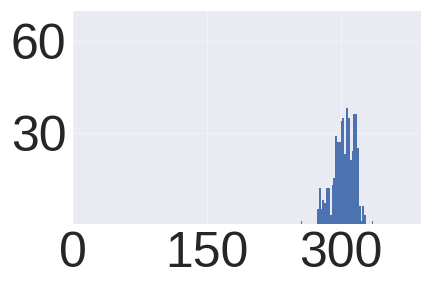}
&\includegraphics[width=0.16\textwidth]{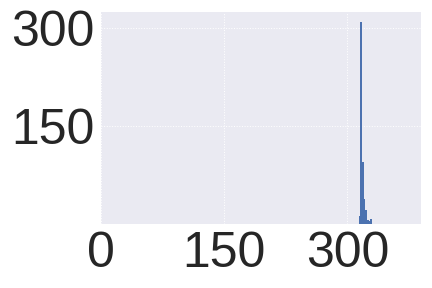}
\\
{}
&{\textbf{(a)} $\ell_1$}
&{\textbf{(b)} $\ell_2$}
&{\textbf{(c)} $\ell_\infty$}
&{ }
&{\textbf{(d)} PD-$\ell_2$}
&{\textbf{(e)} PD-$\ell_2$}
&{\textbf{(f)} PD-$\ell_1$}
\\
\end{tabular}
\endgroup 
\caption{Histograms of the sparsity measure on 500 ImageNet-100 images by solving max-loss form with different $d$'s ($\ell_p$ and PD) and different solvers (APGD with cross-entropy loss, LPA and PWCF). LPA-PD-$\ell_2$ is the perceptual attack used in~\cite{laidlaw2021perceptual}. 
Using PD in max-loss form also results in different sparsity patterns due to the solver and the inner $\ell_p$ distance used (see PD-$\ell_2$, PD-$\ell_1$ related figures above).}
\label{Fig:pat l2 sparsity}
\end{figure*}

\subsection{Adversarial training may not help to achieve generalizable robustness}
\label{subsec: difficult in achieving AR}

\subsubsection{Solution patterns can explain why $\ell_p$ robustness does not generalize}
\label{subsec: ungeneralizable lp}
Despite the effort of finding ways to achieve generalizable AR, it is widely observed that AR achieved by AT does not generalize across simple $\ell_p$ distances~\cite{maini2020adversarial,CroceHein2019Provable}. For example, models adversarially trained by $\ell_\infty$-attacks do not achieve good robust accuracy with $\ell_2$-attacks; $\ell_1$ seems to be a strong attack for all other $\ell_p$ distances, and even on itself. 
Note that \cite{madry2017towards} has observed that the (approximate) global maximizers are distinct and spatially scattered; the patterns we discussed in \cref{sec:pattern_theory} provide a plausible explanation of why AR achieved by AT is expected not to be generalizable---the model just cannot perform well on an unseen distribution (patterns) from what it has seen during training.

\subsubsection{Adversarial training with perceptual distances does not solve the generalization issue}
\label{subsec: ungeneralizable perceptual metric}
\cite{laidlaw2021perceptual} claims that using PD (\cref{Eq. LPIPS Constraint}) in max-loss form can approximate the universal set of adversarial attacks, and models adversarially trained with PAT can generalize to other unseen attacks. However, we challenge the above conclusion: `unseen attacks' does not necessarily translate to `novel perturbations', especially if we investigate the patterns:
\begin{enumerate}[leftmargin=*]
    \item If we test the models pretrained by $\ell_2$-attack and PAT\footnote{Correspond to $\ell_2$ and PAT-AlexNet in Table 3 of \cite{laidlaw2021perceptual}} in \cite{laidlaw2021perceptual} by APGD-CE-$\ell_{1}~(\eps=1200)$ attack (on ImageNet-100 images), both will achieve $0 \%$ robust accuracy---models pretraiend with PAT do not generalize better to $\ell_1$ attacks compared with others.
    \item By investigating the sparsity patterns similar to \cref{sec:pattern_theory}, the adversarial perturbations generated by solving max-loss form with PD are shown to be similar to the APGD-CE-$\ell_2$ generated ones, see (a)-(d) in~\cref{Fig:pat l2 sparsity}. This may explain why the $\ell_2$ and PAT  pre-trained models in \cite{laidlaw2021perceptual} have comparable robust accuracy against multiple tested attacks.
    \item Substituting the $\ell_2$ distance by $\ell_1$ in \cref{Eq. LPIPS Constraint} as the new PD:
    \begin{align}
        d(\mb x, \mb x') \doteq \norm{\phi(\mb x) - \phi(\mb x')}_{1} ~ \text{,}
    \end{align}
    the solution patterns will change; see (e)-(f) in~\cref{Fig:pat l2 sparsity}. Furthermore, (d)-(e) in \cref{Fig:pat l2 sparsity} also shows that different solvers (LPA and PWCF) will also result in different patterns even for PD---PAT will likely suffer from the pattern differences the same way as popular $\ell_p$-attacks, thus not being `universal'.
\end{enumerate} 
To conclude, we think that it is so far unclear whether using the perceptual distances in the AT pipeline can be beneficial in addressing the generalization issue in robustness.

\section{Discussion}
\label{sec: summary} 
In this paper, we introduce a new algorithmic framework, \pygranso~\textbf{w}ith \textbf{C}onstraint-\textbf{F}olding (PWCF), to solve two constrained optimization formulations of the robustness evaluation (RE) problems: max-loss form and min-radius form. PWCF can handle general distance metrics (almost everywhere differentiable) and achieve reliable solutions for these two formulations, which are beyond the reach of existing numerical methods. We remark that PWCF is not intended to beat the performance of the existing SOTA algorithms in these two formulations with the limited $\ell_1$, $\ell_2$ and $\ell_\infty$ distances, nor to improve the adversarial training pipeline. We view PWCF as a reliable and general numerical framework for the current RE packages, and possibly for other future emerging problems in term of constraint optimization with deep neural networks (DNNs).

In addition, we show that using different combinations of losses $\ell$, distance metrics $d$ and solvers to solve max-loss form and min-radius form can lead to different sparsity patterns in the solutions found. Having provided our explanations on why the pattern differences can happen, we also discuss its implications for the research on adversarial robustness: 
\begin{enumerate}
    \item The current practice of RE based on solving max-loss form is insufficient and misleading.
    \item Finding the sample-wise robustness radius by solving min-radius form can be a better robustness metric.
    \item Achieving generalizable robustness by adversarial training (AT) may be intrinsically difficult.
\end{enumerate}

\backmatter





\section*{Acknowledgments}
We thank Hugo Latapie of Cisco Research for his insightful comments on an early draft of this paper. Hengyue Liang, Le Peng, and Ju Sun are partially supported by NSF CMMI 2038403 and Cisco Research under the awards SOW 1043496 and 1085646 PO USA000EP390223. Ying Cui is partially supported by NSF CCF 2153352. The authors acknowledge the Minnesota Supercomputing Institute (MSI) at the University of Minnesota for providing resources
that contributed to the research results reported within this paper.












\appendix 

\section{Appendix} 

\subsection{Projection onto the intersection of a norm ball and box constraints}
\label{Sec:app_projection}
APGD for solving max-loss form (\ref{eq:robust_loss}) with $\ell_p$ distances entails solving Euclidean projection subproblems of the form: 
\begin{align}
\begin{split}
    & \min_{\mb x' \in \RJU^n} \;  \norm{\mb z - \mb x'}_2^2 \\
    \st\; & \norm{\mb x - \mb x'}_p \le \eps, \quad \mb x' \in [0, 1]^n,
\end{split}
\end{align}
where $\mb z = \mb x + \mb w$ is a one-step update of $\mb x$ toward direction $\mb w$. After a simple reparametrization, we have 
\begin{align}  
\label{eq:app_proj_infty}
\begin{split}
    & \min_{\mb \delta \in \RJU^n} \; \norm{\mb w - \mb \delta}_2^2 \\
    \st \; & \norm{\mb \delta}_p \le \eps, \quad \mb x + \mb \delta \in [0, 1]^n.
\end{split}
\end{align}

We focus on $p = 1, 2, \infty$ which are popular in the AR literature. In early works, a ``lazy'' projection scheme---sequentially projecting onto the $\ell_p$ ball and then to the $[0, 1]^n$ box, is used. \cite{croce2021mind} has recently identified the detrimental effect of lazy projection on the performance for $p=1$, and has derived a closed form solution. Here, we prove the correctness of the sequential projection for $p = \infty$ (\cref{thm:app_proj_inf_lemma}), and discuss problems regarding $p = 2$ (\cref{thm:app_proj_l2_lemma}). 

For $p = \infty$, obviously we only need to consider the individual coordinates. 
\begin{lemma} \label{thm:app_proj_inf_lemma}
Assume $x \in [0, 1]$. The unique solution for the strongly convex problem
\begin{align}
\begin{split}
    & \min_{\delta \in \RJU} \; \paren{w - \delta}^2\\
    \st\; & \abs{\delta} \le \eps, \quad x + \delta \in [0, 1]
\end{split}
\end{align}
is given by
\begin{multline}  \label{eq:app_proj_inf_formula}
    \mc P_{\infty, \mathrm{box}} = \\
        \begin{cases} 
            w,\quad w \in [\max(-x, -\eps), \min (1-x, \eps)]\\
        \max(-x, -\eps), \quad w \le \max(-x, -\eps) \\
        \min (1-x, \eps), \quad w \ge \min (1-x, \eps) 
    \end{cases}, 
\end{multline} 
which agrees with the sequential projectors $\mc P_{\infty} \mc P_{\mathrm{box}}$ and $\mc P_{\mathrm{box}} \mc P_{\infty}$.  
\end{lemma} 
One can easily derive the one-step projection formula \cref{eq:app_proj_inf_formula} once the two box constraints can be combined into one: 
\begin{align} \label{eq:app_proj_inf_equiv_constr}
    \max(-\eps, -x) \le \delta \le \min(\eps, 1-x). 
\end{align} 
To show the equivalence to $\mc P_{\infty} \mc P_{\mathrm{box}}$ and $\mc P_{\mathrm{box}} \mc P_{\infty}$, we could write down all projectors analytically and directly verify the claimed equivalence. But that tends to be cumbersome. Here, we invoke an elegant result due to \cite{yu2013decomposing}. For this, we need to quickly set up the notation. For any function $f: \RJU^n \to \RJU \cup \setJu{+ \infty}$, its proximal mapping $\mathrm{Prox}_f$ is defined as 
\begin{align} 
    \mathrm{Prox}_f(\mb y) = \argmin_{\mb z \in \RJU^n} \frac{1}{2} \norm{\mb y - \mb z}_2^2 + f(\mb z).
\end{align} 
When $f$ is the indicator function $\imath_C$ for a set $C$ defined as 
\begin{align}
    \imath_C(\mb z) = 
    \begin{cases} 
        0   &   \mb z \in C \\
            \infty &   \text{otherwise}   
         \end{cases},
\end{align}
then $\mathrm{Prox}_f(\mb y)$ is the Euclidean projector $\mc P_C(\mb y)$. For two closed proper convex functions $f$ and $g$, the paper \cite{yu2013decomposing} has studied conditions for $\mathrm{Prox}_{f+g} = \mathrm{Prox}_{f} \circ \mathrm{Prox}_{g}$. If $f$ and $g$ are two set indicator functions, this exactly asks when the sequential projector is equivalent to the true projector. 
\begin{theorem}[adapted from Theorem 2 of \cite{yu2013decomposing}]
  If $f = \imath_C$ for a closed convex set $C \subset \RJU$, $\mathrm{Prox}_f \circ \mathrm{Prox}_g = \mathrm{Prox}_{f+g}$ for all closed proper convex functions $g: \RJU \to \RJU \cup \setJu{\pm \infty}$. 
\end{theorem} 
The equivalence of projectors we claim in \cref{thm:app_proj_inf_lemma} follows by setting $f = \imath_{\infty}$ and $g = \imath_{\mathrm{box}}$, and vise versa. 

For $p = 2$, sequential projectors are not equivalent to the true projector in general, although empirically we observe that $\mc P_{2} \mc P_{\mathrm{box}}$ is a much better approximation than $\mc P_{\mathrm{box}} \mc P_{2}$. The former is used in the current APGD algorithm of \texttt{AutoAttack}. 
\begin{lemma} \label{thm:app_proj_l2_lemma}
    Assume $\mb x \in [0, 1]^n$. When $p = 2$, the projector for (\ref{eq:app_proj_infty}) $\mc P_{2, \mathrm{box}}$ does not agree with the sequential projectors $\mc P_{2} \mc P_{\mathrm{box}}$ and $\mc P_{\mathrm{box}} \mc P_{2}$ in general. However, both $\mc P_{2} \mc P_{\mathrm{box}}$ and $\mc P_{\mathrm{box}} \mc P_{2}$ always find feasible points for the projection problem. 
\end{lemma} 
\begin{proof} 
\begin{figure}[!htbp]
    \centering 
    \includegraphics[width=0.6\linewidth]{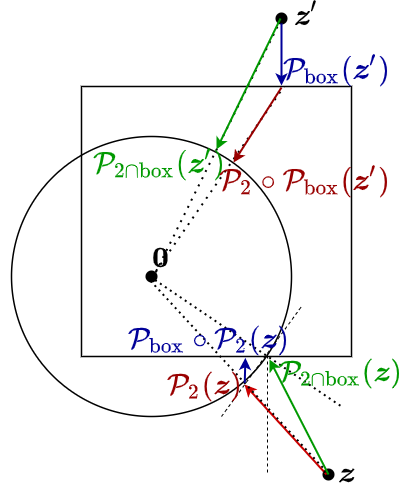}
    \caption{Illustration of the problem with the sequential projectors when $p = 2$. In general, neither of the sequential projectors produces the right projection. }
    \label{fig:l2_box_projection}
\end{figure}
For the non-equivalence, we present a couple of counter-examples in \cref{fig:l2_box_projection}. Note that the point $\mb z$ is inside the normal cone of the bottom right corner point of the intersection: $\setJu{\mb \delta \in \RJU^2: \norm{\mb \delta}_2 \le \eps} \cap \setJu{\mb \delta \in \RJU^2: \mb x + \mb \delta \in [0, 1]^2}$. 

For the feasibility claim, note that for any $\mb y \in \RJU^n$
\begin{align} 
    \mc P_2\paren{\mb y} = 
    \begin{cases}
        \eps \frac{\mb y}{\norm{\mb y}_2} &  \norm{\mb y}_2 \ge \eps \\
        \mb y   &  \text{otherwise}
    \end{cases}
\end{align} 
and for any $y \in \RJU$, 
\begin{align} 
    \mc P_{\mathrm{box}}\paren{y} = 
    \begin{cases}
        1-x &  y \ge 1-x \\
        y  &   -x < y < 1-x \\ 
        -x  & y \le -x 
    \end{cases}
\end{align} 
and $\mc P_{\mathrm{box}}(\mb y)$ acts on any $\mb y \in \RJU^n$ element-wise. For any $\mb y$ inside the $\ell_2$ ball, 
\begin{multline} 
    \norm{\mc P_{\mathrm{box}}\paren{\mb y}}_2 
    = \norm{\mc P_{\mathrm{box}}\paren{\mb y} - \mc P_{\mathrm{box}}\paren{\mb 0}}_2 \\
     \le \norm{\mb y - \mb 0}_2 = \norm{\mb y}_2 \le \eps
\end{multline} 
due to the contraction property of projecting onto convex sets. Therefore, $\mc P_{\mathrm{box}} \mc P_{2}(\mb y)$ is feasible for any $\mb y \in \RJU^n$. Now for any $\mb y$ inside the box: 
\begin{itemize} 
   \item if $\norm{\mb y}_2 < \eps$, $\mc P_2(\mb y) = \mb y$ and $\mc P_2(\mb y)$ remains in the box; 
   \item if $\norm{\mb y}_2 \ge \eps$, $\mc P_2(\mb y) = \eps \frac{\mb y}{\norm{\mb y}_2}$. Since $\eps/\norm{\mb y}_2 \in [0, 1]$, $\mb P_2(\mb y)$ shrinks each component of $\mb y$, but retains their original signs. Thus, $\mb P_2(\mb y)$ remains in the box if $\mb y$ is in the box. 
\end{itemize} 
We conclude that $\mc P_{2} \mc P_{\mathrm{box}}(\mb y)$ is feasible for any $\mb y$, completing the proof. 
\end{proof}

\subsection{Sketch of the BFGS-SQP algorithm in GRANSO}
\label{Sec:granso_summary} 
GRANSO is among the first optimization solvers targeting general non-smooth, non-convex problems with non-smooth constraints \cite{curtis2017bfgs}.

The key algorithm of the GRANSO package is a sequential quadratic programming (SQP) that employs a quasi-Newton algorithm, Broyden-Fletcher-Goldfarb-Shanno (BFGS) \cite{wright1999numerical}, and an exact penalty function (Penalty-SQP). 

The penalty-SQP calculates the search direction from the following quadratic programming (QP) problem:
\begin{align}
\label{penalty_sqp}
    \begin{split}
        \min_{\vd \in \RJU^n, \vs \in \RJU^p} \; & \mu \paren{f\paren{\vx_k} + \nabla f \paren{\vx_k}^{\TJU} \vd} \\
        + & \ve^{\TJU}\vs + \frac{1}{2} \vd^{\TJU} \mH_k \vd \\
        \st \; c \paren{\vx_k} + & \nabla c \paren{\vx_k}^{\TJU} \vd \leq \vs, \quad \vs \geq 0.
    \end{split}
\end{align}
Here, we abuse the notation $c(\cdot)$ to be the total constraints for simplicity (i.e., representing all $c$'s and $h$'s in \cref{eq:NO_form}).
The dual of problem (\cref{penalty_sqp}) is used in the GRANSO package:
\begin{align}
\label{penalty_sqp_dual}
    \begin{split}
        \max_{\vlambda\in\RJU^p} \; & \mu f (\vx_k) + c(\vx_k)^{\TJU}\vlambda \\
        & - \frac{1}{2} \paren{ \mu \nabla f (\vx_k) + \nabla c(\vx_k) \vlambda }^{\TJU} \mH_k^{-1} \\
        & \cdot \; \paren{ \mu \nabla f (\vx_k) + \nabla c(\vx_k) \vlambda } \\
        \st & \quad 0 \leq \vlambda \leq \ve,
    \end{split}
\end{align}
which has only simple box constraints that can be easily handled by many popular QP solvers such as OSQP (ADMM-based algorithm) \cite{osqp}.
Then the primal solution $\vd$ can be recovered from the dual solution $\vlambda$ by solving (\ref{penalty_sqp_dual}):
\begin{align}\label{searching_direction}
    \vd = - \mH_k^{-1} \paren{ \mu \nabla f (\vx_k) + \nabla c(\vx_k) \vlambda }.
\end{align}

The search direction calculated at each step controls the trade-off between minimizing the objective and moving towards the feasible region. To measure how much violence the current search direction will give, a linear model of constraint violation is used:
\begin{align}\label{contr_viol}
    l(\vd;\vx_k) \defeq \norm{\max\Brac{c(\vx_k)+\nabla c(\vx_k)^{\TJU}\vd,\bm{0}}}_{1}.
\end{align}

To dynamically set the penalty parameter, a steering strategy as \cref{alg:steering} is used:

\begin{algorithm}[!tb]
\caption{\\ \text{ } \quad \quad $\brac{\vd_k,{\mu}_{\text{new}}}= \texttt{sqp\_steering}(\vx_k,\mH_k,{\mu})$}
\label{alg:steering}
\begin{algorithmic}[1]
\Require $\vx_k, \mH_k,\mu$ at current iteration 
\Require constants $c_v\in(0,1),c_{\mu} \in (0,1)$ 
\State Calculate $\vd_k$ from \cref{searching_direction} and \formulation (\ref{penalty_sqp_dual}) with $\mu_{\text{new}} = \mu$
\If{$l_{\delta}(\vd_k;\vx_k)< c_v v(\vx_k)$}
\State Calculate $\tilde{\vd_k}$ from \cref{searching_direction} and \formulation (\ref{penalty_sqp_dual}) with $\mu=0$
\While{$l_{\delta}(\vd_k;\vx_k)< c_v l_{\delta}(\tilde{\vd_k};\vx_k)$}
\State $\mu_{new} \defeq c_{\mu} \mu_{\text{new} }$
\State Calculate $\vd_k$ from \cref{searching_direction} and \formulation (\ref{penalty_sqp_dual}) with $\mu=\mu_{\text{new}}$
\EndWhile
\EndIf 
\State \Return{$\vd_k,{\mu}_{\text{new}}$}
\end{algorithmic}
\end{algorithm}

For non-smooth problems, it is usually hard to find a reliable stopping criterion, as the norm of the gradient will not decrease when approaching a minimizer. GRANSO uses an alternative stopping strategy, which is based on the idea of gradient sampling \cite{lewis2013nonsmooth} \cite{burke2020gradient}.

Define the neighboring gradient information (from the $p$ most recent iterates) as:
\begin{align}
    \begin{split}
        & \mG \defeq \brac{\nabla f(\vx_{x_{k+1-l}})\hdots \nabla f(\vx_k)} \text{,}\\
    & \mJ_i \defeq \brac{\nabla c_i(\vx_{x_{k+1-l}})\hdots \nabla c_i (\vx_k)}\text{,} \\ & \; i \in \Brac{1,\hdots ,p}
    \end{split}
\end{align}

Augment \formulation (\ref{penalty_sqp}) and its dual \formulation (\ref{penalty_sqp_dual}) in the steering strategy, we can obtain the augmented dual problem:
\begin{align}\label{QP_termination}
    \begin{split}
        \max_{\vsigma\in \RJU^l, \vlambda\in \RJU^{pl}} 
        \sum_{i=1}^p & c_i(\vx_k)\ve^{\TJU} \vlambda_i  \\
        - \frac{1}{2} \begin{bmatrix} \vsigma \\ \vlambda \end{bmatrix}^{\TJU} \begin{bmatrix} \mG,\mJ_1,\hdots,\mJ_p \end{bmatrix}^{\TJU} & \mH_k^{-1} \begin{bmatrix} \mG,\mJ_1,\hdots,\mJ_p \end{bmatrix}\begin{bmatrix} \vsigma \\ \vlambda \end{bmatrix} \\
         \st \; \bm{0} \leq \vlambda_i \leq \ve, \quad & \ve^{\TJU}\vsigma = \mu, \quad \vsigma \geq \bm{0}
    \end{split}
\end{align}

By solving \cref{QP_termination}, we can obtain $\vd_{\diamond}$:
\begin{align}\label{d_diamond}
    \begin{split}
    &    \vd_{\diamond} = \mH_k^{-1} \begin{bmatrix} \mG,\mJ_1,\hdots,\mJ_p \end{bmatrix} \begin{bmatrix} \vsigma \\ \vlambda \end{bmatrix}
    \end{split}
\end{align}
If the norm of $\vd_{\diamond}$ is sufficiently small, the current iteration can be viewed as near a small neighborhood of a stationary point.

\begin{algorithm}[!tb]
\caption{$ \\ \text{ } \quad \quad \brac{ \vx_*,f_*,\vv_* } = \texttt{bfgs\_sqp}\paren{f(\cdot),\vc(\cdot), \vx_0, \mu_0})$ }
\label{alg:bfgs_sqp}
\begin{algorithmic}[1]
\Require $f,\vc,\vx_0,\mu_0$
\Require  constants $\tau_{\diamond},\tau_{v}$
\State $\mH_0 \defeq \mI,\ \mu \defeq \mu_0$
\State $\phi(\cdot) = \mu f(\cdot) + v(\cdot)$
\State $\nabla \phi(\cdot) = \mu \nabla f(\cdot) + \sum_{i \in \mathcal{P} } \nabla c_i(\cdot) $
\State $v(\cdot) = \norm{\max \Brac{c(\cdot),0}}_{1}$
\State  $\phi_0 \defeq \phi(\vx_0;\mu),\nabla \phi_0 \defeq \nabla \phi(\vx_0;\mu),v_0 \defeq v(\vx_0)$
\For {$k = 0,1,2,\hdots$}
\State $\brac{\vd_k,\hat{\mu}} \defeq \texttt{sqp\_steering}(\vx_k,\mH_k,\bm{\mu})$
\If{$\hat{\mu}< \mu$}
\State $\mu \defeq \hat{\mu}$
\State  $\phi_k \defeq \phi(\vx_k;\mu),\nabla \phi_k  \defeq \nabla \phi(\vx_k;\mu),v_k \defeq v(\vx_k)$
\EndIf
\State{$\brac{\vx_{k+1},\phi_{k+1}, \nabla \phi_{k+1}, v_{k+1} } \defeq \texttt{Armijo\_Wolfe}\paren{\vx_k,\phi_k,\nabla \phi_k, \phi(\cdot) , \nabla \phi(\cdot) } $}
\State Get $\vd_{\diamond}$ from \cref{d_diamond} and \formulation (\ref{QP_termination})
\If{$\norm{\vd_{\diamond}}_{2}< \tau_{\diamond}$ and $v_{k+1}<\tau_v$}
\State break
\EndIf
\State BFGS update $\mH_{k+1}$
\EndFor
\State \Return{$\vx_*,f_*,\vv_*$}
\end{algorithmic}
\end{algorithm}

\subsection{Danskin's theorem and min-max optimization}
\label{sec:danskin_minmax}
In this section, we discuss the importance of computing a good solution to the inner maximization problem when applying first-order methods for AT, i.e., solving \formulation (\ref{eq:minmax_obj}). 

Consider the following minimax problem: 
\begin{equation}\label{app:danskin}
\min_{\mb \theta} g(\mb \theta) \doteq  \left[ \max_{\mb x' \in \Delta} \;  h(\mb \theta, \mb x')\right], 
\end{equation}
where we assume that the function $h$ is locally Lipschitz continuous. To apply first-order methods to solve \cref{app:danskin}, one needs to evaluate a (sub)gradient of $g$ at any given $\mb \theta$. If $h(\mb \theta, \mb x')$ is smooth in $\mb \theta$, one can invoke Danskin's theorem for such an evaluation (see, for example, \cite[Appendix A]{madry2017towards}). However, in DL applications with non-smooth activations or losses, $h(\mb \theta, \delta)$ is not differentiable in $\mb \theta$, and hence a general version of Danskin's theorem is needed.

To proceed, we first introduce a few basic concepts in nonsmooth analysis; general background can be found in~\cite{Clarke1990Optimization,BagirovEtAl2014Introduction,CuiPang2021Modern}.  For a locally Lipschitz continuous function $\varphi:\mathbb{R}^n \to \mathbb{R}$, define its Clarke directional derivative at $\bar{\mb z}\in \mathbb{R}^n$ in any direction $\mb d\in \mathbb{R}^n$ as
\[
\varphi^\circ(\bar{\mb z}; \mb d)\doteq \limsup_{t\downarrow 0, \mb z \to \bar{\mb z}} \frac{\varphi(\mb z+t \mb d) - \varphi(\mb z)}{t}
\]
We say $\varphi$ is Clarke regular at $\bar{\mb z}$ if $\varphi^\circ(\bar{\mb z};\mb d) = \varphi^\prime(\bar{\mb z};\mb d)$ for any $\mb d\in \mathbb{R}^n$, where $\varphi^\prime(\bar{\mb z};\mb d) \doteq \displaystyle\lim_{t\downarrow 0} \frac{1}{t} \paren{\varphi(\bar{\mb z} +t \mb d) - \varphi(\bar{\mb z})}$ is the usual one-sided directional derivative. The Clarke subdifferential of $\varphi$ at $\bar{\mb z}$ is defined as
\[
\partial \varphi (\bar{\mb z}) \doteq \left\{\mb v \in \mathbb{R}^n: \varphi^\circ(\bar{\mb z};\mb d) \geq \mb v^\TJU \mb d \right\}
\]
The following result has its source in \cite[Theorem 2.1]{clarke1975generalized}; see also \cite[Section 5.5]{CuiPang2021Modern}.
\begin{theorem}
Assume that $\Delta$ in \cref{app:danskin} is a compact set, and the function $h$ satisfies
\begin{enumerate}
    \item $h$ is jointly upper semicontinuous in $(\mb \theta, \mb x')$;
    \item $h$ is locally Lipschitz continuous in $\mb \theta$, and the Lipschitz constant is uniform in $\mb x' \in \Delta$;
    \item $h$ is directionally differentiable in $\mb \theta$ for all $\mb x' \in \Delta$; 
\end{enumerate}
If $h$ is Clarke regular in $\mb \theta$ for all $\mb \theta$, and $\partial_{\mb \theta} h$ is upper semicontinuous in $(\mb \theta, \mb x')$, we have that for any $\bar{\mb \theta}$ 
\begin{align} 
\partial g(\bar{\mb \theta}) = \mathrm{conv}\{\partial h(\bar{\mb \theta},\mb x'): \mb x' \in \Delta^*(\bar{\mb \theta})\}
\end{align} 
where $\mathrm{conv}(\cdot)$ denotes the convex hull of a set, and $\Delta^*(\bar{\mb \theta})$ is the set of all optimal solutions of the inner maximization problem at $\bar{\mb \theta}$.
\end{theorem}
The above theorem indicates that in order to get an element from the subdifferential set $\partial g(\bar{\mb \theta})$, we need to get at least one {\bf optimal} solution $\mb x' \in \Delta^*(\bar{\mb \theta})$. A suboptimal solution to the inner maximization problem may result in a useless direction for the algorithm to proceed. To illustrate this, let us consider a simple one-dimensional example
\[
\min_\theta  g(\theta) := \left[\, \max_{-1\leq x' \leq 1}  \;\; \max(\theta x', 0)^2\,\right]
\]
which corresponds to a one-layer neural network with one data point $(0,0)$, the ReLU activation function and squared loss. Starting at $\theta_0 = 1$, we get the first inner maximization problem $\max_{-1\leq x' \leq 1} \max(x',0)^2$. Although its global optimal solution is $x'_* = 1$, the point $x' = 0$ is a stationary point satisfying the first-order optimality condition. If the latter point is mistakenly adopted to compute an element in $\partial g(\theta^0)$, it would result in a zero direction so that the overall gradient descent algorithm cannot proceed. 






\bibliographystyle{IEEEtran}
\bibliography{reference}


\end{document}